\documentclass{article}
\usepackage[main,final]{neurips_2026}

\usepackage[utf8]{inputenc} 
\usepackage[T1]{fontenc}    
\usepackage{url}            
\usepackage{booktabs}       
\usepackage{amsfonts}       
\usepackage{nicefrac}       
\usepackage{microtype}      
\usepackage{xcolor}         
\usepackage{graphicx}       
\newcommand{\PaperAuthorSeparator}{\And}
\usepackage{presets}
\usepackage{algorithm}      
\usepackage{algpseudocode}

\usepackage{float}
\usepackage{enumitem}

\begin{document}

\maketitle

\begin{abstract}
  We study stochastic linear bandits under a natural combination of batching and communication constraints: the time horizon is partitioned into batches of equal size $B$, and during each batch the learner sends $B$ requested arm pulls to an agent, who then observes the corresponding $B$ rewards and responds with a single bit of feedback to the learner. For each batch, the learner specifies the 1-bit quantization rule the agent uses, which may depend on all previously received bits but not on any past rewards directly. This setting addresses a significant yet unexplored ``middle ground'' between previous models having per-round quantization only \emph{or} total bit budgets only. We establish a minimax lower bound showing that in $d$ dimensions with action set size $\lvert \mathcal{A} \rvert$, $\Omega(\min\{T,B\min\{d,\log\lvert\mathcal{A}\rvert\}\})$ regret is unavoidable due to the 1‑bit communication bottleneck, even in the absence of noise. Combined with standard statistical limits, this yields a general lower bound of $\widetilde{\Omega}(\min\{T,\, B\min\{d,\log\lvert\mathcal{A}\rvert\} + \sqrt{dT\min\{d,\log\lvert\mathcal{A}\rvert\}}\})$. We develop two phased‑elimination algorithms based on $G$-optimal designs and 1‑bit mean estimation. The first achieves $\widetilde{O}(dB + d\sqrt{T})$ regret, matching the lower bound up to logarithmic factors when $\lvert \mathcal{A} \rvert = \exp(\Omega(d))$, and the second incorporates a safe‑arm identification and warm‑start procedure to obtain $\widetilde{O}(B\log\lvert \mathcal{A} \rvert + d^{3/2}\sqrt{B} + \sqrt{dT\log\lvert \mathcal{A} \rvert})$ regret, which is near‑optimal in broad scaling regimes of $(\lvert \mathcal{A} \rvert, B, d, T)$. Together, our results demonstrate that a single bit of feedback per batch suffices to nearly match the minimax regret of \emph{unconstrained} linear bandits in broad scaling regimes, even for batch sizes as large as $\Theta(\sqrt{T})$.
\end{abstract}

\section{Introduction} \label{sec:intro}

The study of bandit algorithms has seen a resurgence over the past two decades, driven by applications in online advertising \cite{li2010contextual, mehta2007adwords}, clinical trials \cite{durand2018contextual}, anomaly detection \cite{ding2019interactive, soemers2018adapting}, and recommendation systems \cite{mary2015bandits, bouneffouf2012contextual}. Within this framework, we consider a sequential decision-making problem arising from a fundamental asymmetry in communication bandwidth. A \emph{learner} coordinates with an \emph{agent} in the field: the learner can transmit detailed instructions prescribing a batch of $B$ measurements for the agent to take, but the agent—after observing all $B$ outcomes—can respond with only a single bit of feedback. This model is motivated by settings where devices observing rewards have severely limited communication capacity and cannot accumulate memory between interactions. In sensor networks used for environmental monitoring or industrial diagnostics, for instance, each sensor may be battery-constrained and able to transmit only a limited number of bits per reporting interval, and its on-board memory may be too small to store a full history of measurements. Similarly, in federated learning and edge computing deployments, a central server orchestrates experiments by dispatching actions to end devices—phones, IoT endpoints, and similar devices—which execute them over a fixed window before returning a short summary; crucially, the device participating in one round may differ from the device participating in the next, precluding any accumulation of local state. In both cases, single-device-with-unlimited-memory models such as \cite{salgia2023distributed, amani2023distributed} and per-pull few-bits models such as \cite{hanna2022solving, mitra2023linear, lau2025quantile} (outlined in more detail below) may not fully capture the relevant constraints.

We study the above setting under the standard stochastic linear bandit model, in which each reward is a linear function of an unknown parameter vector corrupted by noise. Our goal is to characterize the fundamental limits of learning under this one-bit uplink bottleneck, and to design algorithms whose performance approaches those limits. Specifically, we focus on \emph{regret minimization}: minimizing the cumulative sub-optimality gap—the difference in mean reward between the best action and the action played—over a fixed time horizon $T$.
In the standard linear bandit setting \cite[Chapter 19]{lattimore2020bandit}, the learner directly observes the generated rewards in each time step, and can use this information to guide decision-making in future time steps.  Two prominent lines of work studying the role of communication constraints in linear bandits are as follows:
\begin{itemize}[leftmargin=5ex]
    \item In works such as \cite{hanna2022solving, mitra2023linear, lau2025quantile}, a \emph{memoryless} decentralized agent observes each reward, and based on that reward alone, it sends back a few bits of information (e.g., 3 or 4 bits), i.e., a quantization of the reward.
    \item In works such as \cite{salgia2023distributed, amani2023distributed}, the agent(s) observing the rewards has \emph{unlimited memory},\footnote{They also allow multiple distributed agents, but we are mainly interested in the special case of a single agent.} and the communication constraint is on the \emph{total number of bits} sent throughout the time horizon.  This allows each agent to perform local estimation and send larger amounts of information at suitably-chosen times (while sending nothing at the other times).
\end{itemize}
As hinted above, the main premise of our work is that these settings leave a \emph{considerable unexplored ``middle ground''} in which the agent has \emph{limited memory} and communication occurs at \emph{regular intervals}, but not every round.  This setting leads to both distinct algorithmic ideas and new insights, such as just \emph{one bit at the end of each $\sqrt{T}$-round batch} (with time horizon $T$) often sufficing to match the minimax-optimal regret of unconstrained learning to within logarithmic factors.

Our paper is organized as follows: In Section \ref{sec:related_work}, we provide a more detailed summary of existing works related to our problem; in Section \ref{sec:contributions}, we outline our main contributions; in Section \ref{sec:problem_setup}, we give a formal presentation of the problem setup; in Section \ref{sec:main_results}, we formally state our main results; and in Section \ref{sec:algs}, we detail the corresponding algorithms. The proofs can be found in the appendix. 

\subsection{Related Work} \label{sec:related_work}

{\bf Linear Bandits.} Stochastic linear bandits were first studied in the ``standard'' setting (i.e., without communication or batch constraints) in \cite{abe1999associative}, and a prominent follow-up work is \cite{auer2002using}, who considered an algorithm for stochastic linear bandits with a finite action set based on the ``optimism principle'' \cite[Chapter 7.1]{lattimore2020bandit}. \cite{dani2008stochastic} extended these ideas to provide an algorithm attaining $\widetilde{O}(d\sqrt{T})$ regret for any action set with bounded mean rewards, which they also showed to be minimax optimal up to logarithmic factors for worst-case action sets. The same $\Omega(d\sqrt{T})$ minimax lower bound was subsequently established for the unit sphere and hypercube action sets \cite[Chapter 24]{lattimore2020bandit}. Further algorithms for linear bandits with finite action sets were later derived \cite{lattimore2020learning, chu2011contextual, li2019nearly}, all of which (alongside \cite{auer2002using}) obtain an upper bound of $\widetilde{O}(\sqrt{dT \log \lvert \cA \rvert})$ on the regret, with varying logarithmic dependence. Complementary lower bounds of $\widetilde{\Omega}(\sqrt{dT \min\{d, \log \lvert \cA \rvert \} })$ are also well-known for fixed finite action sets (e.g., see \cite{zhou19lower, tajdini2025improved}), and $\Omega(\sqrt{dT \log (T/d) \log \lvert \cA \rvert})$ for finite action sets that change in each time step \cite{li2019nearly}. 

{\bf Bandits with Communication Constraints.} 
Most work on linear bandits assumes that the reward of an action can be observed directly by the learner (with full precision).
However, in applications such as sensor networks, rewards are observed by an agent (e.g., sensor) before being communicated to the learner (central server) over a link with limited bandwidth. 
This has motivated research on communication-constrained bandits; see~\cite[Appendix A]{salgia2023distributed} for a broader survey.
The studies most pertinent to this work are those that focus on the quantization of the reward feedback to reduce the uplink communication cost.\footnote{We refer to the ``uplink communication cost'' as the number of bits communicated by the agent to the central learner in a given time step. See Remark~\ref{rem:downlink} for our assumptions
on downlink communication.}
In particular, \cite{hanna2022solving} devised a quantization scheme that allows existing algorithms to be run with an uplink communication cost of a few bits per arm pull, while maintaining the same regret as the unquantized algorithms up to constant multiplicative factors. An approach based on quantizing the linear parameters (rather than the rewards) was also proposed in \cite{mitra2023linear}, attaining $\widetilde{O}(d\sqrt{T})$ regret with an uplink communication cost of $\Theta(d)$ per round. In contrast to these works, our agents have limited memory and can only send one bit after an entire batch of $B$ pulls, which forces the learner to pre‑commit actions and reduces the effective communication rate to $1/B$ bits per round.
Best-arm identification with one bit per pull has also been
studied for quantile rewards~\cite{lau2025quantile} and mean
rewards~\cite{luong2026nearly}, whereas we study cumulative
regret with batched feedback.

Beyond these, communication-constrained bandits have been studied under a variety of related assumptions, including federated linear bandits~\cite{huang2021federated, li2022asynchronous, he2022simple, fan2023federated, shi2024harnessing} and differentially private linear bandits~\cite{shariff2018differentially, dubey2020differentially, hanna2024differentially} where the privacy constraint induces a communication-like bottleneck. 
Since the number of such works is large and most operate under assumptions substantially different from ours (e.g., multi-agent coordination, privacy, or unlimited per-agent memory), we focused primarily on the few above that are most directly comparable.

{\bf Batched Bandits.} Batched bandits were first studied in \cite{cesa2013online}, where they showed that in the multi-armed setting, $O(\log \log T)$ batches suffice to attain order-optimal regret. It was then shown in \cite{gao2019batched} that in order to obtain order-optimal minimax regret in the multi-armed setting, $\Theta(\log \log T)$ batches is both necessary and sufficient for a finite time horizon $T$. $O(\log \log T)$ batches were also shown to be sufficient to obtain order-optimal regret in the contextual linear bandit setting \cite{ruan2021linear, zhang2025almost, hanna2023contexts}, of which linear bandits are a special case. Finally, \cite{ren2024optimal} proved that 3 batches are necessary and sufficient to attain \emph{asymptotic optimality for a fixed instance} as $T \to \infty$. 
We also note that, while not explicitly stated as such, the algorithms in \cite{lattimore2020learning, chu2011contextual, li2019nearly, salgia2023distributed} choose multiple arms simultaneously and thus permit varying degrees of batching. 

The preceding works allow variable-length batches, whereas \emph{fixed batch lengths} can also have significant practical relevance, and are popular in related topics such as Bayesian optimization \cite{nava2022diversified}.  While fixed batch lengths have been considered in linear bandits, the ones that we are aware of have a different emphasis such as privacy \cite{chowdhury2022shuffle, zhou2024on}, best-arm identification \cite{jun2016top}, or converting \emph{arbitrary} algorithms to batch settings \cite{provodin2022impact}.  Accordingly, we believe that our findings on regret with fixed batch lengths are of general interest even regardless of communication considerations.  We also briefly note that the concept of \emph{delayed feedback} is relevant \cite{vernade2020linear, howson2023delayed}, since a batch length $B$ implies a delay lying in the range $[0,B]$. Within this setting, we note a result from \cite{vakili2023delayed}, which when translated to our setup (without $1$-bit constraints), yields a regret upper bound of $\widetilde{O}(B + \sqrt{dT \min\{d, \log \lvert \cA \rvert\}} )$. Compared to the regret in the vanilla linear bandit setup, this introduces an additive ``batch or delay penalty'' of $\widetilde{O}(B)$. In contrast, Corollary~\ref{cor:general_lower_bound} establishes that a significantly larger additive penalty of $\Omega(B \min\{d, \log \lvert \cA \rvert\})$ is unavoidable in our setup, showing that communication constraints make our problem strictly harder than standard delayed feedback setting.
See also Remark~\ref{rem:q_bits} for how increasing the number of feedback bits per batch reduces the communication penalty, while the unavoidable $\Omega(B)$ batching penalty remains.

\subsection{Contributions} \label{sec:contributions}

As noted above, we introduce a new problem setup serving as a ``middle ground'' between two extremes studied in previous works.  Our technical contributions are outlined as follows for an action set $\cA$, batch size $B$, dimension $d$, and time horizon $T$ (see Section \ref{sec:problem_setup} for the detailed problem setup):
\begin{itemize}[leftmargin=5ex]
    \item We provide an algorithm that achieves an $\lvert\cA \rvert$-independent regret bound that is minimax order-optimal (up to logarithmic factors), even when the noise terms $(\xi_t)_{t=1}^T$ are only assumed to have finite variance (thus potentially being much heavier than sub-Gaussian noise).
    
    \item We provide a second algorithm that achieves an $\lvert \cA \rvert$-dependent regret bound that is order-optimal (up to logarithmic factors) in broad scaling regimes of $(\lvert \cA \rvert , B,d,T)$. 
    
    \item We provide an algorithm-independent minimax lower bound, showing that incurring an expected regret of $\Omega(B\min\{d, \log \lvert \cA \rvert \})$ is unavoidable in the worst case. In fact, this bound holds even in the completely noiseless setting ($\sigma = 0$), demonstrating that this bottleneck is purely a consequence of the $1$-bit communication constraint.  This is combined with well-known lower bounds for standard linear bandits to establish the above near-optimality claims. 
\end{itemize}
Our algorithm for $\lvert \cA \rvert$-independent regret is the simpler of the two (though still having numerous technical challenges compared to the standard setting), and involves performing phased elimination with mean rewards that are estimated by \emph{always pulling the same arm $B$ times in a given batch}.  In contrast, our algorithm for $\lvert \cA \rvert$-dependent regret crucially avoids doing this\footnote{Theorem \ref{thm:non_mixing_lb} will reveal that any strategy doing so must be highly suboptimal in regimes of smaller $\lvert \cA \rvert$.} and instead incorporates initial steps to find a \emph{safe arm} whose regret is ``not too high'', and partially relies on it to reduce regret.

\section{Problem Setup} \label{sec:problem_setup}

We consider a stochastic linear bandit instance with dimension $d \in \bN$, a finite\footnote{Although we formally assume $\cA$ is finite to cleanly maintain active sets and compute experimental designs, the regret bounds attained by our $\lvert \cA \rvert$-independent algorithm (Theorem~\ref{thm:exp_A_upper}) also hold for arbitrary compact sets with possibly infinite cardinality via a standard covering argument; see Remark~\ref{rem:infinite_sets} for details.} action set $\cA \subset \bR^d$, unknown parameter vector $\theta \in \bR^d$, and time horizon $T \in \bN$. At each time step $t \in \{1,\dots, T\}$, an action $A_t \in \cA$ is played, and the reward $Y_t$ is generated according to
\begin{equation}
    \label{eq:reward}
    Y_t = \langle A_t,\theta\rangle + \xi_t,
    \qquad
    \bE[\xi_t] = 0,
    \qquad
    \bE[\xi_t^2] \le \sigma^2,
\end{equation}
where $(\xi_t)_{t=1}^T$ is a sequence of i.i.d. noise terms
and $\sigma > 0$ is known. We further assume that mean rewards are \emph{bounded}, in the sense that $\lvert \langle a, \theta\rangle \rvert \leq 1$ for each $ a \in \cA$.

\begin{remark}[Effective dimension]\label{rem:effective_dimension}
For the upper-bound analyses, we may assume full span after expressing the actions and projected parameter in orthonormal coordinates on $\Span(\cA)$. This preserves rewards, regret, cardinality, the communication model, and Euclidean unit-ball constraints. Therefore, Theorems~\ref{thm:exp_A_upper} and~\ref{thm:poly_A_upper}, together with their corollaries, can be improved by  replacing $d$ with $\dim\Span(\cA)$, including their dimension-dependent hypotheses. During elimination, we work in the current active span as explained in Remark~\ref{rem:not_span_Rd}.
For lower bounds, Corollary~\ref{cor:general_lower_bound} is asserted only for $d\le\lvert\cA\rvert$; Theorem~\ref{thm:minimax_lower} has no such restriction.
\end{remark}

For a batch size $B \in \{1,\dots, T\}$, we assume the time horizon is split into $M = \lceil T/B\rceil$ equal-sized batches of size $B$ (except possibly a shorter final batch). In each batch $m \in \{1,\dots, M\}$, we have the following interaction between a central learner and an agent:
\begin{enumerate}[leftmargin=5ex]
    \item At the beginning of the batch, the learner chooses a length-$B$ sequence of arms $A_{(m-1)B+1}$, $A_{(m-1)B +2}, \dots, A_{mB}$ and a $1$-bit quantization function $\smash{Q_m : \bR^B \to \{0,1\}}$, both of which may depend on the feedback received from batches $1,\dots,m-1$.
    
    \item The agent receives the $1$-bit quantization function $Q_m$, and directly observes the rewards $Y_{(m-1)B + 1}, Y_{(m-1)B + 2}, \dots, Y_{mB}$, generated according to \eqref{eq:reward}.
    
    \item The agent computes $Q_m(Y_{(m-1)B + 1}, Y_{(m-1)B + 2}, \dots, Y_{mB})$ and transmits it to the learner.
\end{enumerate}

\begin{remark}
\label{rem:downlink}
We do not impose any (downlink) communication constraint from the learner to the agent, motivated by settings where uplink communication poses more of a bottleneck than downlink communication. Although our lower bound allows arbitrary measurable quantizers, our algorithms use only fixed, specified procedures (e.g., comparing empirical averages to specified thresholds, allocating arm pull proportions, and comparing estimated maximum rewards across groups of arms) rather than arbitrary queries.  Appendix \ref{sec:downlink} summarizes what needs to be sent from the learner to the agent in our algorithms. 
We also note that while we frame the problem with a single agent for clarity, different batches may correspond to different users or devices. We assume that the agent is `batch-wise memoryless', meaning its message cannot depend directly on rewards observed in previous batches. Similar assumptions were adopted in previous works~\cite{hanna2022solving,mitra2023linear,mayekar2023communication, lau2025quantile,luong2026nearly} with $B=1$ (i.e., no batching). The agent may use temporary memory and computation within a batch, but must follow the action sequence fixed before the batch. 
\end{remark}
The goal of the learner is to design an algorithm to minimize the \emph{cumulative regret}, defined as
\begin{equation}
    \label{eq:regret_definition} R_T \coloneqq \sum_{t=1}^T \langle a^* - A_t, \theta \rangle,
\end{equation}
where $a^* \coloneqq \arg\max_{a \in \cA} \, \langle a, \theta \rangle$ is the optimal arm.  We will consider both the expected regret $\E[R_T]$ and high-probability guarantees that hold with probability $1-\delta$, noting the latter can easily be converted to the former (while at most affecting logarithmic terms) by setting $\delta = 1/T$.

\section{Main Results} \label{sec:main_results}

We now formally state our main results, which give upper and
lower bounds on the minimax regret of stochastic linear bandits under 1-bit per-batch communication constraints. These bounds match up to logarithmic factors in broad scaling regimes described below. A notable consequence of our bounds is that
a single bit of feedback per batch suffices to match near-optimal \textit{unquantized, unbatched} regret even for batch sizes as large as
$B = \widetilde{O}(\sqrt{T})$. This demonstrates that the communication
bottleneck is less severe than one might expect. The proofs are given
in Appendices~\ref{app:lower_bound}, \ref{app:proof_exp_upper},
and~\ref{app:proof_poly_upper}.

\subsection{Regret Lower Bound}

We begin by stating our algorithm-independent minimax lower bound. 

\begin{theorem}[Communication-Based Lower Bound] 
\label{thm:minimax_lower}
    For any dimension $d \ge 1$, horizon $T \ge 1$, batch size $B \ge 1$, reward noise variance $\sigma^2 \ge 0$, and integer $K \ge 16$, there exists an action set $\cA$ of size $\lvert \cA \rvert = K$ with $\max_{a \in \cA} \|a\|_2 \le 1$ such that for every algorithm operating in $M=\lceil T/B\rceil$ batches that receives only one bit of feedback per batch, the expected cumulative regret satisfies\footnote{By defining the action set such that $\max_{a \in \cA} \|a\|_2 \le 1$ and allowing the adversary to pick any $\theta$ such that $\|\theta\|_2 \le 1$, the mean reward satisfies  $|\langle a, \theta \rangle| \le \|a\|_2 \|\theta\|_2 \le 1$ by the Cauchy-Schwarz inequality.} 
    \[
    \sup_{\|\theta\|_2 \le 1} \E[R_T] = \Omega\left(\min\{T, B\min(d,\log \lvert \cA \rvert)\}\right).
    \]
\end{theorem}

We informally outline the proof idea as follows:  When $\lvert \cA \rvert \le \exp(\Theta(d))$, the $B \log \lvert \cA \rvert$ term arises because we can form $\lvert \cA \rvert$ ``well-separated'' unit-norm vectors using tools from coding theory.  We require $\log_2 \lvert \cA \rvert$ bits to identify the correct one (corresponding to $B \log \lvert \cA \rvert$ pulls), and before this is achieved, we incur constant regret per round.  When $\lvert \cA \rvert \ge \exp(\omega(d))$, we can still use this argument on a subset of size $\exp(\Theta(d))$, with the remaining actions being ``dummy'' suboptimal actions.

Crucially, this lower bound holds even in the completely noiseless setting ($\sigma = 0$): even with perfect reward observations, the 1-bit uplink constraint alone forces $\Omega(B\min\{d,\log|\cA|\})$ regret. This demonstrates that the bottleneck is purely a communication bottleneck rather than a statistical one.

For the remainder of the paper we assume $\sigma = \Theta(1)$ (see Remark~\ref{rem:sigma_scaling} for extensions to other scalings); stochastic noise then 
introduces an additional statistical barrier on top of the communication bottleneck. Combining Theorem~\ref{thm:minimax_lower} with the well-known minimax lower bounds for standard linear bandits without communication constraints (e.g., \cite[Theorem 2 \& Corollary 3]{zhou19lower}), we immediately obtain the following general lower bound.\footnote{The two lower bounds may use different action sets of the prescribed cardinality; taking their maximum suffices.}

\begin{corollary}[General Regret Lower Bound]
\label{cor:general_lower_bound}
Under the setup in Theorem~\ref{thm:minimax_lower} with $\sigma = \Theta(1)$,
for any specified set size $\lvert\cA\rvert \ge d$, there exists a corresponding
action set $\cA$ of that size with $\max_{a\in\cA}\|a\|_2 \le 1$ such that
the expected cumulative regret satisfies
\begin{equation*}
  \label{eq:general_lower_bound}
  \sup_{\|\theta\|_2\le 1}\E[R_T]
  = \widetilde{\Omega}\left(\min\!\Big\{T,\;
      B\min\{d,\log\lvert\cA\rvert\}
      + \sqrt{dT\min\{d,\log\lvert\cA\rvert\}}
    \Big\}\right).
\end{equation*}
When $\lvert\cA\rvert = \exp(\Omega(d))$ this simplifies to
$\widetilde{\Omega}\big(\min\big\{T,\, dB + d\sqrt{T}\big\}\big)$, and when $\lvert\cA\rvert = \exp(O(d))$ this simplifies to $\widetilde{\Omega}\big(\min\big\{T,\, B\log|\cA| + \sqrt{dT\log|\cA|}\big\}\big)$.  Here $\widetilde{\Omega}(\cdot)$ hides possible logarithmic dependencies on $d$ in the final term.
\end{corollary}

The lower bound can be further tightened for the class of ``non-mixing'' algorithms, where in each batch $m = 1,\dotsc, M$ the learner selects a \emph{single} arm $A_m \in \mathcal{A}$, and the agent pulls this arm $B$ times, i.e., every arm chosen in the batch must be identical.

\begin{theorem}[Lower Bound for Non-Mixing Algorithms]
\label{thm:non_mixing_lb}
Fix a dimension $d \ge 2$, horizon $T \ge 1$, batch size $B \ge 1$, and reward noise variance $\sigma^2 \ge 0$. Let the action set be the standard basis $\cA = \{e_1,\dots,e_d\}$. For every non-mixing algorithm (as defined above), we have
\[
    \sup_{\theta\in \mathcal{A}} \E[R_T] = \Omega(\min\{T, dB\}).
\]
\end{theorem}

\subsection{Regret Upper Bounds}
With the lower bounds on regret established, we now give two constructive upper bounds: Theorem~\ref{thm:exp_A_upper} gives a regret bound that is independent of the action set size $\lvert \cA \rvert$, while Theorem~\ref{thm:poly_A_upper} gives an $\lvert\cA \rvert$-dependent upper bound that we then specialize to polynomial-size action sets. Section~\ref{sec:algs} explains the algorithms and their
main ideas.

\begin{theorem}[$\lvert \cA \rvert$-Independent Upper Bound] \label{thm:exp_A_upper}
    Fix a finite action set $\cA \subset \bR^d$, and consider the problem described in Section \ref{sec:problem_setup} with time horizon $T \in \bN$, variance $\sigma^2  = \Theta(1)$, and batch size $B \in \{1,\dots, T\}$. Then, for any $\delta \in (0,1)$, there exists an algorithm (see Algorithm \ref{alg:exp_action_set} in Appendix~\ref{app:proof_exp_upper}) that, with probability at least $1-\delta$, incurs a regret of at most 
    \begin{equation} \label{eq:exp_d_regret}
        R_T = \widetilde{O}\left( dB\log\left(\frac{1}{\delta}\right) + d\sqrt{T\log\Big(\frac{1}{\delta}\Big)}\right),
    \end{equation}
    where $\widetilde{O}$ hides logarithmic dependencies on $d$ and $T$.
\end{theorem}

In view of Corollary~\ref{cor:general_lower_bound}, we see that by setting $\delta = 1/T$, Theorem \ref{thm:exp_A_upper} yields an expected regret that matches the minimax lower bound (up to logarithmic factors) when $\lvert \cA \rvert = \exp(\Omega(d))$.

\begin{remark}[Extension to Infinite Action Sets]
\label{rem:infinite_sets}
Theorem~\ref{thm:exp_A_upper} extends to any nonempty compact action set $\cA\subset\bR^d$ by discretization.
Taking $\eps=1/T$ and $\mathcal L=\{\vartheta\in\bR^d:
\sup_{a\in\cA}|\langle a,\vartheta\rangle|\le1\}$ in
\cite[Exercise~27.6(c)]{lattimore2020bandit}, we obtain
a finite subset $\cA_{1/T}\subseteq\cA$ with $\lvert \cA_{1/T}\rvert \leq (6dT)^d$, such that every
$a\in\cA$ admits a representative $b\in\cA_{1/T}$ with
$\sup_{\vartheta\in\mathcal L}
|\langle a-b,\vartheta\rangle|\le1/T$.
Since $\theta\in\mathcal L$, the best mean reward on $\cA_{1/T}$ is within $1/T$ of that on $\cA$.
Applying Theorem~\ref{thm:exp_A_upper} to this finite set,
working in its span if necessary, adds at most $1$
to the regret relative to $\cA$.
\end{remark}

\begin{theorem}[$\lvert \cA \rvert$-Dependent Upper Bound] \label{thm:poly_A_upper}
     Fix a finite action set $\cA \subset \bR^d$. For the problem described in Section \ref{sec:problem_setup} with time horizon $T \in \bN$, variance $\sigma^2  = \Theta(1)$, and batch size $B \in \{1,\dots, T\}$, there exists an algorithm (see Algorithm~\ref{alg:poly_action_set} in Appendix~\ref{app:proof_poly_upper}) that, with probability at least $1-\delta$, incurs a regret of at most 
    \begin{equation} \label{eq:general_d_regret}
        R_T =  \widetilde{O}\left(B \log \Big( \frac{\lvert \cA \rvert }{\delta}\Big) +  d^{\frac{3}{2}}\sqrt{B} \log^{\frac{5}{2}}\Big(\frac{\lvert \cA \rvert}{\delta}\Big) +  \sqrt{dT \log \Big(\frac{\lvert \cA \rvert}{\delta}\Big)}\right),
    \end{equation}
    where $\widetilde{O}$ hides logarithmic dependencies on $d$ and $T$.\footnote{The first two terms can be tightened to
$\min\{B\log(|\cA|/\delta) + d^{3/2}\sqrt{B}\log^{5/2}(|\cA|/\delta),\,
dB\log(|\cA|/\delta)\}$ by taking the better of the small-$B$ and large-$B$ analyses, see Theorem~\ref{thm:step3_small_b_regret} in Appendix~\ref{app:proof_poly_upper}.}
\end{theorem}

To better understand the behavior of this regret bound, we specialize Theorem~\ref{thm:poly_A_upper} to the important special case where the action set size is polynomial in $d$. Here we set $\delta = 1/T$ to obtain a bound on the expected regret, though a high-probability counterpart naturally also follows.

\begin{corollary}[Upper Bound for Polynomial-Size Action Sets] \label{cor:poly_A_upper}
    When $\lvert \cA \rvert = {\rm poly}(d)$ and $\sigma = \Theta(1)$, the expected regret of Algorithm~\ref{alg:poly_action_set} is bounded by
    \begin{equation} \label{eq:poly_d_regret}
        \E[R_T] = \widetilde{O}\left(B + d^{\frac{3}{2}}\sqrt{B} + \sqrt{dT} \right).
    \end{equation}
\end{corollary}

Notice that \eqref{eq:poly_d_regret} consists of three components: the communication term $B$, the statistical term $\sqrt{dT}$, and an ``algorithmic initialization penalty'' $d^{3/2}\sqrt{B}$ (see Section \ref{sec:algs} for discussion). The algorithm achieves near minimax optimality whenever this penalty term is absorbed by either of the two lower bounds.  This is formalized in the following corollary, which also establishes that the degradation is not too severe when the penalty is not fully absorbed.

\begin{corollary}
\label{cor:gap_poly_A}
Assume $d\le\lvert\mathcal A\rvert=\mathrm{poly}(d)$ and $\sigma=\Theta(1)$.
Write $\mathrm{LB}(B,T,d)\coloneqq\min\{T,B+\sqrt{dT}\}$.
Compared to the minimax lower bound
$\widetilde{\Omega}(\mathrm{LB}(B,T,d))$
from Corollary~\ref{cor:general_lower_bound}, the expected regret $\E[R_T]$ of Algorithm~\ref{alg:poly_action_set} is minimax order-optimal up to $\mathrm{polylog}(d, T)$ factors provided that at least one of $B = \Omega(d^3)$ or $T = \Omega(d^2 B)$ holds (which is guaranteed when $T = \Omega(d^5)$). 
Furthermore, in all remaining regimes, we have $\frac{\E[R_T]}{\mathrm{LB}(B, T, d)} \leq \widetilde{O}(\sqrt{d})$.
\end{corollary}

The proof of Corollary \ref{cor:gap_poly_A} is deferred to Appendix~\ref{subsec:bound_gap}. The condition $T/B = \Omega(d^2)$ requires that the horizon is long enough for the learner to receive $\Omega(d^2)$ batches of feedback; this is natural given that the agent transmits only 1 bit per batch, so the learner is severely information-starved in short horizons. However, Corollary~\ref{cor:gap_poly_A} guarantees a gap of at most $\widetilde{O}(\sqrt{d})$ even in these harder regimes.

\begin{remark}[Scaling of $\sigma$] \label{rem:sigma_scaling}
     For ease of presentation, and since it is the most widely-considered regime, we follow the convention in~\cite[Chapter~19]{lattimore2020bandit} and take $\sigma = \Theta(1)$ in Theorems~\ref{thm:exp_A_upper} and~\ref{thm:poly_A_upper} and their proofs. However, our methods readily extend to any $\sigma$ that is not pathologically small. For any known $\sigma$ such that $\sigma = O(1)$ and $\sigma = \Omega(T^{-c})$ for some fixed $c>0$, a similar analysis can be performed using the same algorithms. By tracking the $\sigma$ dependence in our analysis, the first term in \eqref{eq:exp_d_regret} and the first two terms in \eqref{eq:general_d_regret} change by at most an additional $O\big(\log(1/\sigma)\big) = O(\log T)$ factor, while a multiplicative factor of $\sigma$ appears in front of the final term in both equations.  The final term in the lower bound also becomes $\widetilde{\Omega}(\min(T,\sigma\sqrt{dT\min\{d,\log|\mathcal{A}|\}}))$ (e.g., by slightly adapting the proof of \cite[Thm.~2]{tajdini2025improved}), and near‑optimality is preserved in broad scaling regimes.
\end{remark}

\begin{remark}[Multiple Bits per Batch]
\label{rem:q_bits}
Our results can be generalized to the $q$-bit setting, where the agent may transmit $q \ge 1$ bits at the end of each batch.  We outline the results here, and provide more details in Appendix~\ref{app:q_bits}.
The communication-based lower bound in Theorem~\ref{thm:minimax_lower} becomes
$\Omega\bigl(\min\{T,\,B+(B/q)\min\{d,\log|\cA|\}\}\bigr)$ with a very 
similar proof as the 1-bit case.  
The additive $B$ term reflects the batching penalty that
persists even with unquantized feedback, since the first batch's
actions must be chosen before receiving any rewards.
For the upper bounds, we define $\bar q=\min\{q,B\}$, since 
the benefit of having more bits turns out to saturate at $B$ (at least
for minimax regret in the $\widetilde{O}(\cdot)$ sense). 
Modified versions of Algorithms~\ref{alg:exp_action_set}
and~\ref{alg:poly_action_set} attain the regret bounds in
\eqref{eq:exp_d_regret} and~\eqref{eq:general_d_regret}
with $B$ replaced by $B/\bar q$
and an additional $\widetilde{O}(B\log(1/\delta))$ term.
These bounds remain optimal up to logarithmic factors for exponentially large action sets, and for polynomial-size action sets whenever the resulting $d^{3/2}\sqrt{B/\bar q}$ term is dominated by the other terms.\footnote{The tightened bound given in the footnote to Theorem~\ref{thm:poly_A_upper} also applies with~$B$ replaced by~$B/\bar q$. In particular, for~$q\ge B$, the small-batch bound, together with the additional batching overhead and the
trivial regret bound, gives $\E[R_T]=\widetilde O(\min\{T,B+\sqrt{dT}\})$ for $d\le|\cA|=\mathrm{poly}(d)$ and $\delta=1/T$. Thus optimality up to logarithmic factors also holds in this regime.}
\end{remark}

\subsection{Discussion and Open Problems} \label{subsec:discussions_limitations}

While Theorem~\ref{thm:exp_A_upper} attains near-optimal regret for
exponential-sized action sets without any assumption on $T$, some gaps
remain in the guarantees of Theorem~\ref{thm:poly_A_upper} for smaller
action sets. For $\lvert\cA\rvert = \mathrm{poly}(d)$,
Corollary~\ref{cor:gap_poly_A} establishes that $T = \Omega(d^5)$ suffices
for near order-optimality (or more generally, at least one of $B = \Omega(d^3)$ or $T = \Omega(d^2 B)$), which captures broad scalings but is stricter than the $T = \Omega(d)$
condition required in the standard unconstrained setting. For
intermediate-sized action sets $\lvert\cA\rvert = \exp(\Theta(d^\alpha))$
with $\alpha \in (0,1)$, a similar argument to Corollary~\ref{cor:gap_poly_A} reveals that Theorem~\ref{thm:poly_A_upper} achieves the
target bound $\widetilde{O}(d^\alpha B + \sqrt{d^{1+\alpha}T})$ when
either $B = \Omega(d^{3(1+\alpha)})$ or $T/B = \Omega(d^{2(1+2\alpha)})$.
In summary, our results demonstrate near-optimality across broad scaling
regimes of $(\lvert\cA\rvert, B, d, T)$; closing the remaining gaps is an
interesting direction for future work.

\section{Proposed Algorithms} \label{sec:algs}

In this section, we introduce the two algorithms establishing Theorems~\ref{thm:exp_A_upper} and~\ref{thm:poly_A_upper}. Both algorithms use a well-known \emph{phased elimination} framework along with suitably-chosen {\em 1-bit mean estimation} subroutines, but they have some important differences:
\begin{itemize}[leftmargin=5ex]
    \item The algorithm for Theorem~\ref{thm:exp_A_upper} is non-mixing and favors simplicity, but it incurs a $\sqrt{d}$ factor in the confidence width that ultimately limits the regimes in which the regret bound is tight.
    \item The algorithm for Theorem~\ref{thm:poly_A_upper} modifies some of the key steps around arm allocation and robust mean estimation. For large $B$, it also crucially adds two \emph{safe-arm identification} and \emph{warm-start} subroutines in which mixing strategies are used.  The motivation for these will be given below.
\end{itemize}
See Remark~\ref{rem:comparison} in Appendix~\ref{app:proof_poly_upper} for further comparison of the two algorithms, and Algorithms~\ref{alg:exp_action_set} and~\ref{alg:poly_action_set} in Appendices~\ref{app:proof_exp_upper} and~\ref{app:proof_poly_upper} for the full algorithm descriptions.



\subsection{Preliminaries: Phased Elimination and 1-bit Mean Estimation} \label{subsec:prelims}

\textbf{Phased elimination with $G$-optimal designs.} Phased elimination algorithms \cite{lattimore2020learning, camilleri2021high, bogunovic2021stochastic, liu2024corruption, tajdini2025improved, lattimore2020bandit} partition the time horizon into \emph{epochs} $h = 1, 2, \dots$, maintaining an \emph{active set} $\cA_h \subseteq \cA$ of plausibly-optimal arms (with $\cA_1 = \cA$). Within each epoch, the learner uses an \emph{experimental design} $\rho_h \colon \cA_h \to [0,1]$ to efficiently estimate the mean rewards of \emph{all} arms in $\cA_h$ before eliminating those that appear suboptimal based on the confidence width $W_h$.

Throughout, we make use of near $G$-optimal designs; specifically, provided that $\cA_h$ is compact and spans $\bR^d$, there exist designs with the following guarantees:
\begin{equation}
\label{eq:near_optimal_design}
    \max_{a \in \cA_h} \lVert a \rVert_{G_h^{-1}}^2 \le 2d,
    \qquad
    \lvert \supp(\rho_h)\rvert \le 4d(\log\log d + 11) = \widetilde{O}(d),
\end{equation}
where $G_h = \sum_{a \in \cA_h}\rho_h(a) aa^\top$ and $\lVert a \rVert_V = \sqrt{a^\top V a}$ for a positive semi-definite matrix $V$.  For explicitly represented finite active sets, such designs can be computed in polynomial time via the Frank-Wolfe algorithm
\cite[Proposition 3.17]{todd2016minimum}. The set $\supp(\rho_h)$ is called the \emph{core set}, and notably has a size of only $\widetilde{O}(d)$. In the standard (unquantized and unbatched) setting, the rewards are directly observed, each core-set arm is pulled in proportion to $\rho_h$, and a weighted least-squares estimator $\hat{\theta}_h$ is formed. The mean reward of any arm $b \in \cA_h$ is then estimated by $\langle b, \hat{\theta}_h\rangle$. With appropriate epoch lengths and elimination thresholds, this yields the standard minimax regret bound of $\widetilde{O}(\sqrt{dT \log \lvert \cA \rvert})$.

\begin{remark}[Active span]\label{rem:not_span_Rd}
The properties in~\eqref{eq:near_optimal_design} are stated assuming that $\cA_h$ spans $\bR^d$. Although the initial action set satisfies this assumption by the convention in Remark~\ref{rem:effective_dimension}, elimination may reduce the active span. We then work in the smaller subspace $\Span(\cA_h)$, interpreting $G_h^{-1}$ as the inverse on this span. Since $\dim\Span(\cA_h)\le d$, the bounds in~\eqref{eq:near_optimal_design} remain valid with the initial dimension $d$; see~\cite[Remark 5.2]{lattimore2020learning}.
\end{remark}

\begin{remark}[Computational complexity] \label{rem:complexity}
    As mentioned above, the core sets in each epoch $h$ can be found in polynomial time using the Frank-Wolfe algorithm. However, the overall computational complexity in each epoch is $\Omega(|\cA_h|)$ for a direct implementation, since estimating the rewards of the arms in $\cA_h$ involves enumerating all $|\cA_h\rvert$ arms.  This is potentially expensive in some scenarios (e.g., when $|\cA| = \exp(\Omega(d))$), but this limitation is common to all phased elimination algorithms in which the arms are enumerated (e.g., see \cite{camilleri2021high}).  For suitably structured action sets, more efficient oracle-based implementations may be possible (see~\cite{wagenmaker2021exp} for related results in the combinatorial semi-bandit setting), but we do not pursue this idea in this paper.
\end{remark}

\textbf{Adapting to our setting.} Two immediate challenges arise in our constrained setup:
(i) the learner observes only 1-bit summaries of $B$ rewards rather than $B$ unquantized rewards, rendering the standard weighted least-squares estimator unusable, and
(ii) the allocation dictated by the continuous experimental design may not yield a whole number of batches, complicating the protocol.

We resolve (ii) by rounding per-arm allocations up to multiples of $B$: each batch then pulls either a single arm (in Algorithm~\ref{alg:exp_action_set}, and in most batches of Algorithm~\ref{alg:poly_action_set}) or one core-set arm plus a pre-determined \emph{safe arm} to fill the remaining slots (in the warm-start of Algorithm~\ref{alg:poly_action_set}). 

To resolve (i), we replace the standard mean estimator with a \emph{1-bit mean estimator} built on the framework of \cite{lau2026order,lau2026sequential}. In our setting, a single 1-bit ``query'' is executed by having the agent pull the same arm $B$ times (or possibly fewer in Algorithm~\ref{alg:poly_action_set}, but we defer this distinction to later) and applying a 1-bit threshold to the sample mean. This reduces the effective standard deviation of the continuous measurement to $\sigma' = \sigma/\sqrt{B}$. With this established, we can employ two subroutines: \textsc{1BitLocalize} uses $\widetilde{O}\left(\log \frac{1}{\sigma'} + \log \frac{1}{\delta}\right)$ queries to return an interval $[L_a, U_a] \subseteq [-1,1]$ of width $O(\sigma')$ that contains $\mu_a$ with high probability; and \textsc{1BitRefine} uses $n_{\rm q}$ additional queries (for some suitably chosen $n_{\rm q}$) to return a final estimate $\hat{Y}_{a,h}$ with variance $\widetilde{O}(\sigma'^2/n_{\rm q})$ and bias ${O}(\sigma'/T^2)$. The detailed subroutines and their formal guarantees are given in Appendix~\ref{app: 1bit subroutines}.

\textbf{Common template.} Combining the preceding ingredients, a given epoch (indexed by $h$) in both algorithms proceeds as follows:
\begin{enumerate}[leftmargin=4ex, label=(\arabic*)]
    \item \label{enum:G_opt} Compute a near $G$-optimal design $\rho_h$ on $\cA_h$ satisfying~\eqref{eq:near_optimal_design}.
    \item Run \textsc{1BitLocalize} on each $a \in \supp(\rho_h)$ to obtain an interval $[L_a, U_a]$.
    \item Split the refinement budget into $K$ blocks. For each block $j = 1, \dots, K$ and each core-set arm $a \in \supp(\rho_h)$, run \textsc{1BitRefine} to obtain independent scalar estimates $\hat{Y}_{a,h}^{(j)}$.
    
     \item Form a robust mean reward estimate $\hat{\mu}_{b,h}$ for every active arm $b \in \cA_h$ using the scalar estimates $\hat{Y}_{a,h}^{(j)}$ and the design weights $\rho_h$. This involves computing design-weighted parameter estimates (using the inverse matrix $G_h^{-1}$) and applying median-of-means aggregation. The order of steps is algorithm-dependent: Algorithm~\ref{alg:exp_action_set} takes the median of the scalars first, whereas Algorithm~\ref{alg:poly_action_set} computes block-wise parameters first.

    \item\label{enum:elimination} Eliminate arms whose empirical gap exceeds twice the confidence width~$W_h$:
    \[
        \cA_{h+1} = \left\{a \in \cA_h : \max_{b \in \cA_h} \hat{\mu}_{b,h} - \hat{\mu}_{a,h} \le 2 W_h \right\}.
    \]
\end{enumerate}

The two algorithms differ in (a) how refinement samples are allocated across core-set arms, (b) the order of operations in Step 4, (c) the number of blocks $K$, (d) the confidence width $W_h$, and (e) additional pre-processing before the main loop (for the $\lvert \cA \rvert$-dependent algorithm with large $B$). Both algorithms make use of a common precision parameter chosen as $\eps_h = \sigma 2^{-h}/\sqrt{B}$, but they differ in how this translates to the final confidence width $W_h$. We now outline both algorithms in more detail.

\subsection{Algorithm~\ref{alg:exp_action_set}: \texorpdfstring{$\lvert\cA\rvert$}{|A|}-Independent Regret}

For our $\lvert\cA\rvert$-independent result (which is primarily suited to large action sets, e.g., exponential-sized), we exploit the observation that it suffices to accurately estimate only the $\widetilde{O}(d)$ scalars $\{\mu_a\}_{a\in\supp(\rho_h)}$ rather than all $|\cA|$ arm means simultaneously, replacing a union bound over $|\cA|$ arms with one over only the $\widetilde{O}(d)$ core-set arms. This enables two simplifications: 
\begin{itemize}[leftmargin=4ex]
    \item Refinement samples are distributed \emph{uniformly} across core-set arms (rather than proportionally to $\rho_h$), ensuring that each scalar block estimate $\hat{Y}_{a,h}^{(j)}$ has variance $O(\eps_h^2)$ for valid localization intervals, with a bound independent of $\rho_h$.

    \item Median-of-means is applied directly to the scalar estimates to form a robust target $\hat{Y}_{a,h} = \operatorname{median}\{\hat{Y}_{a,h}^{(j)}\}_{j=1}^K$. The global parameter is then estimated via the design weights as
    \begin{equation}
        \label{eq:theta_hat_exp}
        \hat{\theta}_h = G_h^{-1}\sum_{a \in \supp(\rho_h)} \rho_h(a) \cdot a \cdot \hat{Y}_{a,h}
        \qquad \text{where} \qquad
        G_h = \sum_{a \in \supp(\rho_h)} \rho_h(a)\,aa^\top.
    \end{equation}
    Because concentration is only required over the $\widetilde{O}(d)$ support arms, this requires only $K = \widetilde{O}(\log(d/\delta))$ blocks\footnote{When referring to the number of median-of-means blocks, we slightly abuse notation throughout and hide ${\rm poly}(\log T)$ factors in the $\widetilde{O}(\cdot)$ notation for ease of presentation.} rather than $\widetilde{O}(\log(\lvert\cA\rvert/\delta))$.

\end{itemize}
The price paid for this simplification is a slightly looser confidence width $W_h = \sqrt{2d}\,\eps_h$, arising from bounding the error propagation through $\hat{\theta}_h$ (via Corollary~\ref{cor:sum_rho_times_b_G_inv_a} in Appendix~\ref{app:conc_inequalities}). Pseudocode is given in Algorithm~\ref{alg:exp_action_set} in Appendix~\ref{app:proof_exp_upper}, and the full analysis is also given in Appendix~\ref{app:proof_exp_upper}.

\subsection{Algorithm~\ref{alg:poly_action_set}: \texorpdfstring{$\lvert\cA\rvert$}{|A|}-Dependent Regret} \label{sec:poly_alg}

For smaller (e.g., polynomial-sized) action sets, the extra $\sqrt{d}$ penalty in the confidence width of Algorithm~\ref{alg:exp_action_set} prevents us from reaching minimax optimality. Algorithm~\ref{alg:poly_action_set} (detailed in Appendix~\ref{app:proof_poly_upper}) rectifies this by allocating refinement samples \emph{proportionally to $\rho_h$} (see~\eqref{eq:near_optimal_design}). By doing so, the resulting variances $\Var(\hat{Y}_{a,h}^{(j)}) \propto 1/\rho_h(a)$ perfectly cancel against the design weights when propagating through the block-wise parameter estimates
\begin{equation} \label{eq:theta_hat_poly}
    \hat{\theta}_h^{(j)} = G_h^{-1}\sum_{a \in \supp(\rho_h)} \rho_h(a) \cdot a \cdot \hat{Y}_{a,h}^{(j)}
    \qquad \text{where} \qquad
        G_h = \sum_{a \in \supp(\rho_h)} \rho_h(a)\,aa^\top,
\end{equation}
yielding the tighter confidence width $W_h = \eps_h$. However, because median-of-means must now be applied to the inner products $\hat{\mu}_{b,h} = \operatorname{median}\{\langle b, \hat{\theta}_h^{(j)}\rangle\}_{j=1}^{K}$ simultaneously for all $b \in \cA_h$, ensuring concentration requires a union bound over the active set. Under our finite-variance assumption, this necessitates using $K = \widetilde{O}(\log(\lvert\cA\rvert/\delta))$ blocks, contributing the $\sqrt{\log\lvert\cA\rvert}$ factor to the regret.

\textbf{Small-$B$ regime.} Let $\tau = \widetilde{O}(d\log(\lvert\cA\rvert/\delta))$ denote the ``threshold'' between the small-$B$ regime and large-$B$ regime (see Line \ref{algline:tau} of Algorithm \ref{alg:poly_action_set} for a precise definition). When $B  < \tau $, we simply initialize $\cA_1 = \cA$ and run the common template above for all epochs. While this incurs an initialization cost of $\widetilde{O}(dBK)$, the restricted size of $B$ ensures this cost is bounded by $\widetilde{O}(d^{3/2}\sqrt{B}\log^{3/2}(\lvert \cA \rvert/\delta))$. 

\textbf{Large-$B$ regime.} When $B \geq \tau $, an initialization term of $\widetilde{O}(dBK)$ becomes prohibitive, with the trivial global gap bound $\Delta_b \le 2$ forcing significant over-exploration when $B \gg d$. We circumvent this by prepending two pre-processing steps that restrict the effective suboptimality gap of every arm entering $\cA_1$ from $O(1)$ down to $O(\sqrt{d/B})$:
\begin{itemize}[leftmargin=4ex]
    \item \emph{Safe arm identification} (Algorithm~\ref{alg:safe_arm}). Through at most $\lceil\log_2\lvert\cA\rvert\rceil$ rounds of tournament-style bisection, we iteratively split the candidate set in half, using a single bit to indicate which half has the higher estimated maximum mean reward, and keep only the winning half. Within each round, the agent executes a $G$-optimal design \emph{within a single batch}, pulling multiple core-set arms simultaneously and computing their unquantized sample means, in contrast to the single-arm batches used in the main phased elimination loop. After at most $\lceil\log_2\lvert\cA\rvert\rceil$ rounds, the surviving arm $a_{\rm safe}$ has suboptimality gap at most $\widetilde{O}(\log^{3/2}(\lvert\cA\rvert/\delta)\sqrt{d/B})$.
    This step uses $O(B\log\lvert\mathcal{A}\rvert)$ arm pulls.

    \item \emph{Warm start} (Algorithm~\ref{alg:warm_start}). We run a single warm-start epoch of phased elimination in which $a_{\rm safe}$ fills every slot not used by a core-set arm. This produces a refined active set $\cA_1 \subseteq \cA$ whose maximum gap is $O(\sqrt{d/B})$. The key insight is that $a_{\rm safe}$ absorbs the vast majority of slots in each batch, effectively decoupling the exploration budget from the rigid batch size constraint: core-set arms are pulled only $\widetilde{O}(B\log(\lvert\cA\rvert/\delta))$ times in total. Together with the safe-arm gap bound, this yields the first two terms in~\eqref{eq:general_d_regret}, in place of the $\widetilde{O}(dBK)$ initialization regret incurred without warm starting.
\end{itemize}

Following these pre-processing steps, the main phased elimination loop proceeds on $\cA_1$ according to the common template. Pseudocode for the main loop is given in Algorithm~\ref{alg:poly_action_set}, with the pre-processing subroutines detailed in Algorithms~\ref{alg:safe_arm} and~\ref{alg:warm_start}. See Appendix~\ref{app:proof_poly_upper} for the detailed regret analysis.
While our problem formulation does not limit the amount of downlink communication, Appendix \ref{sec:downlink} summarizes what needs to be sent from the learner to the agent in each batch.

\section{Conclusion} \label{sec:conclusion} 
 We have studied the problem of batched stochastic linear bandits with a `batch-wise memoryless' agent and a $1$-bit communication constraint. We presented two algorithms and provided a new minimax lower bound in this setting. For exponential-sized action sets, we provided an algorithm that matches our lower bound on the regret up to logarithmic factors, while for smaller action sets, we provided similar guarantees under broad scalings of the batch size, dimension, and time horizon.  Notably, the regret from the standard setting can be matched to within logarithmic factors even in certain severely communication-constrained scenarios such as ``1 bit per length-$\sqrt{T}$ batch''. 
 Perhaps the most immediate direction for future work is to devise an algorithm that attains near-optimal regret in the remaining scaling regimes where a gap remains.


\section*{Declaration of LLM Usage}
The authors determined the overall research direction and led the development of the main results and proofs.
Large language models (LLMs) were used to help draft and organize the submitted version of the manuscript. During the camera-ready revision, they also assisted with checking and refining technical
arguments, including the conditioning arguments in
Appendices~\ref{app:proof_exp_upper} and~\ref{app:proof_poly_upper}
and the extension to multiple bits per batch in
Appendix~\ref{app:q_bits}.
The authors reviewed and revised the LLM-assisted material
and take full responsibility for the final manuscript and its results.

\section*{Acknowledgement}

This work is supported by the Singapore National Research Foundation (NRF) and the Ministry of Digital Development and Information (MDDI) under the AI Visiting Professorship programme (Award AIVP-2024-003).



\bibliography{ref}
\bibliographystyle{apalike}

\newpage


\appendix

\section{Concentration and Experimental Design Inequalities}
\label{app:conc_inequalities}
Here we provide formal statements of several inequalities used throughout the paper. We begin by presenting the theoretical guarantee of the median of means estimator, with the precise statement being taken from \cite[Theorem 2]{lugosi2019mean}.

\begin{lemma} \label{lem:mom_statement}
    Let $X_1,\dots,X_n$ be a sequence of i.i.d. random variables each with mean $\mu$ and variance $\sigma^2$. Suppose that $n = mk$ for a positive integer $m$ and $k = \lceil 8 \log (\frac{1}{\delta})\rceil$. Then for any $\delta \in (0,1)$, the median-of-means estimator $\hat{\mu}_{\rm MoM}$ with $k$ blocks satisfies
    \begin{equation}
        \label{eq:mom_guarantee} \lvert \hat{\mu}_{\rm MoM} - \mu \rvert \leq \sqrt{\frac{32\sigma^2}{n}\log \Big(\frac{1}{\delta}\Big)}
    \end{equation}
    with probability at least $1-\delta$.
\end{lemma}

The next two results concern experimental design properties that will be used throughout the subsequent proofs. While these results are well-known, we include their proofs for the sake of completeness.

\begin{lemma}
\label{lem:sum_rho_times_squared_b_G_inv_a}
    Fix an arbitrary finite action set $\cX \subset \bR^d$ and let $\rho \colon \cX \to [0, 1]$ be a design satisfying~\eqref{eq:near_optimal_design} with respect to $\cX$. For each action $b \in \cX$, we have
    \[
        \sum\limits_{a \in \supp(\rho)} \rho(a) \cdot (b^\top G^{-1} a)^2 \le 2d.
    \]
\end{lemma}
\begin{proof}
    This follows from basic algebraic manipulations and the first property in~\eqref{eq:near_optimal_design}:
    \begin{align*}
           \sum\limits_{a \in \supp(\rho)} \rho(a) \cdot (b^\top G^{-1} a)^2 
        &= \sum_{a \in \supp(\rho)} \rho(a) \cdot (b^\top G^{-1} a) \cdot (b^\top G^{-1} a)  \\
        &= \sum_{a \in \supp(\rho)} \rho(a) \cdot (b^\top G^{-1} a) \cdot (a^\top G^{-1} b)  \\
        &=  b^\top G^{-1} \underbrace{\left( \sum_{a \in \supp(\rho)} \rho(a)  a a^\top \right)}_{= \,G} G^{-1} b\\
        &= b^\top G^{-1} b \\
        &\le 2d
    \end{align*}
    as desired.
\end{proof}

\begin{corollary}
    \label{cor:sum_rho_times_b_G_inv_a}
    Under the same setup as in Lemma~\ref{lem:sum_rho_times_squared_b_G_inv_a}, we have for each action $b \in \cX$ that
    \[
         \sum_{a \in\supp(\rho)} \rho(a) \cdot \lvert b^\top G^{-1} a \rvert\le \sqrt{2d}.
    \]
\end{corollary}
\begin{proof}
    By the Cauchy-Schwarz inequality and $\sum_a\rho_h(a) = 1$, we have
\begin{align*}
    \sum_{a \in \supp(\rho)} \rho(a) \cdot \lvert b^\top G^{-1} a \rvert
  &= \sum_{a \in \supp(\rho)} \left( \sqrt{\rho(a)} \right) \cdot \left( \sqrt{\rho(a)} \cdot \lvert b^\top G^{-1} a \rvert \right) \\
  &\le \sqrt{\sum_{a\in \supp (\rho)} \rho(a)} \cdot \sqrt{\sum_{a \in \supp(\rho)} \rho(a) \cdot (b^\top G^{-1} a)^2} \\
  &\le \sqrt{2d},
\end{align*}
where the last step follows from Lemma~\ref{lem:sum_rho_times_squared_b_G_inv_a}.
\end{proof}

\section{Details of 1-Bit Mean Estimation Subroutines}
\label{app: 1bit subroutines}
 
The subroutines in this section operate on a sequence of i.i.d. scalar
observations $X_1, X_2, \ldots$ with mean $\mu \coloneqq \E[X] \in [-1,1]$
and variance $\mathrm{Var}(X) \le \sigma'^2$, where $\sigma' > 0$ is the
\emph{effective standard deviation} of a single observation. 

The sample complexity is stated in terms of the number of 1-bit threshold \emph{queries}: one query presents a single observation $X_j$ to the agent, who returns the single bit $\mathds{1}\{X_j \ge \gamma\}$ (or $\mathds{1}\{X_j\le\gamma\}$)  for a learner-chosen threshold $\gamma$.
The cost in \emph{raw arm pulls} equals the number of queries multiplied by
the number of arm pulls per observation, which varies depending on the context.
In our algorithms, two settings arise:
\begin{itemize}[leftmargin=4ex]
    \item \textbf{Full-batch queries} (Algorithm~\ref{alg:exp_action_set}
      and Step~3 of Algorithm~\ref{alg:poly_action_set}): In each batch the
      agent pulls arm $a$ exactly $B$ times, computes the batch average
      $\bar{Y}_B = \tfrac{1}{B}\sum_{i=1}^B Y_i$, and returns a single bit.
      Thus
      \[
        X = \bar{Y}_B,
        \qquad
        \sigma' = \frac{\sigma}{\sqrt{B}},
        \qquad
        \text{and one query costs } B \text{ raw pulls of arm } a.
      \]
    \item \textbf{Partial-batch queries} (Algorithm~\ref{alg:warm_start},
      Warm-Start epoch): For arm $a \in \mathrm{supp}(\rho_0)$ with allocation
      $u_0(a) = \lceil B\rho_0(a)\rceil$, the agent pulls arm $a$ exactly
      $u_0(a)$ times per batch (filling the remaining $B - u_0(a)$ slots with
      $a_{\mathrm{safe}}$), computes the average of those $u_0(a)$ rewards,
      and returns a single bit.  Thus
      \[
        X = \bar{Y}_{u_0(a)},
        \qquad
        \sigma' = \frac{\sigma}{\sqrt{u_0(a)}},
        \qquad
        \text{and one query costs } u_0(a) \text{ raw pulls of arm } a.
      \]
\end{itemize}
Both subroutines below apply verbatim to either setting once $\sigma'$ is
identified from the context.\footnote{Recent
work~\cite{lau2026non,miao2026universal,hu2026interaction}
attains the order-optimal 1-bit mean-estimation rates with all
localization and refinement queries chosen together in advance (i.e., non-adaptive queries), using general 1-bit queries instead of simple threshold queries.  
We retain the simpler adaptive threshold-query construction below,
since adaptation across batches is permitted in our model.}
 
\subsection{\textsc{1BitLocalize}: Coarse Interval Estimation}
\label{subsec:1Bitlocalize}

When our bandit algorithms invoke $[L_a, U_a] \leftarrow \textsc{1BitLocalize}(a, \delta_{\mathrm{loc}}, \sigma')$, the subroutine executes the localization protocol of~\cite{lau2026sequential} over $[-1,1]$. This procedure treats each scalar observation $X_j$ as a single noisy measurement of $\mu_a$ with standard deviation~$\sigma'$.

Specifically, the protocol uses existing median estimation techniques such as noisy binary search (e.g.,~\cite{gretta2023sharp}) to first locate a high-probability confidence interval $[L_a', U_a']$ containing the median $M$ of the observation clipped to $[-1,1]$
using 1-bit threshold queries.\footnote{Clipping does not increase the variance or the distance of a median from $\mu_a\in[-1,1]$.} Then, by leveraging the well-known mean-median inequality $|\E[X] - M| \le \sigma'$, we naturally obtain a high-probability confidence interval $[L_a,U_a]
=\bigl[\max\{-1,L_a'-\sigma'\},\,
       \min\{1,U_a'+\sigma'\}\bigr]$
containing the true mean $\mu_a$. See~\cite[Appendix~A, Step~1]{lau2026sequential} for the full derivation leading to the following guarantee.
 
\begin{proposition}[\textsc{1BitLocalize} guarantee]
\label{prop:localize}
There exists an absolute constant $C_{\mathrm{loc}} > 0$ such that $\textup{\textsc{1BitLocalize}}(a, \delta_{\mathrm{loc}}, \sigma')$ uses at most
\[
  n_{\mathrm{loc}}
  = C_{\mathrm{loc}}\left(\log \frac{1}{\sigma'}
    + \log \frac{1}{\delta_{\mathrm{loc}}}\right)
\]
queries\footnote{For ease of presentation, we assume  $\sigma'$ is strictly less than 1 throughout, so that the $\log(1/\sigma')$ term is non-negative.  If $\sigma' \ge 1$, we can simply skip the localization step and return $[L_a, U_a] = [-1,1]$, since $\lvert [-1,1]\rvert \le 2\sigma'$ in this case.} and returns an interval $[L_a, U_a]$ satisfying
\[
  \P\left(
    \mu_a \in [L_a, U_a]
    \;\;\text{and}\;\;
    U_a - L_a \le 8\sigma'
  \right)
  \ge 1 - \delta_{\mathrm{loc}}.
\]
The corresponding raw sample cost (i.e., total number of arm pulls) is $n_{\mathrm{loc}}$ multiplied by the
number of raw arm pulls per query. 
\end{proposition}

\subsection{\textsc{1BitRefine}: Refined Mean Estimation}
\label{subsec:1bitRefine}
 
Fix the complete history immediately before refinement, including its inputs, and suppose that $\mu_a\in[L_a,U_a]$ and $U_a-L_a\le8\sigma'$. All probabilities, expectations, variances, and biases in this subsection are taken under the conditional law given this history. Refinement uses fresh i.i.d.\ observations and independent threshold randomization.
Set $X'\coloneqq X-(L_a+U_a)/2$. Then $X'$ has mean $\mu'\coloneqq\mu_a-(L_a+U_a)/2$, with $|\mu'|\le4\sigma'$ and $\mathrm{Var}(X')\le\sigma'^2$. The base estimator proceeds as follows.

\paragraph{1. Modified Cutoff and Partitioning.}
Fix a target accuracy $\eps > 0$.  We select a truncation threshold
\begin{equation}
\label{eq:cutoff}
  t = 2^{i_{\max}} \cdot \sigma',
  \quad \text{where} \quad
  i_{\max} \coloneqq \left\lceil \log_2 T^2 \right\rceil,
\end{equation}
so that $t \ge T^2\sigma'$.  Following \cite{lau2026order}, we then partition $[-t, t)$ into non-overlapping regions
    $R_1, R_{-1}, R_2, R_{-2}, \ldots, R_{\imax}, R_{-\imax}$ 
    defined by exponentially growing interval boundaries $m_i \sigma'$: 
    \begin{equation}
      R_i  =
      \begin{cases}
        \sigma' \cdot \left[m_{i-1} , m_i  \right) & \text{if }  i \ge 1
        \\ \\
        \sigma' \cdot \left[-m_{|i|},-m_{|i|-1}\right) & \text{if } i \le -1
      \end{cases}
      \quad \text{where} \quad
      m_0 = 0 \text{ and } 
      m_i = 2^{i} \text{ for } i \ge 1.
    \end{equation}
These half-open intervals assign every point of $[-t,t)$ to exactly one region. 
 
Using a similar argument in~\cite[Appendix A, Step 2]{lau2026order}, we can show that treating the insignificant region $\{X'\notin[-t,t)\}$ as zero introduces a truncation bias
bounded by $O(\sigma'/T^2)$. We include the relevant steps here for  completeness. For $T \ge 3$, since $|\mu'| \le 4\sigma'$ under the localization event, we have 
\[
    t \ge T^2\sigma' \ge 8\sigma' \ge 2|\mu'|
    \implies t - \lvert \mu' \rvert \ge \frac{t}{2}
    \implies \{|X'| \ge t \} \subseteq \{|X'-\mu'| \ge t/2 \}.
\]
Using this event inclusion together with the triangle inequality and Chebyshev's inequality, we obtain
\begin{align}
  \left|\E\left[X' \cdot \mathds{1}\{X'\notin[-t,t)\}\right]\right|
  &\le
  \E\left[|X'-\mu'| \cdot
      \mathds{1}\left\{|X'| \ge t\right\}\right]
    + \E\left[|\mu'| \cdot
      \mathds{1}\left\{|X'| \ge t\right\}\right]     \notag  \\
   &\le \E\left[|X'-\mu'| \cdot
      \mathds{1}\Big\{|X'-\mu'| \ge \frac{t}{2}\Big\}\right]
    + |\mu'| \cdot \P(|X'| \ge t)       \notag \\
    &\le \E\left[ \frac{|X'-\mu'|^2}{t/2}\right]
    + |\mu'| \cdot \P(|X'| \ge t)       \notag \\
  &\le
    \frac{\sigma'^2}{t/2}
    + 4\sigma' \cdot \frac{\sigma'^2}{(t/2)^2} \notag \\
  &\le
    \frac{2\sigma'}{T^2} + \frac{16\sigma'}{T^4} \notag \\
  &\le
    \frac{3\sigma'}{T^2}, \label{eq:trunc_bias}
\end{align}
where the second last inequality follows from $t \ge T^2 \sigma'$ and the last inequality holds for $T \ge 4$.  We note that the choice of truncation threshold $t \ge T^2\sigma'$ makes the bias bound depend solely on~$T$ and $\sigma'$, thus decoupling bias from the target accuracy parameter $\eps$.  
This is the key modification we make relative to~\cite{lau2026order},
where the authors chose $t = \Theta(\sigma'^2/\eps)$ to achieve a bias of
$O(\eps)$, which is sufficient for their purpose of 1-dimensional mean estimation.
 
\paragraph{2. Region-Wise Estimation via Threshold Queries.}
For each region $R_i=[a_i,b_i)$, we estimate its contribution $\mu_i'\coloneqq\E[X' \cdot \mathds{1}\{X'\in R_i\}]$ using the randomized threshold queries of~\cite[Sec.~2.1, Step~4]{lau2026order}. Fix an absolute constant $C\ge243$ and set $n =\lceil C\sigma'^2/\eps^2\rceil$, the number of observations for each query type. With $T_i\sim\mathrm{Unif}(a_i,b_i)$ independent of $X'$, the four query types are
\begin{equation}\label{eq:regional_queries}
  \mathds{1}\{X'\ge a_i\},\quad
  \mathds{1}\{X'\ge T_i\},\quad
  \mathds{1}\{X'<b_i\},\quad
  \mathds{1}\{X'<T_i\}\,.
\end{equation}
Each query type uses $n$ fresh observations. Different query types and regions use disjoint observations, and every random-threshold query draws a fresh independent $T_i$. Once the input interval and parameters are fixed, all query
thresholds can be chosen before any refinement responses are observed; therefore the refinement stage is non-adaptive. Strict comparisons are obtained by complementing the corresponding non-strict threshold bits.\footnote{Alternatively, the interval queries $\mathds{1}\{a_i\le X'<T_i\}$ and $\mathds{1}\{T_i\le X'<b_i\}$, also permitted by our model, require only two query types per region instead of four; see~\cite[Appendix~A, Steps~4--5]{lau2026sequential}.}

Write $\hat p_{i,1},\ldots,\hat p_{i,4}$ for the empirical means of the query responses in~\eqref{eq:regional_queries}, in the displayed order. The local estimate for region $R_i$ is
\begin{equation}\label{eq:regional_estimator}
  \hat\mu_i'
  =a_i(\hat p_{i,1}-\hat p_{i,2})
   +b_i(\hat p_{i,3}-\hat p_{i,4})\,.
\end{equation}
The final shifted estimate aggregates these local estimates across the entire collection of regions up to $\imax$:
\begin{equation*}
  \hat\mu_a
  =\sum_{1\le|i|\le i_{\max}}\hat\mu_i'
   +\frac{L_a+U_a}{2},
\end{equation*}
which satisfies the following guarantee.
\begin{lemma}[\textsc{1BitRefine} fixed-accuracy]
\label{lem:refinement_fixed_eps}
Let $T\ge4$. Given inputs $T,\sigma',\eps$ as well as a valid interval $[L_a,U_a]$ containing $\mu_a$, with $U_a-L_a\le8\sigma'$, the estimator described above uses at most
\begin{equation}\label{eq:n_base}
  n_{\mathrm{base}}
  =8\left\lceil\frac{C\sigma'^2}{\eps^2}\right\rceil
    \left\lceil2\log_2T\right\rceil
  =O\left( \left(1+ \frac{\sigma'^2}{\eps^2} \right) \log T\right)
\end{equation}
queries, and its output $\hat\mu_a$ satisfies
\begin{equation}\label{eq:1bit_fixed_accuracy}
  \mathrm{Var}(\hat\mu_a)\le\eps^2
  \qquad\text{and}\qquad
  |\mathrm{Bias}(\hat\mu_a)|\le\frac{3\sigma'}{T^2}\,.
\end{equation}
\end{lemma}

\begin{proof}
There are four query types in each of the $2i_{\max}$ regions, so the total query count is $4 \cdot 2i_{\max} \cdot n=n_{\mathrm{base}}$, as in~\eqref{eq:n_base}.

For a fixed realization $X'=x\in R_i$, averaging over $T_i\sim\mathrm{Unif}(a_i,b_i)$ gives
\[
    \E[\mathds{1}\{x\ge a_i\}-\mathds{1}\{x\ge T_i\}]
  =1-\P(T_i\le x)=\frac{b_i-x}{b_i-a_i}
\]
and
\[
    \E[\mathds{1}\{x<b_i\}-\mathds{1}\{x<T_i\}]
  =1-\P(T_i>x)=\frac{x-a_i}{b_i-a_i}.
\]
For $x\notin R_i$, the two indicators in each pair above coincide, so both differences vanish. Since each empirical mean is unbiased (i.e., its expectation is precisely the corresponding query-response probability),
taking expectations in~\eqref{eq:regional_estimator} yields the unbiasedness of the local estimates:
\begin{equation*}
  \E[\hat\mu_i']
  =\E\left[
     \frac{a_i(b_i-X')+b_i(X'-a_i)}{b_i-a_i}
     \cdot \mathds{1}\{X'\in R_i\}\right]
  =\E[X' \mathds{1}\{X' \in R_i\}] 
  = \mu_i'.
\end{equation*}
The regions $R_{\pm 1}, R_{\pm 2}, \ldots, R_{\pm \imax}$ partition $[-t,t)$, so the bias of $\hat\mu_a$ is $-\E[X'\mathds{1}\{X'\notin[-t,t)\}]$. The claimed bound on this bias in \eqref{eq:1bit_fixed_accuracy} follows directly from~\eqref{eq:trunc_bias}.

For the variance, we sum the independent contributions across the regions, following~\cite[Appendix~A, Step~5(a)--(b)]{lau2026order}. Since $\hat p_{i,1},\ldots,\hat p_{i,4}$ are independent (due to relying on disjoint sets of samples), the variance expands as   
\[
    \Var(\hat{\mu}'_i) = a_i^2 \Var(\hat{p}_{i,1}) + a_i^2 \Var(\hat{p}_{i,2}) + b_i^2 \Var(\hat{p}_{i,3}) + b_i^2 \Var(\hat{p}_{i,4}).
\]
For the two central regions $i \in \{-1, 1\}$, we have $a_1=b_{-1}=0$ and
$b_1=-a_{-1}=2\sigma'$. Since each empirical Bernoulli
mean~$\hat{p}_{i,k}$ has variance at most~$1/n$, these regions contribute
at most~$16\sigma'^2/n$ to the variance.

For outer-right regions $i \ge 2$, a response of $1$ to either of the first two queries in~\eqref{eq:regional_queries}, or $0$ to either of the last two, requires $X'\ge a_i$. Hence, the Bernoulli variance bound $p(1-p)\le\min\{p,1-p\}$ gives $\Var(\hat p_{i,k})\le\P(X'\ge a_i)/n$ for $k=1,\ldots,4$. Since $a_i=m_{i-1}\sigma'$ and $b_i=2a_i$, we obtain
\begin{equation*}
  \Var(\hat\mu_i')
  \le\frac{2a_i^2+2b_i^2}{n}\P(X'\ge a_i)
  =\frac{10m_{i-1}^2\sigma'^2}{n}
    \P(X'\ge m_{i-1}\sigma')\,.
\end{equation*}
For the outer-left regions $R_{-i}$ with $i \ge 2$, the same argument uses $\P(X'\le-m_{i-1}\sigma')$, with the roles of responses $0$ and $1$ reversed. Adding these bounds yields
\begin{equation}\label{eq:regional_variance}
    \Var(\hat{\mu}'_i) + \Var(\hat{\mu}'_{-i}) \le \frac{10m_{i-1}^2\sigma'^2}{n} \P(|X'|\ge m_{i-1}\sigma').
\end{equation}
Summing~\eqref{eq:regional_variance} over all $i \ge 2$ and adding the central-region bound gives
\begin{equation}
\label{eq:total_variance}
    \Var(\hat\mu_a) \le \frac{16\sigma'^2}{n} + \frac{10}{n} \sum_{i=2}^{i_{\max}} m_{i-1}^2 \sigma'^2 \cdot \P(|X'|\ge m_{i-1}\sigma').
\end{equation}
We write
\[
    m_{i-1}^2 \sigma'^2  \cdot \P(|X'|\ge m_{i-1}\sigma') = \E[ m_{i-1}^2 \sigma'^2  \cdot \mathds{1}\{|X'|\ge m_{i-1}\sigma'\}],
\]
and bound the summation in \eqref{eq:total_variance} pointwise for any fixed realization $X' = x$. 
For any active index in the sum, the condition $|x| \ge m_{i-1}\sigma'$ ensures that $m_{i-1}^2\sigma'^2 \le x^2$. Because the thresholds double ($m_i = 2m_{i-1}$), summing in reverse from the largest active index bounds the sum by a geometric series that decays by a factor of $4$ at each step:
\begin{equation}\label{eq:refinement_geometric_sum}
  \sigma'^2 \sum_{i=2}^{i_{\max}} m_{i-1}^2 \mathds{1}\{|x| \ge m_{i-1}\sigma'\}
  \le x^2 \left(1 + \frac{1}{4} + \frac{1}{16} + \cdots\right)
  = \frac{4x^2}{3}\,.
\end{equation}
Taking the expectation of~\eqref{eq:refinement_geometric_sum} with respect to $X'$ and combining with~\eqref{eq:total_variance} yields
\[
    \Var(\hat\mu_a)
  \le \frac{1}{n}\left(16\sigma'^2 + \frac{40}{3}\E[X'^2]\right) 
  \le \frac{1}{n} \left(16\sigma'^2+ \frac{40 \left( \sigma'^2+ |\mu'|^2 \right)}{3} \right)
  \le \frac{C\sigma'^2}{n}
  \le \eps^2,
\]
where the last two inequalities follow from $|\mu'| \le 4 \sigma'$, $16+(40/3)\cdot17=728/3<C$, and $n\ge C\sigma'^2/\eps^2$. 
\end{proof}
 
In our phased elimination bandit algorithms, rather than specifying a target accuracy upfront, the learner allocates some fixed query budget $n_{\rm q}$ to each arm. By inverting the sample complexity bound in Lemma~\ref{lem:refinement_fixed_eps}, we obtain the following fixed-budget guarantee, which serves as the primary formal interface invoked by our bandit algorithms via the call $\hat{Y}_{a,h}^{(j)} \leftarrow \textsc{1BitRefine}(a, [L_a, U_a], n_{\rm q})$
 
\begin{corollary}[\textsc{1BitRefine} fixed-budget]
\label{cor:refinement_fixed_n}
Let~$T\ge4$, and let~$C$ be the constant from Lemma~\ref{lem:refinement_fixed_eps}. Given a fixed input interval~$[L_a,U_a]$ containing~$\mu_a$ with~$U_a-L_a\le8\sigma'$, suppose that~$n_{\rm q}\ge48\log_2T$ queries are available. Running the estimator above with target accuracy set to $\eps=\sigma'\sqrt{(48C\log_2T)/n_{\rm q}}$ yields an estimate~$\hat\mu_a$ satisfying

\begin{equation} \label{eq:1bit_fixed_budget}
  \mathrm{Var}(\hat\mu_a)
  \le \frac{C_V\sigma'^2\log_2 T}{n_{\rm q}}
  \qquad\text{and}\qquad
  \left|\mathrm{Bias}(\hat\mu_a)\right| \le \frac{3\sigma'}{T^2},
\end{equation}
where $C_V  = 48 C \ge 48$ is a universal absolute constant. 
\end{corollary}
 
\begin{proof}
We verify that the available query budget~$n_{\rm q}$ is sufficient to run the estimator in Lemma~\ref{lem:refinement_fixed_eps} at the chosen target accuracy $\eps$, i.e., $n_{\rm q}  \ge n_{\mathrm{base}}$. Since $n_{\rm q} \ge 48\log_2 T$, we have $C\sigma'^2/\eps^2 = n_{\rm q} /(48 \log_2 T) \ge 1$, so applying
$\lceil x\rceil \le 2x$ for $x \ge 1$, we get
\[
  n_{\mathrm{base}}
  \le 
  8 \left\lceil \frac{C \sigma'^2}{\eps^2}\right\rceil
    \left\lceil 2\log_2 T\right\rceil
  \le 8\cdot\left(\frac{2n_{\rm q}}{48\log_2 T}\right)\cdot 3\log_2 T
  = n_{\rm q}.
\]
Substituting $\eps$ into Lemma~\ref{lem:refinement_fixed_eps} gives
$\mathrm{Var}(\hat\mu_a) \le \eps^2 = C_V\sigma'^2\log_2 T/n_{\rm q}$, where $C_V = 48C \ge 48$.
The bias bound follows directly from~\eqref{eq:trunc_bias}, which holds under the fixed valid-transcript conditioning, regardless of $\eps$ or $n_{\rm q}$. Each completed call uses its full query allocation; any queries not needed by the estimator use fresh observations whose bits are simply discarded.
\end{proof}

\section{Proofs of Lower Bounds} 
\label{app:lower_bound}

\subsection{Proof of Theorem \ref{thm:minimax_lower} (General Lower Bound)} 

We first state a connection between Hamming distance and inner product, whose proof is immediate from their definitions (and thus omitted).
\begin{lemma}
\label{lem:ip-hamming}
Let $x,x'\in\{\pm1/\sqrt d\}^d$. If the Hamming distance $d_{\rm H}(x, x') \coloneq\sum_{i=1}^d \mathds{1}\{x_i \neq x_i'\}$ between $x$ and $x'$ is $h$, then
\[
\langle x,x'\rangle = 1 - \frac{2h}{d}.
\]
In particular, if $h\ge d/4$, then $\langle x,x'\rangle \le 1/2$ for each $x \ne x'$.
\end{lemma}

Therefore, in order to produce a set $\mathcal{V} \subset \{\pm1/\sqrt d\}^d$ with small pairwise inner products, it suffices to construct a binary code $\cC \subset \{0, 1\}^d$ with block length $d$ and Hamming distance at least $d/4$. The Gilbert--Varshamov (GV) bound guarantees the existence of such a code.

\begin{proposition}[{A specific instantiation of \cite{gilbert1952comparison, varshamov1957estimate}}]
\label{prop:gv}
For every integer $d$, there exists a binary code $\cC\subset\{0,1\}^d$ with minimum Hamming distance at least $\lceil d/4\rceil$ and size
\[
|\cC| \ge 2^{(1-H_2(1/4)) d - o(d)},
\]
where $H_2(p)=-p\log_2 p-(1-p)\log_2(1-p)$ is the binary entropy function. Numerically $1-H_2(1/4)\approx 0.1887$, so $|\cC| \geq 2^{cd}$ for a constant $c > 0$.
\end{proposition}

We now prove Theorem~\ref{thm:minimax_lower} using the bounds above.

\begin{proof}[Proof of Theorem~\ref{thm:minimax_lower}] 
Since instantaneous regrets are non-negative, we have $R_T \ge R_{B\lfloor T/B\rfloor}$. Therefore,  we may assume $T = BM$ for integer $M$ without loss of generality, making every batch exactly~$B$ pulls.
Fix the GV-rate constant $c>0$ from Proposition \ref{prop:gv} for $\delta=1/4$. We consider two cases of~$d$.

\paragraph{Case 1: $d < 4/c$ (small, constant dimension).}
We first consider $|\cA| = 2$ and then generalize to higher $|\cA|$.  Consider $\mathcal{A} = \{e_1, -e_1\} \subset \mathbb{R}^d$ with $e_1$ being the first standard basis vector, and draw
$\theta$ uniformly from $\mathcal{A}$.  The optimal arm is $\theta$
(with mean reward $+1$) and the only other arm is $-\theta$ (with mean reward
$-1$), so every suboptimal pull incurs regret $2$.  Before the first
batch the learner has seen zero bits of feedback, and hence has no information about $\theta$. By symmetry, every pull in batch 1 is suboptimal with probability $1/2$.  Summing over the $\min\{T,B\}$ pulls in the first
batch gives $\E[R_T]  \ge \min\{T,B\} \cdot 2 \cdot \frac{1}{2} = \min\{T,B\}$.  This matches the target lower bound when $|\cA| = 2$, because $\min(d, \log|\mathcal{A}|) = O(1)$, and so $\min\{T,B\} = \Omega(\min\{T, B\min(d,\log|\mathcal{A}|)\})$.

When $2 \le d < 4/c$ and $|\cA| > 2$, a similar argument applies by adding a further $|\cA|-2$ ``dummy'' arms whose first coordinate is zero, thus meaning these arms are strictly suboptimal and their rewards convey no information about which of $\{e_1,-e_1\}$ is optimal.  When $d=1$ and $|\cA| > 2$, a similar argument holds with the dummy arms having low value (e.g., in $[-1/2,1/2]$); a similar idea will be used in Case 2 below, so is skipped here.

For the rest of the proof, we assume $d \ge 4/c$, thus guaranteeing
$\lfloor 2^{cd}\rfloor \ge 16$.

\paragraph{Case 2: $d \ge 4/c$ (large dimension).} 
 To construct our hard instance for a given size $\lvert \cA \rvert = K$, we first define the size of our active hard subset: $K' = \min\{K, \lfloor 2^{cd} \rfloor\}$.
By Proposition \ref{prop:gv}, there exists a binary code $\cC \subset \{0,1\}^d$ with minimum distance at least $\lceil d/4 \rceil$ and size at least $2^{cd}$. We select an arbitrary subset of this code of size exactly $K'$, and map this subset to $\mathcal{V} \subset \{\pm 1/\sqrt{d}\}^d$ via the bijection $0 \mapsto -1/\sqrt{d}$ and $1 \mapsto +1/\sqrt{d}$. Let this hard subset be $\cA'$.

To expand this to the full size $K$, we add $K - K'$ arms that are distinct but suboptimal by at least $1/2$. Since $\cB(1/2) = \{ x \in \mathbb{R}^d: \|x\|_2 \le 1/2 \}$, i.e., the ball of radius $1/2$, contains uncountably many vectors, we can construct the remaining $K-K'$ actions by selecting arbitrary distinct vectors with an $L_2$ norm of at most $1/2$. Let this padding set be $\cA_{\rm pad}$, and define our action set as $\cA = \cA' \cup \cA_{\rm pad}$, so $\lvert \cA \rvert = K$.

To lower bound the worst-case expected regret, we apply Yao's minimax principle and instead lower bound the expected regret of every deterministic algorithm when the true parameter
$\Theta$ is drawn uniformly from the hard subset $\cA'$. Note that $\|\Theta\|_2 = 1$.

Under this prior, the unique optimal arm is $\Theta$, yielding a mean reward of $\|\Theta\|_2^2 = 1$. If the learner plays any suboptimal arm $a \in \cA' \setminus \{\Theta\}$, the mean reward is at most $1/2$ by Lemma~\ref{lem:ip-hamming} and the Gilbert-Varshamov code construction. If the learner plays any padding arm $a \in \cA_{\rm pad}$, the mean reward is at most $1/2$ by the Cauchy-Schwarz inequality:
\[
    \lvert \langle a, \Theta \rangle\rvert  \le \|a\|_2 \|\Theta\|_2 \le \frac{1}{2}(1) = \frac{1}{2}.
\]
Thus, any arm pulled from $\cA$ that is not exactly the true parameter $\Theta$ incurs an instantaneous regret of at least $\Delta = 1/2$.

Fix a batch index $i \in \{1,\dots,M\}$. Let $\bfQ_{i-1} = (Q_1,\dots, Q_{i-1})$ denote the sequence of 1-bit messages revealed after the first $i-1$ batches. By the 1-bit communication constraint, the history $\bfQ_{i-1}$ consists of exactly $i-1$ bits. Therefore, its entropy and hence its mutual information with $\Theta$ is trivially bounded by
\[
    I(\Theta;\bfQ_{i-1}) \le H(\bfQ_{i-1}) \le i-1.
\]
In batch $i$, the learner selects $B$ actions $A_{i,1}, \dots, A_{i,B}$. Because the algorithm cannot observe rewards mid-batch, every action $A_{i,t}$ is a deterministic function of the history $\bfQ_{i-1}$. Since $\Theta$ is drawn uniformly from $K'$ possibilities, we can apply Fano's inequality~\cite[Corollary 3.13]{Polyanskiy_Wu_2025} to bound  the probability that any specific pull successfully selects the optimal arm:
\[
    \bP\big(A_{i,t} = \Theta\big) \le \frac{I(\Theta;\bfQ_{i-1}) + \log_2 2}{\log_2 K'} \le \frac{(i-1)+1}{\log_2 K'} = \frac{i}{\log_2 K'}.
\]
Thus, the probability that a pull in batch $i$ is suboptimal is at least $1 - \frac{i}{\log_2 K'}$. Since any suboptimal pull incurs a gap of at least $1/2$, the expected instantaneous regret for \emph{every single pull} in batch $i$ is at least 
\[
    \frac{1}{2} \cdot \left( 1- \frac{i}{\log_2 K'} \right).
\]
Let $L = \lfloor \log_2 K' \rfloor$. Because $K \ge 16$ and $\lfloor 2^{cd}\rfloor \ge 16$ (since $d \ge 4/c$), we have $K' \ge 16$, meaning $L \ge 4$. We consider two cases for the total number of batches $M = T/B$: (i) $M \ge L$, and (ii) $M < L$.

For case (i), the expected regret generated by the $B$ pulls in each batch $i = 1, \dotsc, L$ is at least $\frac{B}{2}\big(1-\frac{i}{\log_2 K'}\big)$. Therefore, the expected regret over just the first $L$ batches is
\[
    \E_{\Theta \sim{\rm Unif}(\cA')}[R_{L B}] \ge \sum_{i=1}^L \frac{B}{2}\Big(1-\frac{i}{\log_2 K'}\Big) = \frac{B}{2}\Big(L - \frac{L(L+1)}{2 \log_2 K'}\Big).
\]
Because $L = \lfloor \log_2 K' \rfloor$, we know $\frac{L+1}{\log_2 K'} \le \frac{\log_2 K' + 1}{\log_2 K'} = 1 + \frac{1}{\log_2 K'}$. Since $K' \ge 16$, this ratio is at most $1 + 1/4 = 1.25$, and we obtain
\[
    \E_{\Theta \sim {\rm Unif}(\cA')}[R_{T}] \ge \E_{\Theta \sim{\rm Unif}(\cA')}[R_{L B}] \ge \frac{B}{2}\Big(L - \frac{1.25 L}{2}\Big) = \frac{B}{2} (0.375 L) = \Omega(B \log K').
\]

For case (ii), we sum the per-batch regret bound over all $M$ batches, yielding
\[
    \E_{\Theta \sim {\rm Unif}(\cA')}[R_{T}] \ge \sum_{i=1}^M \frac{B}{2}\Big(1-\frac{i}{\log_2 K'}\Big) = \frac{B}{2}\Big(M - \frac{M(M+1)}{2 \log_2 K'}\Big).
\]
Since $M$ and $L$ are integers with $M < L$, we have $M+1 \le L \le \log_2 K'$. This implies $\frac{M+1}{\log_2 K'} \le 1$. Consequently, we have
\[
    \E_{\Theta \sim{\rm Unif}(\cA')}[R_{T}] \ge \frac{B}{2}\Big(M - \frac{M}{2}\Big) = \frac{BM}{4} = \Omega(T).
\]
Combining both cases, the expected regret is $\E_{\Theta \sim {\rm Unif}(\cA')}[R_T] = \Omega(\min\{T, B \log K'\})$. By Yao's minimax principle, the worst-case expected regret is lower bounded by this expected regret. Substituting $K' = \min\{K, \lfloor 2^{cd}\rfloor\}$ and noting that $\log \lfloor 2^{cd} \rfloor = \Theta(d)$, we conclude that
\[
    \sup_{\|\theta\|_2 \le 1} \E[R_T] \ge 
    \E_{\Theta \sim {\rm Unif}(\cA')}[R_T] = \Omega\left(\min\{T, B\min(d,\log \lvert \cA \rvert)\}\right).
\]
This completes the proof. 
\end{proof}

\subsection{Proof of Theorem \ref{thm:non_mixing_lb} (Lower Bound for Non-Mixing Algorithms)} 
\begin{proof}[Proof of Theorem~\ref{thm:non_mixing_lb}]
To lower bound the worst-case expected regret, we again apply Yao's minimax principle, restrict attention to deterministic algorithms, and assume that the adversary draws
the true parameter $\Theta$ uniformly from
$\cA = \{e_1,\dots, e_d\}$. Under this prior, the unique optimal arm is $\Theta$ with a mean reward of $1$, while any other arm $e_j \neq \Theta$ yields a mean reward of $0$.

To emphasize the severity of the non-mixing constraint, we consider a strictly easier environment: We remove any noise (i.e., consider $\sigma=0$) and allow the agent to transmit the exact, noiseless rewards to the learner, completely bypassing the 1-bit communication constraint. Any lower bound in this idealized setting trivially applies to the communication-constrained (and possibly noisy) problem.

In this noiseless setting, pulling (or ``querying'') a single arm $A \in \mathcal{A}$ for a batch yields $B$ identical observations of $\mathds{1}\{A = \Theta\}$. Because the vectors in $\cA$ are orthogonal, pulling any $A \neq \Theta$ yields a reward of $0$. Thus, a query only reveals whether $A = \Theta$ or $A \neq \Theta$. Until the optimal arm is successfully queried, the learner receives a sequence of identical $0$s.

Let $S_i$ be the set of arms queried over the first $i$ batches. Because the algorithm is non-mixing, it can test at most one arm per batch, and thus $|S_i| \le i$. Since the sequence of $0$s provides no distinguishing information about the unqueried arms, the posterior distribution of $\Theta$ remains uniform over the remaining unqueried arms. Testing a previously unqueried arm in each batch therefore finds
$\Theta$ within~$i$ batches with probability $i/d$ for $i\le d$,
and repeated queries cannot increase this probability. Hence,
\[
    \bP(\Theta \in S_i) \le \frac{i}{d}.
\]
Let $M_0 = \lfloor d/2 \rfloor$. For any number of batches $i \le M_0$, the probability that the algorithm has \emph{not} yet found the optimal arm is bounded by:
\[
    \bP(\Theta \notin S_i)  = 1 - \bP(\Theta \in S_{i}) 
    \ge 1 - \frac{i}{d} \ge 1 - \frac{\lfloor d/2 \rfloor}{d} \ge \frac{1}{2}.
\]
If $\Theta \notin S_i$, the algorithm has failed to find the optimal arm in the first $i$ batches. For any $i < \lceil T/B \rceil$, this corresponds to $iB$ suboptimal pulls, accumulating exactly $iB$ regret. Therefore, the expected regret over the first $i$ batches is at least
\[
    \E_{\Theta \sim {\rm Unif}(\cA)}[R_{iB}] \ge \bP(\Theta \notin S_i) \cdot (iB) \ge \frac{iB}{2}.
\]
We compute the total expected regret $\E_{\Theta}[R_T]$ by considering two cases for the total number of batches $M = \lceil T/B \rceil$:
\begin{itemize}[leftmargin=4ex]
    \item \textbf{Case 1 ($M \le M_0$):} The algorithm halts before reaching the threshold $M_0$. By evaluating the failure probability at the final batch $M$, the algorithm fails to find the optimal arm over all $T$ pulls with probability at least $1 - M/d \ge 1 - M_0/d \ge 1/2$. Thus, the expected regret over the entire horizon is at least $\frac{1}{2} T = \Omega(T)$.
    
    \item \textbf{Case 2 ($M > M_0$):} The algorithm continues past the threshold. Since regret is strictly non-decreasing, we can lower bound the total regret by the expected regret accumulated during just the first $M_0$ batches. This yields at least $\frac{1}{2} M_0 B = \frac{1}{2} \lfloor d/2 \rfloor B = \Omega(dB)$.
\end{itemize}
Combining both cases, the expected regret is $\E_{\Theta \sim {\rm Unif}(\cA)}[R_T] = \Omega(\min\{T, dB\})$. By Yao's minimax principle, the worst-case expected regret is lower bounded by this expected regret: 
\[
    \sup_{\theta \in \cA} \E[R_T] \ge \E_{\Theta \sim {\rm Unif}(\cA)}[R_T] =  \Omega(\min\{T, dB\}).
\]
This completes the proof.
\end{proof}

\section{Proof of Theorem~\ref{thm:exp_A_upper} (\texorpdfstring{$\lvert\cA\rvert$}{|A|}-Independent Upper Bound)}
\label{app:proof_exp_upper}

In this section, we formally establish Theorem~\ref{thm:exp_A_upper} by providing the detailed pseudocode and regret analysis of our first algorithm with an $\lvert \cA \rvert$-independent regret bound. As discussed in Section~\ref{subsec:prelims}, this approach exploits the fact that it suffices to estimate only the $\lvert\supp(\rho_h)\rvert = \widetilde{O}(d)$ scalar quantities $\{\langle a,\theta\rangle\}_{a \in \supp(\rho_h)}$ to successfully recover $\hat{\theta}_h$. Thus, we can apply median-of-means directly to these individual scalar estimates, limiting the required number of blocks to $K = \widetilde{O}(\log(d/\delta))$. Furthermore, by distributing the refinement queries uniformly across the core set, each scalar block estimate $\hat Y_{a,h}^{(j)}$ has variance at most~$\eps_h^2/16$ conditional on a fixed valid localization history. Median-of-means and Corollary~\ref{cor:sum_rho_times_b_G_inv_a} then give the confidence width~$W_h=\sqrt{2d}\,\eps_h$.

The detailed pseudocode is presented in Algorithm~\ref{alg:exp_action_set}. We first proceed with the analysis by first bounding the sample complexity of the localization and refinement stages (Appendix~\ref{app:exp_epoch_cost}). We then establish a high-probability good event for the 1-bit estimators (Appendix~\ref{app:exp_good_event}), verify the confidence width and retention of the optimal arm (Appendix~\ref{app:exp_conf_width}), and finally bound the cumulative regret (Appendix~\ref{app:exp_regret}).

\begin{algorithm}
\caption{1-Bit Phased Elimination with Uniform Arm Allocations}
\label{alg:exp_action_set}
\begin{algorithmic}[1]
\State \textbf{Input:} Action set $\cA \subset \bR^d$, batch size $B$,
  confidence $\delta \in (0,1)$, variance $\sigma^2$, time horizon $T$.
\State Initialize epoch $h = 1$, active set $\cA_1 = \cA$.
\State Set $S_{\max} = 4d(\log\log d + 11)$ and $\widetilde{H} = \lceil \log_2 T \rceil$.
\State Set the number of Median-of-Means (MoM) blocks:
\label{algline:K_exp}
  \begin{equation*}
    K = \left\lceil 8 \log \left(\frac{2 S_{\max} \widetilde{H}}{\delta}\right) \right\rceil.
  \end{equation*}

\For{epoch $h = 1, 2, \dots$}

  \State Find design $\rho_h \colon \cA_h \to [0,1]$ satisfying
    \eqref{eq:near_optimal_design}.
  \State Set precision and confidence width:
  \label{algline:eps_h_exp}
    \begin{equation*}
      \eps_h = \frac{\sigma}{2^h \sqrt{B}}, \qquad W_h = \sqrt{2d} \cdot \eps_h.
    \end{equation*}

  \For{each arm $a \in \supp(\rho_h)$}
    \State $[L_a, U_a] \leftarrow \textsc{1BitLocalize}\left(a,\;
      \delta_{\mathrm{loc}} = \dfrac{\delta}{2 S_{\max} \widetilde{H}}, \; \sigma' = \dfrac{\sigma}{\sqrt{B}}\right)$ (from Appendix~\ref{subsec:1Bitlocalize}).
  \EndFor

  \State Set the uniform arm budget $u_h$ for all $a \in \supp(\rho_h)$:
  \label{algline:uh_exp}
    \begin{equation*}
     u_h = B \left\lceil   16 C_V \cdot 4^h \cdot \log_2 T \right\rceil,
    \end{equation*}
    \Statex \hspace{\algorithmicindent}
    where $C_V \ge 48$ is the absolute constant in the variance from
    Corollary~\ref{cor:refinement_fixed_n}.

  \For{each arm $a \in \supp(\rho_h)$}
    \For{block $j = 1, \dots, K$}
      \State $\hat{Y}_{a,h}^{(j)} \leftarrow
        \textsc{1BitRefine}\left(a, [L_a, U_a],\; n_{\rm q} = \dfrac{u_h}{B} \right)$ (from Appendix~\ref{subsec:1bitRefine}).
    \EndFor
    \State Compute robust scalar estimate:
      $\hat{Y}_{a,h} = \mathrm{median}\big\{
        \hat{Y}_{a,h}^{(1)}, \dots, \hat{Y}_{a,h}^{(K)}\big\}$.
  \EndFor

  \State Compute the design-weighted global parameter estimate~$\hat{\theta}_h \in \bR^d$ according to~\eqref{eq:theta_hat_exp} if
    \Statex \hspace{\algorithmicindent} ${\rm span}(\cA_h) = \bR^d$, or similarly with $\hat{\theta}_h \in {\rm span}(\cA_h)$ otherwise (see Remark \ref{rem:not_span_Rd}).
    
  
  \For{each arm $b \in \cA_h$}
    \State Compute $\hat{\mu}_{b,h} = b^\top \hat{\theta}_h $.
  \EndFor

  \State Update active set:\label{algline:exp_alg_elimination_rule}
    $\cA_{h+1} = \left\{b \in \cA_h :
      \max_{x \in \cA_h} \hat{\mu}_{x,h} - \hat{\mu}_{b,h} \le 2W_h\right\}$.

    \State Update sample counter; \textbf{break} when budget $T$ is exhausted.

\EndFor
\end{algorithmic}
\end{algorithm}

\begin{remark}[Finite-horizon execution]
\label{rem:finite_horizon}
The algorithm runs with the prescribed allocations until $T$ pulls have been made, meaning any subroutine may terminate early after the $T$-th pull.  Only pulls that are actually made contribute to the regret.  
However, when upper bounding the regret of a stage, we may assume that every allocated pull is performed (even if the stage is actually interrupted), since additional pulls can only increase regret.   For the probability analysis, ``good event'' conditions are imposed only on completed stages, and if a specified conditioning point is not reached, we define the corresponding history to be the history at termination.  
\end{remark}

Throughout this section, let $\sigma' = \sigma/\sqrt{B}$ and
$\eps_h = 2^{-h}\sigma'$ as in Line~\ref{algline:eps_h_exp} of Algorithm~\ref{alg:exp_action_set}.
We invoke Corollary~\ref{cor:refinement_fixed_n} with the full-batch
effective standard deviation $\sigma' = \sigma/\sqrt{B}$; each query
corresponds to $B$ raw arm pulls. Conditional on a fixed valid localization history, the guarantees from~\eqref{eq:1bit_fixed_budget} are:
\begin{equation}
\label{eq:subroutine_constants}
  \Var(\hat{Y}) \le \frac{C_V \sigma'^2\log_2 T}{n_{\mathrm{q}}},
  \qquad
  \lvert\mathrm{Bias}(\hat{Y})\rvert \le \frac{3\sigma'}{T^2},
\end{equation}
where $n_{\mathrm{q}}$ denotes the number of \emph{queries} (i.e., the
number of raw samples divided by $B$) and $C_V \ge 48$ is the variance constant from
Corollary~\ref{cor:refinement_fixed_n}.  Crucially, the bias bound is independent of both $n_{\mathrm{q}}$ and the epoch index $h$.

\subsection{Sample Complexity per Epoch}
\label{app:exp_epoch_cost}
To facilitate the regret analysis, we separate the sample complexity of the localization and refinement stages in each epoch.

\paragraph{Localization cost.}
By Proposition~\ref{prop:localize}, each call to \textsc{1BitLocalize} with effective standard
deviation $\sigma' = \sigma/\sqrt{B}$ and failure probability
$\delta_{\mathrm{loc}} = \delta/(2S_{\max}\widetilde{H})$ uses at most
\[
    n_{\mathrm{loc}}
    = C_{\mathrm{loc}}\left(\log\Big(\frac{1}{\sigma'}\Big) + \log\Big(\frac{1}{\delta_{\mathrm{loc}}}\Big)\right)
    = C_{\mathrm{loc}}\left(\frac{1}{2}\log\Big(\frac{B}{\sigma^2}\Big)
        + \log\left(\frac{2S_{\max}\widetilde{H}}{\delta}\right)\right)
\]
queries, where $C_{\mathrm{loc}} > 0$ is the absolute constant from Proposition~\ref{prop:localize}. Since each query costs $B$ raw arm pulls, and since $S_{\max} = O(d \log \log d)$, and $\widetilde{H} = \lceil \log_2 T \rceil$, the total localization cost over all support arms is bounded by
\begin{equation}
    \label{eq:mloc_h_exp}
    m_{\mathrm{loc},h} \le S_{\max} \cdot B \cdot n_{\mathrm{loc}}
    = O\left(dB \log\log d \cdot \log\Big(\frac{dB \log T}{\delta}\Big)\right),
\end{equation}
where we have absorbed $\sigma = \Theta(1)$ into the $O(\cdot)$ notation.

\paragraph{Refinement cost.}
Because the uniform allocation $u_h$ (see Line~\ref{algline:uh_exp}) is identical for every arm
in the support, the total refinement cost $m_{\mathrm{ref},h}$ is at most
$K \cdot |\supp(\rho_h)| \cdot u_h$ raw arm pulls. Using
$\lvert\supp(\rho_h)\rvert \le S_{\max} = O(d \log\log d)$ and $u_h \le B(16C_V \cdot 4^h \log_2 T + 1)$, we have
\begin{equation}
    \label{eq:mref_h_exp}
    m_{\mathrm{ref},h}
    \le K \cdot S_{\max} \cdot B\!\left(16C_V\cdot 4^h\log_2 T + 1\right)
    = O\!\left(K d B \log T \log\log d \cdot 4^h\right).
\end{equation}

\subsection{High-Probability Good Event}
\label{app:exp_good_event}

Let $H$ be the last epoch of Algorithm~\ref{alg:exp_action_set} reached by time~$T$, with $H=0$ if none.
Define the good event
$\cE = \cE_{\mathrm{loc}} \cap \cE_{\mathrm{ref}}$, where
\begin{align*}
  \cE_{\mathrm{loc}} &\coloneqq
      \bigcap_{h=1}^{H} \bigcap_{a \in \supp(\rho_h)}
      \left\{\mu_a \in [L_a, U_a]
        \;\text{ and }\;
        U_a - L_a \le \tfrac{8\sigma}{\sqrt{B}}\right\},\\[4pt]
  \cE_{\mathrm{ref}} &\coloneqq
      \bigcap_{h=1}^{H} \bigcap_{a \in \supp(\rho_h)}
      \left\{\lvert \hat{Y}_{a,h} - \mu_a \rvert \le \eps_h\right\}.
\end{align*}

\begin{lemma}[Good event probability]
  \label{lem:exp_good_event}
  $\bP(\cE) \ge 1 - \delta$.
\end{lemma}

\begin{proof}
We decompose the proof into multiple parts as follows.

\textbf{Part 1: Deterministic bounds.}
We first state two deterministic bounds used throughout the proof.

\textit{1.1 Epoch count $H \le \widetilde{H}$.} By Line~\ref{algline:uh_exp} and
$4^{\lceil\log_2 T\rceil} \ge T^2$, we have
\[
  u_{\widetilde{H}}
  = B\!\left\lceil 16C_V\cdot 4^{\widetilde{H}}\log_2 T\right\rceil
  \ge 16C_V B\cdot T^2\log_2 T
  > T,
\]
so epoch~$\widetilde H$ cannot be completed within~$T$ pulls.
The algorithm therefore cannot reach epoch~$\widetilde H+1$,
giving $H\le\widetilde H=\lceil\log_2T\rceil$.

\textit{1.2 Core set size.} By the near $G$-optimal design
property~\eqref{eq:near_optimal_design}, the core set satisfies
$|\supp(\rho_h)| \le S_{\max}$ for every epoch $h \ge 1$.

\textbf{Part 2: Localization event $\cE_{\mathrm{loc}}$.}
The two bounds above imply that Algorithm~\ref{alg:exp_action_set} makes at most $S_{\max} \cdot \widetilde{H}$ total calls to \textsc{1BitLocalize}. 
By Proposition~\ref{prop:localize}, each call uses fresh arm pulls and has failure probability at most~$\delta_{\mathrm{loc}}=\delta/(2S_{\max}\widetilde H)$ conditional on the history before that call.
A union bound over all calls gives 
\[
  \P(\cE_{\mathrm{loc}}^c)
  \le \widetilde{H} \cdot S_{\max} \cdot \delta_{\mathrm{loc}}
  = \frac{\delta}{2}.
\]

\textbf{Part 3: Refinement event $\cE_{\mathrm{ref}}$.}
We show $\P(\cE_{\mathrm{ref}}^c\cap\cE_{\mathrm{loc}})\le\delta/2$.

Fix an epoch index~$h\in\{1,\ldots,\widetilde H\}$. 
Let~$\mathcal G_h$ be the $\sigma$-algebra generated by the history
immediately after localization in epoch~$h$ and before its refinement
observations are drawn.
Condition on a fixed realization of~$\mathcal G_h$ in which localization is completed, all its intervals are valid, and enough pulls remain to complete the allocated refinement stage.  This fixes~$\cA_h$, $\rho_h$, $u_h$, and the intervals~$\{[L_a,U_a]\}_{a\in\supp(\rho_h)}$.  The refinement observations and random thresholds are fresh.
All expectations, biases, and variances below refer to this conditional law.

For any fixed
$a \in \supp(\rho_h)$, since each block $j \in \{1,\ldots,K\}$ uses fresh, independent arm pulls with the same fixed parameters, the $K$ estimates
$\hat{Y}_{a,h}^{(1)},\ldots,\hat{Y}_{a,h}^{(K)}$ are i.i.d., and so
Lemma~\ref{lem:mom_statement} applies to their median.
We bound the bias and variance of a single block estimate.
\begin{itemize}[leftmargin=4ex]
  \item \textbf{Bias.}
  Since $n_{\rm q} = u_h/B \ge 16C_V \cdot 4^h \log_2 T \ge 48\log_2 T$,
  Corollary~\ref{cor:refinement_fixed_n} gives
  $|\mathrm{Bias}(\hat{Y}_{a,h}^{(1)})| \le 3\sigma'/T^2$.
  Since $h \le H \le \widetilde{H}$ and $2^{\widetilde{H}} =
  2^{\lceil \log_2 T\rceil} \le 2T$, we have $2^h \le 2T$, so
  \[
    \frac{3\sigma'}{T^2}
    \le \frac{6\sigma'}{T \cdot 2^h}
    = \frac{6\varepsilon_h}{T}
    \le \frac{\varepsilon_h}{2},
  \]
  where the last step holds for $T \ge 12$.

  \item \textbf{Variance.}
  Applying Corollary~\ref{cor:refinement_fixed_n} with~$n_{\rm q}=u_h/B$ and using the budget bound~$u_h/B\ge16C_V4^h\log_2T$ gives
  \[
    \mathrm{Var}(\hat{Y}_{a,h}^{(1)})
    \le \frac{C_V \sigma'^2 \log_2 T}{u_h/B}
    \le \frac{C_V \sigma'^2 \log_2 T}{16C_V \cdot 4^h \log_2 T}
    = \frac{\varepsilon_h^2}{16}.
  \]
  where the last equality uses~$\varepsilon_h=\sigma'/2^h$.
  
\end{itemize}
Applying Lemma~\ref{lem:mom_statement} to the $K$ i.i.d.\ block estimates (i.e., $k=K$ therein) 
with variance at most $\varepsilon_h^2/16$ and a failure
probability of $\delta/(2S_{\max}\widetilde{H})$, we obtain
\[
  \left|\hat{Y}_{a,h} - \E[\hat{Y}_{a,h}^{(1)}]\right|
  \le \sqrt{\frac{32 \cdot (\varepsilon_h^2/16)}{K}
      \log\!\left(\frac{2S_{\max}\widetilde{H}}{\delta}\right)}
  = \varepsilon_h \sqrt{\frac{2\log(2S_{\max}\widetilde{H}/\delta)}{K}}
  \le \frac{\varepsilon_h}{2},
\]
where the last step uses $K \ge 8\log(2S_{\max}\widetilde{H}/\delta)$.
Combined with the bias bound, the triangle inequality gives
\[
  |\hat{Y}_{a,h} - \mu_a|
  \le \underbrace{|\hat{Y}_{a,h} - \E[\hat{Y}_{a,h}^{(1)}]|}
        _{\le\,\varepsilon_h/2}
   + \underbrace{|\mathrm{Bias}(\hat{Y}_{a,h}^{(1)})|}_{\le\,\varepsilon_h/2}
  \le \varepsilon_h,
\]
with probability at least $1 - \delta/(2S_{\max}\widetilde{H})$
under our fixed conditioning.
Let~$\mathcal V_h\in\mathcal G_h$ be the event that localization
in epoch~$h$ is completed by time~$T$ and all its intervals are valid.
Let~$F_h$ be the event that refinement in epoch~$h$ is completed
by time~$T$ and $|\hat Y_{a,h}-\mu_a|>\varepsilon_h$
for some~$a\in\supp(\rho_h)$.
A union bound over~$\supp(\rho_h)$, whose size is at most~$S_{\max}$, gives $\P(F_h\mid\mathcal G_h)\le\delta/(2\widetilde H)$ on~$\mathcal V_h$.
Since~$F_h\cap\mathcal E_{\mathrm{loc}}\subseteq F_h\cap\mathcal V_h$ and~$\mathcal V_h\in\mathcal G_h$, the tower property gives
\begin{equation}
\label{eq:tower_property}
\begin{aligned}
  \P(F_h\cap\mathcal E_{\mathrm{loc}})
  &\le\P(F_h\cap\mathcal V_h)\\
  &=\E\left[
      \E\left[
        \mathds{1}\{\mathcal V_h\} \cdot \mathds{1}\{F_h\}
        \;\middle|\;\mathcal G_h
      \right]
    \right]\\
  &=\E\left[
      \mathds{1}\{\mathcal V_h\} \cdot \P(F_h\mid\mathcal G_h)
    \right]\\
  &\le\frac{\delta}{2\widetilde H}.
\end{aligned}
\end{equation}
Since~$H\le\widetilde H$, a union bound over the fixed indices~$h=1,\ldots,\widetilde H$ gives
\begin{equation*}
  \P(\mathcal E_{\mathrm{ref}}^c\cap\mathcal E_{\mathrm{loc}})
  \le\sum_{h=1}^{\widetilde H}\frac{\delta}{2\widetilde H}
  =\frac{\delta}{2}\,.
\end{equation*}

\textbf{Part 4: Conclusion.}
Combining this with the localization failure bound, we obtain
\[
  \P(\mathcal E^c)
  \le\P(\mathcal E_{\mathrm{loc}}^c)
      +\P(\mathcal E_{\mathrm{ref}}^c\cap\mathcal E_{\mathrm{loc}})
  \le\frac{\delta}{2}+\frac{\delta}{2}
   =\delta.
\]
Thus $\P(\mathcal E)\ge1-\delta$, as desired.
\end{proof}

\subsection{Confidence Width and Arm Elimination}
\label{app:exp_conf_width}

\begin{lemma}[Confidence width]
  \label{lem:huge_confidence_width}
  Under event $\cE$, for every epoch~$h$ whose refinement stage is completed and every $b\in\cA_h$, we have
  $\lvert \hat\mu_{b,h} - \mu_b \rvert \le W_h = \sqrt{2d}\eps_h$.
\end{lemma}

\begin{proof}
Using~\eqref{eq:theta_hat_exp} and the identity
\[
    \theta = G_h^{-1}\sum_{a \in \supp(\rho_h)} \rho_h(a)a a^\top\theta
    \implies
    \mu_b = b^{\top} \theta = b^\top  G_h^{-1}\sum_{a \in \supp(\rho_h)} \rho_h(a)\cdot a\cdot \mu_a,
\]
we have
\begin{equation*}
    \lvert \hat\mu_{b,h} - \mu_b \rvert
  = \left\lvert b^\top G_h^{-1}
      \sum_{a}
      \rho_h(a)\cdot a\cdot(\hat{Y}_{a,h} - \mu_a)
    \right\rvert 
  \le \sum_{a}  \big\lvert \hat{Y}_{a,h} - \mu_a \big\rvert \cdot \rho_h(a)\cdot
      \left\lvert b^\top G_h^{-1} a \right\rvert.
\end{equation*}
Under $\cE_{\mathrm{ref}}$, each $\lvert\hat{Y}_{a,h}-\mu_a\rvert\le\eps_h$.
Combining this with Corollary~\ref{cor:sum_rho_times_b_G_inv_a}, we obtain $\lvert \hat\mu_{b,h} - \mu_b \rvert \le \eps_h\sqrt{2d} = W_h$.
\end{proof}

\begin{lemma}[Optimal arm never eliminated]\label{lem:exp_opt_arm}
  Under $\cE$, the optimal arm $a^* = \arg\max_{a \in \cA} \mu_a$ satisfies
  $a^* \in \cA_h$ for all epochs $h \ge 1$.
\end{lemma}

\begin{proof}
    We prove this by induction. The base case $\cA_1 = \cA$ is immediate.  For the inductive step, suppose that $a^* \in \cA_h$, and assume for contradiction that $a^*$ is eliminated at epoch $h$. Then $\hat\mu_{b,h} - \hat\mu_{a^*,h} > 2W_h$ for some $b \in \cA_h$. Lemma~\ref{lem:huge_confidence_width} gives $\mu_b - \mu_{a^*} \ge (\hat\mu_{b,h} - W_h) - (\hat\mu_{a^*,h} + W_h) > 0$, contradicting the optimality of $a^*$.
\end{proof}

\begin{lemma}[Suboptimality gap for active arms]
\label{lem:exp_active_gap}
Under $\cE$, for all epochs $h \ge 2$ and all active arms $b \in \cA_h$, we have
\[
  \Delta_b \coloneqq \mu_{a^*} - \mu_b \le 4W_{h-1} = 8W_h.
\]
\end{lemma}
\begin{proof}
Fix $h \ge 2$ and $b \in \cA_h$. Since $b$ survived the elimination step
at the end of epoch $h-1$, and $a^* \in \cA_{h-1}$ by
Lemma~\ref{lem:exp_opt_arm}, the elimination rule (see Line~\ref{algline:exp_alg_elimination_rule}) gives
\[
  \hat\mu_{a^*,h-1} - \hat\mu_{b,h-1}
  \le \max_{x \in \cA_{h-1}}\hat\mu_{x,h-1} - \hat\mu_{b,h-1}
  \le 2W_{h-1}.
\]
Applying Lemma~\ref{lem:huge_confidence_width} to bound
$\mu_{a^*} \le \hat\mu_{a^*,h-1} + W_{h-1}$ and
$\mu_b \ge \hat\mu_{b,h-1} - W_{h-1}$, we obtain
\[
  \Delta_b
  = \mu_{a^*} - \mu_b
  \le (\hat\mu_{a^*,h-1} - \hat\mu_{b,h-1}) + 2W_{h-1}
  \le 4W_{h-1}
  = 8W_h,
\]
where the last equality uses $W_{h-1} = 2W_h$ (see Line~\ref{algline:eps_h_exp}).
\end{proof}

\subsection{Regret Bound}
\label{app:exp_regret}

\begin{proof}[Proof of Theorem~\ref{thm:exp_A_upper}]
We condition on the good event $\cE$, which by Lemma~\ref{lem:exp_good_event} occurs with probability at least $1-\delta$. Throughout the remainder of this proof, all bounds are deterministic consequences of this event.

Let $\Delta_b = \mu_{a^*} - \mu_b$ denote the suboptimality gap of any active arm $b \in \cA_h$. Because all mean rewards lie in $[-1, 1]$, the gap is trivially bounded by $\Delta_b \le 2$. For epochs $h \ge 2$, Lemma~\ref{lem:exp_active_gap} gives $\Delta_b \le \min\left\{ 8W_h, 2\right\}$.

We upper bound the total regret by splitting it into the regret incurred during the localization stages and the regret incurred during the refinement stages across all epochs. Let $H \le \lceil \log_2 T \rceil$ be the epoch at which the algorithm halts.

\textbf{Localization Regret.}
Because the instantaneous regret is always at most $2$, the total regret incurred during the localization phases across all $H$ epochs is bounded via~\eqref{eq:mloc_h_exp}:
\begin{equation}
  \label{eq:R_loc_exp}
  R_{\mathrm{loc}} \le \sum_{h=1}^{H} 2 \cdot m_{\mathrm{loc},h} 
  = O\left(H \cdot dB \log\log d \log\Big(\frac{dB \log T}{\delta}\Big)\right) 
  = \widetilde{O}\left(dB \log\Big(\frac{1}{\delta}\Big)\right),
\end{equation}
where the $\widetilde{O}(\cdot)$ notation absorbs the $\text{polylog}(d, T)$ factors.

\textbf{Refinement Regret.}
For the first epoch ($h=1$), we bound the refinement gap trivially by $2$:
\begin{equation}
  \label{eq:R1_ref_exp}
  R_{\mathrm{ref}, 1} \le 2 \cdot m_{\mathrm{ref}, 1} = O(K d B \log T \log\log d) = \widetilde{O}(K \cdot dB).
\end{equation}
For any subsequent epoch $h \ge 2$, Lemma~\ref{lem:exp_active_gap} gives
$\Delta_b \le 8W_h$. It follows that the refinement regret for epoch
$h \ge 2$ is bounded by
\begin{equation}
  \label{eq:Rh_ref_exp}
  R_{\mathrm{ref}, h} \le m_{\mathrm{ref}, h} \cdot 8W_h
  = m_{\mathrm{ref}, h} \cdot \frac{8\sqrt{2d}\sigma}{2^h\sqrt{B}}
  = O\left(2^{-h} \cdot m_{\mathrm{ref}, h} \cdot \sqrt{\frac{d}{B}} \right).
\end{equation}
Epochs with $m_{\mathrm{ref},h}=0$ contribute no refinement regret.
For the remaining epochs, taking the bound
$m_{\mathrm{ref}, h} \le
O(K d B \log T \log\log d \cdot 4^h)$ established in~\eqref{eq:mref_h_exp} and rearranging yields
\[
    4^{-h} \le O\left( \frac{K d B \log T \log\log d}{m_{\mathrm{ref}, h}} \right) \implies 2^{-h} \le O\left( \sqrt{\frac{K d B \log T \log\log d}{m_{\mathrm{ref}, h}}} \right).
\]
Substituting this bound into~\eqref{eq:Rh_ref_exp} allows the batch size $B$ to cancel out, and we obtain
\[
    R_{\mathrm{ref}, h} 
    \le O\left( \sqrt{\frac{K d B \log T \log\log d}{m_{\mathrm{ref}, h}}} \cdot m_{\mathrm{ref}, h} \cdot \sqrt{\frac{d}{B}} \right)
    = O\left( d \sqrt{K \log T \log\log d} \cdot \sqrt{m_{\mathrm{ref}, h}} \right).
\]
Summing this over $h = 2, \dots, H$ and applying the Cauchy-Schwarz inequality and the trivial bound $\sum_{h=2}^H m_{\mathrm{ref}, h} \le T$, we obtain
\begin{equation}
\label{eq:Rh_ref_exp_CS_trick}
    \sum_{h=2}^H \sqrt{m_{\mathrm{ref}, h}} 
= \sum_{h=2}^H \left( \sqrt{m_{\mathrm{ref}, h}}  \cdot 1 \right)
\le 
\sqrt{\sum_{h=2}^H m_{\mathrm{ref}, h}} \cdot \sqrt{\sum_{h=2}^H 1}
\le
\sqrt{T \cdot H},
\end{equation}
Since $H \le \lceil \log_2 T \rceil$, the total refinement regret for $h \ge 2$ is bounded by
\begin{equation}
\label{eq:R_ref_sum_exp}
    \sum_{h=2}^{H} R_{\mathrm{ref}, h}
  \le O\left(d \sqrt{K \log T \log\log d} \cdot \sqrt{T \log T} \right) 
  = \widetilde{O}\left(d \sqrt{KT}\right). 
\end{equation}

\textbf{Total regret.}
Combining the bounds from~\eqref{eq:R_loc_exp},~\eqref{eq:R1_ref_exp},
and~\eqref{eq:R_ref_sum_exp}, the cumulative regret satisfies
\begin{equation}
\label{eq:RT_exp_presubstitution}
    R_T \le R_{\mathrm{loc}} + R_{\mathrm{ref},1} + \sum_{h=2}^{H} R_{\mathrm{ref},h}
    = \widetilde{O}\left(dB \log\left(\frac{1}{\delta}\right) + K \cdot dB + d \sqrt{KT} \right).
\end{equation}
Using $S_{\max} = O(d\log\log d)$ and $\widetilde{H} = \lceil\log_2 T\rceil$, we have
\[
    K = \left\lceil 8\log\Big(\frac{2S_{\max}\widetilde{H}}{\delta}\Big)\right\rceil =  O\left(\log\left(\frac{d \log T}{\delta}\right)\right),
\]
and so the middle term in~\eqref{eq:RT_exp_presubstitution} is absorbed by the first, giving the final expression
\begin{equation}
  \label{eq:RT_exp_final}
  R_T = \widetilde{O}\!\left(
    dB\log\Big(\frac{1}{\delta}\Big)
    +  d\sqrt{T\log\Big(\frac{1}{\delta}\Big)}
  \right).
\end{equation}
This matches the bound stated in Theorem~\ref{thm:exp_A_upper}.
\end{proof}

\section{Proof of Theorem~\ref{thm:poly_A_upper} (\texorpdfstring{$\lvert\cA\rvert$}{|A|}-Dependent Upper Bound)}
\label{app:proof_poly_upper}

In this section, we provide the pseudocode and detailed regret analysis for our second phased elimination algorithm, formally establishing the bounds stated in Theorem~\ref{thm:poly_A_upper}. The detailed pseudocode is presented in Algorithm~\ref{alg:poly_action_set}.
We again use the finite-horizon conventions in
Remark~\ref{rem:finite_horizon}, applying the completed-stage convention also to individual rounds of \textsc{SafeArm}.

We split our analysis into two regimes: the small-$B$ regime, where $B < 8K_\star d(\log \log d + 11) = \widetilde{O}(d\log (\lvert \cA \rvert /\delta))$, and the large-$B$ regime, where $B$ exceeds this threshold (see Line \ref{line:threshold} of Algorithm \ref{alg:poly_action_set} for the definition of $K_\star$). In the small-$B$ regime, the pre-processing is safely bypassed and only Step~3 (phased elimination) is executed with the initial active set $\cA_1 = \cA$ (analyzed in Appendix~\ref{app:step3_small_B}). In the large-$B$ regime, Steps~1 and~2 (safe arm identification and a warm start epoch) are run first to produce a ``warm-started'' active set $\cA_1 \subseteq \cA$ before proceeding to Step~3 (analyzed in Appendix~\ref{app:step3_large_B}).
The main difference between the two regimes is the worst-case suboptimality gap available at the \emph{first epoch} of Step~3: a trivial $O(1)$ bound in the small-$B$ case versus the refined $O(\sigma\sqrt{d/B})$ guarantee from the warm start in the large-$B$ case.

\begin{remark}[Comparison with Algorithm~\ref{alg:exp_action_set}] \label{rem:comparison}
Algorithm~\ref{alg:poly_action_set} differs from Algorithm~\ref{alg:exp_action_set} in three interconnected ways, each
stemming from the goal of capturing the tight $\lvert\cA\rvert$-dependence
required for smaller action sets (e.g., $\lvert\cA\rvert = \mathrm{poly}(d)$).
\begin{itemize}[leftmargin=4ex]
    \item \textbf{Proportional vs.\ uniform allocation.}
    Algorithm~\ref{alg:exp_action_set} allocates a uniform query budget $u_h$ to every core-set arm, so each scalar block estimate $\hat{Y}_{a,h}^{(j)}$ has variance $O(\varepsilon_h^2)$ conditional on a fixed valid localization history.
    Consequently, propagating these scalar error bounds through $\hat\theta_h$ via Corollary~\ref{cor:sum_rho_times_b_G_inv_a} gives the confidence width $W_h=\sqrt{2d}\,\varepsilon_h$.
    Algorithm~\ref{alg:poly_action_set} instead allocates $u_h(a) \propto \rho_h(a)$, meaning $\mathrm{Var}(\hat{Y}_{a,h}^{(j)}) \propto 1/\rho_h(a)$. These $\rho_h(a)$ factors perfectly cancel when computing $\mathrm{Var}(b^\top\hat\theta_h^{(j)})$ via Lemma~\ref{lem:sum_rho_times_squared_b_G_inv_a}, yielding the tighter confidence width $W_h = \varepsilon_h$. This variance cancellation is precisely what enables Algorithm~\ref{alg:poly_action_set} to match the optimal $\sigma\sqrt{dT\log\lvert\cA\rvert}$ ``standard linear bandit term'' in the minimax lower bound.
    
    \item \textbf{Order of aggregation and larger $K$.}
    As noted in Section~\ref{sec:algs}, the two algorithms differ in when the median-of-means aggregation is applied. Algorithm~\ref{alg:exp_action_set} applies it directly to the $\lvert\supp(\rho_h)\rvert = \widetilde{O}(d)$ scalar estimates $\hat{Y}_{a,h}^{(j)}$, meaning simultaneous concentration over the support set requires only $K = \widetilde{O}(\log(d/\delta))$ blocks. Algorithm~\ref{alg:poly_action_set} instead forms block-wise parameter estimates $\hat{\theta}_h^{(j)}$ first, and then evaluates the inner products $b^\top\hat\theta_h^{(j)}$ for all active arms $b \in \cA_h$. We ensure simultaneous concentration over all active arms using a union bound over at most $\lvert\cA\rvert$ arms. Accordingly, we use $K=\widetilde{O}(\log(\lvert\cA\rvert/\delta))$ blocks, yielding the $\sqrt{\log(\lvert\cA\rvert/\delta)}$ factor in the statistical term of Theorem~\ref{thm:poly_A_upper}.

    \item \textbf{Safe arm and warm start for large $B$.}
    In Algorithm~\ref{alg:exp_action_set}, because $K = \widetilde{O}(\log(d/\delta))$,
    the first-epoch regret $\widetilde{O}(dBK) = \widetilde{O}(dB)$ matches
    the communication-based lower bound $\Omega(dB)$ (see Theorem \ref{thm:minimax_lower}) up to logarithmic factors,
    where we use $\Omega(B\min\{d,\log\lvert\cA\rvert\}) = \Omega(dB)$ for
    the exponentially large action sets targeted by the algorithm.
    In Algorithm~\ref{alg:poly_action_set}, however, the larger block count
    $K = \widetilde{O}(\log(\lvert\cA\rvert/\delta))$ inflates the first-epoch cost to $\widetilde{O}(dB\log\lvert\cA\rvert)$, which exceeds the
    communication-based lower bound $\Omega(B\log\lvert\cA\rvert)$ by a factor of
    $d$ and dominates the regret when $B$ is large.
    The safe arm identification and warm-start steps
    (Algorithms~\ref{alg:safe_arm} and~\ref{alg:warm_start}) address this
    by  reducing the first-epoch suboptimality gap from $O(1)$ down to
    $O(\sqrt{d/B})$ before the main loop begins, shrinking the first-epoch regret from $\widetilde{O}(dB\log\lvert\cA\rvert)$ to $\widetilde{O}(d^{3/2}\sqrt{B}\log\lvert\cA\rvert)$.
\end{itemize}
\end{remark}

\subsection{Small-\texorpdfstring{$B$}{B} Regime: Step~3 as a Standalone Algorithm}
\label{app:step3_small_B}

In the small $B$ regime where $B < 8K_\star d(\log \log d + 11) = \widetilde{O}(d \log (\lvert \cA \rvert /\delta))$, Algorithm \ref{alg:poly_action_set} skips Steps 1 and 2 and directly runs the phased elimination algorithm (Step 3) starting with $\cA_1 = \cA$. We first establish the shared technical machinery for the phased elimination (sample complexity, good event, confidence width, and arm elimination), which applies to both regimes. We then derive the small-$B$ regret bound. The large-$B$ analysis in Appendix~\ref{app:step3_large_B} reuses most of these lemmas unchanged, differing only in the first-epoch suboptimality gap.

Throughout this section, let $\sigma' = \frac{\sigma}{\sqrt{B}}$ and $\eps_h = 2^{-h}\sigma'$ as defined in Line~\ref{algline:poly_eps_h} of Algorithm~\ref{alg:poly_action_set}. We invoke Corollary~\ref{cor:refinement_fixed_n} with the full-batch effective standard deviation $\sigma' = \frac{\sigma}{\sqrt{B}}$; each query corresponds to $B$ raw arm pulls. Conditional on a fixed valid localization history, the guarantees from~\eqref{eq:1bit_fixed_budget} are:
\begin{equation}
\label{eq:subroutine_constants_poly}
  \Var(\hat{Y}) \le \frac{C_V \sigma'^2\log_2 T}{n_{\mathrm{q}}},
  \qquad
  \lvert\mathrm{Bias}(\hat{Y})\rvert \le \frac{3\sigma'}{T^2},
\end{equation}
where $n_{\mathrm{q}}$ denotes the number of \emph{queries} (i.e., the number of raw samples divided by $B$) and $C_V \ge 48$ is the variance constant from Corollary~\ref{cor:refinement_fixed_n}. The bias bound is independent of $n_{\mathrm{q}}$ and the epoch index~$h$.

\subsubsection{Sample Complexity per Epoch}
\label{app:step3_cost}

\paragraph{Localization cost.} By the same argument as in Appendix~\ref{app:exp_epoch_cost}, with failure probability
$\delta_{\mathrm{loc}} = \delta/\big(6 S_{\max} \widetilde H \big)$ in place of
$\delta/\big(2 S_{\max} \widetilde H \big)$, the total localization cost over all support
arms satisfies
\begin{equation}
\label{eq:poly_mloc}
  m_{{\rm loc},h}
  \;\le\;
  C_{\mathrm{loc}} \cdot S_{\max} \cdot B\!\left(
      \frac{1}{2}\log\frac{B}{\sigma^2}
      + \log\frac{6 S_{\max} \widetilde H}{\delta}
  \right)
  \;=\; \widetilde{O}\!\left(dB\log\Big(\frac{1}{\delta}\Big)\right).
\end{equation}
As in Appendix~\ref{app:proof_exp_upper}, the $\log(1/\sigma)$ term is absorbed into the $\widetilde{O}(\cdot)$ notation since we assume $\sigma = \Theta(1)$.


\paragraph{Refinement cost.}
We bound the ``true'' refinement budget by summing the per-arm allocations $u_h(a)$ from~Line~\ref{algline:poly_uh_a} over the support and multiplying by $K$ blocks. 
Using $\lceil x\rceil \le x+1$ and $\lvert\supp(\rho_h)\rvert = \widetilde{O}(d)$, we bound the number of arm pulls used for refinement in an epoch by
\begin{equation} \label{eq:A-dependent_mref_bound}
\begin{aligned}
    m_{\mathrm{ref},h}
    &\le K\sum_{a\in\supp(\rho_h)} u_h(a) \\
    &\le KB \sum_{a\in\supp(\rho_h)}  \left( 32C_V  \cdot d \cdot  \log_2 T \cdot 4^h \cdot \rho_h(a) + 48 \log_2 T + 1 \right) \\
    &= O(dBK \cdot \log T \cdot \log\log d \cdot 4^h),
\end{aligned}
\end{equation}

The localization floor $\widetilde{O}(dB\log(1/\delta))$ is present at every
epoch, while the refinement term $\widetilde{O}(dBK\cdot 4^h)$ grows
geometrically and thus the largest $h$ value is the dominant one.

\begin{algorithm}[!t]
\caption{1-bit Phased Elimination with Safe Arm and Warm Start}
\label{alg:poly_action_set}
\begin{algorithmic}[1]
\State \textbf{Input:} Action set $\cA \subset \mathbb{R}^d$, batch size $B$, confidence $\delta \in (0,1)$, variance $\sigma^2$, time horizon $T$.
\State Set $K_\star = \left\lceil 8 \log \left(\frac{3 \lvert \cA \rvert \cdot\lceil \log_2 \lvert \cA \rvert\rceil }{\delta}\right) \right \rceil$. \label{line:threshold}
\If{$B \geq 8K_\star d(\log \log d + 11)$} \label{algline:tau}
\State Compute $a_{\rm safe} \leftarrow \textsc{SafeArm}(\cA, B, \delta)$ (Algorithm \ref{alg:safe_arm} in Appendix \ref{subsec:safe_arm_proof}).
\State Set $\cA_1 \leftarrow \textsc{WarmStart}(\cA, a_{\rm safe}, B, \delta)$ (Algorithm \ref{alg:warm_start} in Appendix \ref{subsec:warm_start_proof}).
\Else \State Set $\cA_1 \leftarrow \cA$
\EndIf
\State Set $S_{\max} = 4d(\log\log d + 11)$ and $\widetilde{H} = \lceil \log_2 T \rceil$.
\State Set the number of Median-of-Means (MoM) blocks:
    \begin{equation*}
      K = \left\lceil 8 \log \left(\frac{6 \lvert \cA \rvert \widetilde{H}}{\delta}\right)\right\rceil.
    \end{equation*}

\For{epoch $h = 1, 2, \dots$}

  \State Find design $\rho_h \colon \cA_h \to [0,1]$ satisfying \eqref{eq:near_optimal_design}.
  \State Set precision and confidence width:
  \label{algline:poly_eps_h}
    \begin{equation*}
      \eps_h = \frac{\sigma}{2^h\sqrt{B}},
      \qquad
      W_h = \eps_h,
    \end{equation*}
    \Statex \hspace{\algorithmicindent}
    where $C_V \ge 48$ is the variance constant from Corollary~\ref{cor:refinement_fixed_n}.
   
  \For{each arm $a \in \supp(\rho_h)$}
    \State $[L_a, U_a] \leftarrow \textsc{1BitLocalize}\!\left(a,\;
      \delta_{{\rm loc}} = \dfrac{\delta}{6S_{\max} \widetilde{H}}, \;\sigma' = \dfrac{\sigma}{\sqrt{B}}
    \right)$ (from Appendix~\ref{subsec:1Bitlocalize}).
  \EndFor

  \State Set per-arm sample budget for each block:
  \label{algline:poly_uh_a}
    \begin{equation*}
      u_h(a) = B\max\!\left\{ \left\lceil 48\log_2 T \right\rceil, \;
        \left\lceil 32C_V  \cdot d   \log_2 T \cdot 4^h \cdot \rho_h(a)\right\rceil
      \right\}.
    \end{equation*}

  \For{block $j = 1, \dots, K$}
    \State Compute $\hat{Y}_{a,h}^{(j)} \leftarrow
      \textsc{1BitRefine}\left(a, [L_a, U_a],\; n_{\rm q} = u_h(a)/B\right)$
      (from Appendix~\ref{subsec:1bitRefine}) 
    \Statex  \hspace{\algorithmicindent} \hspace{\algorithmicindent}   for each $a \in \supp(\rho_h)$.
    \State Compute the design-weighted parameter estimate~$\hat{\theta}_h^{(j)} \in \bR^d$ according to~\eqref{eq:theta_hat_poly} if
    \Statex \hspace{\algorithmicindent} \hspace{\algorithmicindent} ${\rm span}(\cA_h) = \bR^d$, or similarly with $\hat{\theta}_h^{(j)} \in {\rm span}(\cA_h)$ otherwise (see Remark \ref{rem:not_span_Rd}).

  \EndFor

  \For{each arm $b \in \cA_h$}
    \State $\hat\mu_{b,h} = \mathrm{median}\!\big\{
      b^\top \hat\theta_h^{(1)}, \dots,
       b^\top \hat\theta_h^{(K)}\big\}$.
  \EndFor

  \State Update active set:
    $\cA_{h+1} = \left\{a \in \cA_h :
      \max_{b\in\cA_h}\hat\mu_{b,h} - \hat\mu_{a,h} \le 2W_h\right\}$.

  \State Update sample counter; \textbf{break} when budget $T$ is exhausted.

\EndFor

\end{algorithmic}
\end{algorithm}

\subsubsection{High-Probability Good Event}
\label{app:step3_good_event}

Let $H$ be the last epoch of Step~3 reached by time~$T$,
with $H=0$ if none.
Define the good event
$\cE_3 = \cE_{3,\mathrm{loc}} \cap \cE_{3,\mathrm{ref}}$, where
\begin{align*}
  \cE_{3,\mathrm{loc}} &\coloneqq
    \bigcap_{h=1}^{H}\bigcap_{a\in\supp(\rho_h)}
    \left\{
      \mu_a \in [L_a,U_a]
      \;\text{ and }\;
      U_a - L_a \le \frac{8\sigma}{\sqrt{B}}
    \right\}, \\[4pt]
  \cE_{3,\mathrm{ref}} &\coloneqq
    \bigcap_{h=1}^{H}\bigcap_{b\in\cA_h}
    \left\{
      \left|\hat\mu_{b,h} - \mu_b\right| \le \varepsilon_h
    \right\}.
\end{align*}

\begin{lemma}[Good event probability for Step~3]
\label{lem:step3_good_event}
$\P(\cE_3) \ge 1-\delta/3$.
The same bound holds conditional on the history before Step~3.
\end{lemma}

\begin{proof}
The proof follows the same structure as Lemma~\ref{lem:exp_good_event};
we detail only the steps that differ, primarily the shift from bounding
individual scalar estimates $\hat{Y}_{a,h}$ to bounding the aggregated estimates $b^\top\hat\theta_h^{(j)}$, and union-bounding over the active sets $\cA_h$ rather than the core sets $\supp(\rho_h)$.

\textbf{Part 1: Deterministic bounds.}
Parts 1.1 and 1.2 of Lemma~\ref{lem:exp_good_event} apply similarly.
By Lines~\ref{algline:poly_eps_h} and~\ref{algline:poly_uh_a} of Algorithm~\ref{alg:poly_action_set},
if epoch~$\widetilde H$ is reached, its full refinement allocation
satisfies
\begin{equation*}
 K\sum_{a\in\supp(\rho_{\widetilde H})}u_{\widetilde H}(a)
  \ge 32C_VdBK\log_2T\cdot4^{\widetilde H}
  \ge 32C_VdBK\log_2T\cdot T^2
  >T,
\end{equation*}
where the first inequality uses
$u_{\widetilde H}(a)\ge
32C_VdB\log_2T\cdot4^{\widetilde H}\rho_{\widetilde H}(a)$
from Line~\ref{algline:poly_uh_a} and
$\sum_a\rho_{\widetilde H}(a)=1$.
Thus epoch~$\widetilde H$ cannot be completed within~$T$ pulls,
giving $H\le\widetilde H$ deterministically. The core set bound $|\supp(\rho_h)| \le S_{\max}$ holds by~\eqref{eq:near_optimal_design}. Finally, since the active set is obtained by successive elimination, we have $\cA_h \subseteq \cA$, and hence $|\cA_h| \le \lvert \cA \rvert$ for every epoch $h$.

\textbf{Part 2: Localization event $\cE_{3,\mathrm{loc}}$.}
This is similar to Part 2 of Lemma~\ref{lem:exp_good_event}, with failure probability
$\delta_{\mathrm{loc}} = \delta/(6S_{\max}\widetilde{H})$.
Taking the union bound gives $\P(\cE_{3,\mathrm{loc}}^c) \le \delta/6$.

\textbf{Part 3: Refinement event $\cE_{3,\mathrm{ref}}$.}
We show that $\P(\cE_{3,\mathrm{ref}}^c \cap \cE_{3,\mathrm{loc}})\le\delta/6$.

Fix an epoch index~$h\in\{1,\ldots,\widetilde H\}$.
As in Part~3 of Lemma~\ref{lem:exp_good_event}, condition on a fixed
post-localization history~$\mathcal G_h$ with valid intervals and enough
remaining pulls for the allocated refinement stage. The active set, design, allocations, and intervals are fixed, while the refinement observations and random thresholds are fresh. All expectations, biases, and variances below refer to this conditional law.  For any fixed $b \in \cA_h$, the $K$ block estimates
$b^\top\hat\theta_h^{(1)},\ldots,b^\top\hat\theta_h^{(K)}$ are i.i.d.\
by the same fresh-samples argument.
Before bounding their bias and variance, we note that each invocation of \textsc{1BitRefine} uses a sample budget $n_{\rm q}(a) = u_h(a)/B$  that satisfies two simultaneous lower bounds via Line~\ref{algline:poly_uh_a}: $n_{\rm q}(a) \ge 48\log_2 T$ ensuring Corollary~\ref{cor:refinement_fixed_n} applies; and $n_{\rm q}(a) \ge 32C_V  \cdot d  \cdot \log_2 T \cdot 4^h \cdot \rho_h(a)$, which yields 
\begin{equation}
\label{eq:step_3_n_q_variance_bound}
    \frac{1}{n_{\rm q}(a)} \le 
    \frac{1}{32C_V  \cdot d  \cdot \log_2 T \cdot 4^h \cdot \rho_h(a)}
\end{equation} for the upcoming variance bound.  We then compute the following:

\begin{itemize}[leftmargin=4ex]

\item \textit{Bias.}
By writing $\mu_b = b^\top G_h^{-1} G_h\theta$ and using the definitions of $G_h$ and $\hat{\theta}_h^{(1)}$ from~\eqref{eq:theta_hat_poly}, we can decompose the bias as
\[
\E\left[b^\top\hat\theta_h^{(1)} \right] - \mu_b
= \sum_{a\in\supp(\rho_h)} \rho_h(a) \cdot (b^\top G_h^{-1}a)
  \cdot \left(\E[\hat{Y}_{a,h}^{(1)}] - \mu_a\right).
\]
Applying the triangle inequality together with
Corollaries~\ref{cor:refinement_fixed_n} and~\ref{cor:sum_rho_times_b_G_inv_a}, we obtain
\[
  \left|\E[b^\top\hat\theta_h^{(1)}] - \mu_b\right|
  \le \frac{3\sigma'}{T^2}
      \sum_{a}\rho_h(a)  |b^\top G_h^{-1}a|
  \le \frac{3\sqrt{2d}\,\sigma'}{T^2}
  = \frac{3\sqrt{2d} \cdot 2^h\varepsilon_h}{T^2}
  \le \frac{6\sqrt{2d} \varepsilon_h}{T}
  \le \frac{\varepsilon_h}{2},
\]
where the last two steps use $2^h\le2T$ and
the allocated refinement stage fits within the horizon:
\[
  T
  \ge K\sum_{a\in\supp(\rho_h)}u_h(a)
   \ge 128C_VdBK\log_2T 
   \ge 128d 
   \ge 12\sqrt{2d}.
\]

    \item \textit{Variance.} By the definition of~$\hat{\theta}_h^{(1)}$ from~\eqref{eq:theta_hat_poly} and the independence of the $1$-bit estimates, we can write the variance as
    \[
    \Var(b^\top \hat\theta_h^{(1)}) =
    \sum_{a} \rho_h(a)^2  \cdot (b^\top G_h^{-1}a)^2 \cdot
         \Var\left(\hat{Y}_{a,h}^{(1)} \right).
    \]
    Applying Corollary~\ref{cor:refinement_fixed_n} with bound~\eqref{eq:step_3_n_q_variance_bound} and Lemma~\ref{lem:sum_rho_times_squared_b_G_inv_a},
    \begin{align*}
        \mathrm{Var}(b^\top\hat\theta_h^{(1)})
          &\le \sum_{a} \rho_h(a)^2 \cdot (b^\top G_h^{-1}a)^2   \cdot\frac{C_V \sigma'^2\log_2 T}{32C_V  \cdot d \cdot  \log_2 T \cdot 4^h \cdot \rho_h(a)}\\
          &\le \frac{\sigma'^2}{32 d \cdot 4^h}
               \sum_{a}\rho_h(a) \cdot (b^\top G_h^{-1}a)^2 \\
          &\le \frac{\sigma'^2}{16 \cdot 4^h} \\
          &= \frac{\varepsilon_h^2}{16},
    \end{align*}
    where the last step substitutes $\varepsilon_h = \sigma'/2^h$.
\end{itemize}

Applying Lemma~\ref{lem:mom_statement} with failure probability $\delta/(6\lvert\cA\rvert\widetilde H)$ and combining with the bias bound via the triangle inequality gives $|\hat\mu_{b,h}-\mu_b|\le\varepsilon_h$ with probability at least $1-\delta/(6\lvert\cA\rvert\widetilde H)$ under our fixed conditioning.
Use~$\mathcal G_h$ and~$\mathcal V_h$ with the same conventions as in Part~3 of Lemma~\ref{lem:exp_good_event}, and let~$F_h$ be the event that refinement in epoch~$h$ is completed
by time~$T$ and $|\hat\mu_{b,h}-\mu_b|>\varepsilon_h$
for some~$b\in\cA_h$. A union bound over~$\cA_h$, using~$|\cA_h|\le\lvert\cA\rvert$, gives $\P(F_h\mid\mathcal G_h)\le\delta/(6\widetilde H)$ on~$\mathcal V_h$. Applying the averaging argument in~\eqref{eq:tower_property} with~$\cE_{3,\mathrm{loc}}$ in place of~$\mathcal E_{\mathrm{loc}}$, then taking a union bound over the fixed epoch indices, gives
\begin{equation*}
  \P(\cE_{3,\mathrm{ref}}^c\cap\cE_{3,\mathrm{loc}})
  \le\sum_{h=1}^{\widetilde H}\frac{\delta}{6\widetilde H}
  =\frac{\delta}{6}\,.
\end{equation*}

\textbf{Part 4: Conclusion.}
Combining this with the localization failure bound gives
\begin{equation*}
  \P(\cE_3^c)
  \le\P(\cE_{3,\mathrm{loc}}^c)
     +\P(\cE_{3,\mathrm{ref}}^c\cap\cE_{3,\mathrm{loc}})
  \le\frac{\delta}{6}+\frac{\delta}{6}
  =\frac{\delta}{3}.
\end{equation*}
Hence, $\P(\cE_3)\ge1-\delta/3$. Let~$\mathcal F_3$ be the
$\sigma$-algebra generated by the history before Step~3.
The same proof applies conditional on~$\mathcal F_3$, giving
$\P(\cE_3\mid\mathcal F_3)\ge1-\delta/3$.
\end{proof}

\subsubsection{Confidence Width and Arm Elimination}
\label{app:step3_conf_width}

Based on $\cE_{3,\mathrm{ref}} \supseteq \cE_{3}$ in the good event, we immediately have the following guarantee on the confidence width.

\begin{lemma}[Confidence width for Step~3]
\label{lem:step3_width}
Under $\cE_3$, for every epoch~$h$ whose refinement stage
is completed and every active arm $b\in\cA_h$, the estimation error satisfies $\lvert\hat\mu_{b,h} - \mu_b\rvert \le \varepsilon_h = W_h$.
\end{lemma}

\begin{lemma}[Optimal arm never eliminated]
\label{lem:step3_opt_arm}
Under $\cE_3$, the optimal arm $a^* = \arg\max_{a\in\cA}\mu_a$ satisfies
$a^* \in \cA_h$ for all epochs $h \ge 1$.
\end{lemma}
\begin{proof}
Identical to Lemma~\ref{lem:exp_opt_arm}, replacing $\cE$ with $\cE_3$
and Lemma~\ref{lem:huge_confidence_width} with Lemma~\ref{lem:step3_width}.
\end{proof}

\begin{lemma}[Suboptimality gap for active arms, Step~3]
\label{lem:step3_active_gap}
Under $\cE_3$, for all epochs $h \ge 2$ and all active arms $b \in \cA_h$,
\[
  \Delta_b \coloneqq \mu_{a^*} - \mu_b \le 4W_{h-1} = 8W_h.
\]
\end{lemma}
\begin{proof}
Identical to Lemma~\ref{lem:exp_active_gap}, replacing $\cE$ with $\cE_3$,
Lemma~\ref{lem:exp_opt_arm} with Lemma~\ref{lem:step3_opt_arm}, and
Lemma~\ref{lem:huge_confidence_width} with Lemma~\ref{lem:step3_width}.
\end{proof}

\subsubsection{Regret Bound for the Small-\texorpdfstring{$B$}{B} Regime}
\label{app:step3_regret_small_B}

\begin{theorem}[Regret of Step~3 without Warm Start]
\label{thm:step3_small_b_regret}
Under $\cE_3$, the cumulative regret incurred by the standalone phased elimination algorithm (i.e., Algorithm \ref{alg:poly_action_set} without warm start) satisfies
\[
  R_{T} = \widetilde{O}\!\left(
    dB\log\!\Big(\frac{\lvert\cA\rvert}{\delta}\Big)
    + \sqrt{dT\log\!\Big(\frac{\lvert\cA\rvert}{\delta}\Big)}
  \right).
\]
\end{theorem}

\begin{proof}
We upper bound the total regret by splitting it into localization and
refinement regret. Let $H \le \widetilde{H} = \lceil\log_2 T\rceil$
be the final epoch at which the algorithm halts.

\textbf{Localization Regret.}
Since the instantaneous regret is at most $2$ throughout, the total
localization regret over all $H$ epochs is bounded via~\eqref{eq:poly_mloc}:
\begin{equation}
  \label{eq:R_loc_poly}
  R_{\mathrm{loc}}
  \le \sum_{h=1}^H 2m_{\mathrm{loc},h}
  = O\left(H \cdot dB \cdot \log\log d  \cdot
      \log\left(\frac{dB\log T}{\delta}\right)\right)
  = \widetilde{O}\!\left(dB\log\left(\frac{1}{\delta}\right)\right),
\end{equation}
where the $\widetilde{O}(\cdot)$ notation absorbs the $\text{polylog}(d, T)$ factors.

\textbf{Refinement Regret.}
For the first epoch ($h = 1$), the trivial gap bound $\Delta_b \le 2$
and the upper bound~\eqref{eq:A-dependent_mref_bound} give
\begin{equation}
  \label{eq:R1_smallB}
  R_{\mathrm{ref},1}
  \le 2 \cdot m_{\mathrm{ref},1}
  = O(dBK \cdot \log T \cdot \log\log d)
  = \widetilde{O}(dBK).
\end{equation}
For epochs $h \ge 2$, Lemma~\ref{lem:step3_active_gap} gives
$\Delta_b \le 8W_h = 8\eps_h = 8\sigma/(2^h\sqrt{B})$ for all active
arms $b \in \cA_h$. The refinement regret for epoch $h$ is therefore
\[
  R_{\mathrm{ref},h}
  \le m_{\mathrm{ref},h} \cdot 8W_h
  = m_{\mathrm{ref},h} \cdot \frac{8\sigma}{2^h\sqrt{B}}.
\]
Epochs with $m_{\mathrm{ref},h}=0$ contribute no refinement regret.
For the remaining epochs, taking the bound
$m_{\mathrm{ref},h} =
O(dBK \cdot \log T \cdot \log\log d \cdot 4^h)$
established in~\eqref{eq:A-dependent_mref_bound}
and rearranging yields
\[
  4^{-h}
  = O\left(\frac{dBK \cdot \log T \cdot \log\log d}{m_{\mathrm{ref},h}}
  \right)
  \implies
  2^{-h} = O\!\left(\sqrt{\frac{dBK \cdot \log T \cdot \log\log d}{m_{\mathrm{ref},h}}}\right).
\]
Substituting and using $\sigma = \Theta(1)$ yields
\begin{equation*}
  R_{\mathrm{ref},h}
  = O\left(\sqrt{dK \cdot \log T \cdot \log\log d \cdot m_{\mathrm{ref},h}}\right).
\end{equation*}
Summing over $h = 2,\ldots,H$ and applying Cauchy-Schwarz inequality as in~\eqref{eq:Rh_ref_exp_CS_trick}, we obtain
\begin{equation}\label{eq:R_ref_sum_poly}
    \sum_{h=2}^H R_{\mathrm{ref},h}
  = O\left(\sqrt{dK \log T\log\log d} \cdot \sum_{h=2}^H \sqrt{m_{\mathrm{ref},h}} \right)
  = \widetilde{O}\left(\sqrt{dKT}  \right).
\end{equation}
\textbf{Total Regret.}
Combining~\eqref{eq:R_loc_poly},~\eqref{eq:R1_smallB},
and~\eqref{eq:R_ref_sum_poly},
\[
  R_T
  \le R_{\mathrm{loc}} + R_{\mathrm{ref},1} + \sum_{h=2}^H R_{\mathrm{ref},h}
  = \widetilde{O}\!\left(
      dB\log\!\Big(\frac{1}{\delta}\Big) + dBK + \sqrt{KdT}
    \right).
\]
Substituting $K = \lceil 8\log(6\lvert \cA \rvert\widetilde{H}/\delta)\rceil
= O(\log(\lvert \cA \rvert\log T/\delta))$ yields
\begin{equation} \label{eq:small_B_regret}
  R_T = \widetilde{O}\!\left(
    dB\log\!\Big(\frac{\lvert \cA \rvert}{\delta}\Big)
    + \sqrt{dT\log\!\Big(\frac{\lvert \cA \rvert}{\delta}\Big)}
  \right)
\end{equation}
as stated in Theorem~\ref{thm:step3_small_b_regret}.
\end{proof}

In the small-$B$ regime, $B = \widetilde{O}(d \log (\lvert \cA \rvert /\delta))$ implies $dB =\widetilde{O} (d^{3/2}\sqrt{B \log (\lvert \cA \rvert/\delta)})$, and so the first term in \eqref{eq:small_B_regret} scales as $d^{3/2}\sqrt{B}\log^{\frac{3}{2}}(\lvert \cA \rvert /\delta)$ up to logarithmic factors, meaning the regret is upper bounded by
\[
  R_T = \widetilde{O}\left( d^{3/2}\sqrt{B} \log^{\frac{3}{2}}\Big(\frac{\lvert\cA\rvert}{ \delta}\Big) + \sqrt{dT\log\Big(\frac{\lvert\cA\rvert}{\delta}\Big)}
  \right).
\]

\subsection{Regret Bounds for Large \texorpdfstring{$B$}{B}}
\label{app:step3_large_B}

We first explain why Theorem~\ref{thm:step3_small_b_regret} is insufficient
for large $B$. The first-epoch regret $\widetilde{O}(dBK)$ in that theorem
arises from the trivial gap bound $\Delta_b \le 2$: without prior
information, no tighter bound is available at epoch~$1$.
When $B \gg d$, this is wasteful because the $dBK$ term can dominate both the
communication bottleneck $\Omega(B\log\lvert \cA \rvert)$ and the standard linear bandit term
$\Omega(\sqrt{dT\log\lvert \cA \rvert})$.  A key observation towards improving this is that when $B \gg d$, a single batch is large enough to
run a full experimental design \emph{within the batch}.

Specifically, the first two steps of Algorithm~\ref{alg:poly_action_set} address this by
prepending two pre-processing stages before the main phased elimination loop.
Step~1 identifies a \emph{safe arm} $a_{\mathrm{safe}}$ with suboptimality
gap $\widetilde{O}(\log^{3/2}(\lvert\cA\rvert/\delta)\sqrt{d/B})$. Step~2 uses $a_{\mathrm{safe}}$ to run a
\emph{warm-start epoch} that filters the initial active set down to
$\cA_1 \subseteq \cA$ with $\Delta_b = O(\sqrt{d/B})$ for all
$b \in \cA_1$. With this tighter first-epoch gap bound, the first-epoch
regret shrinks from $\widetilde{O}(dBK)$ to an improved term of $\widetilde{O}(d^{3/2}\sqrt{B}K)$. 

We now present the pseudocode for each pre-processing stage, followed by
its regret analysis. Throughout, we assume
$B \ge 8K_\star d(\log\log d + 11)$, where
$K_\star = \lceil 8\log(3\lvert \cA \rvert\lceil\log_2\lvert \cA \rvert\rceil/\delta)\rceil$.

\subsubsection{Step 1: Safe Arm Identification}

\begin{algorithm}
\caption{\textsc{SafeArm}$(\cA, B, \delta)$: Safe Arm Identification}
\label{alg:safe_arm}
\begin{algorithmic}[1]
\State \textbf{Input:} Action set $\cA \subset \mathbb{R}^d$, batch size
  $B$, confidence $\delta$.
\State \textbf{Initialize:} Active set $\cA_1 = \cA$.
\State Set number of Median-of-Means (MoM) blocks 
$K_\star = \Big\lceil 8 \log \Big(\dfrac{3 \lvert \cA \rvert \cdot\lceil \log_2 \lvert \cA \rvert\rceil }{\delta}\Big) \Big \rceil$.
\For{$h = 1$ \textbf{to} $\lceil\log_2\lvert \cA \rvert\rceil$}
  \If{$|\cA_h|=1$}
    \State \textbf{return} the unique arm in $\cA_h$ as $a_{\mathrm{safe}}$.
  \EndIf
  \State Arbitrarily split $\cA_h$ into two equal groups $\cG_A$,
    $\cG_B$ (up to rounding).
  \State Compute design $\rho_h \colon \cA_h \to [0,1]$
    satisfying~\eqref{eq:near_optimal_design}.
  \State Assign pulls $u_h(a) = \lceil B\rho_h(a)/(2K_\star)\rceil$ for each $a \in \supp(\rho_h)$.
  \Statex
  \Statex \quad \textbf{Agent execution:}
  \For{block $j = 1,\ldots,K_\star$}
    \State Pull arm $a$ a total of $u_h(a)$ times and compute sample mean $\bar{Y}_{a,h}^{(j)}$ of rewards for \Statex  \hspace{\algorithmicindent} \hspace{\algorithmicindent}  each arm $a \in \supp(\rho_h)$.
    \State Compute $\hat{\theta}_h^{(j)} \in \bR^d$ according to \eqref{eq:theta_hat_poly} if ${\rm span}(\cA_h) = \bR^d$, or similarly with
    \Statex \hspace{\algorithmicindent} \hspace{\algorithmicindent}  $\hat{\theta}_h^{(j)} \in {\rm span}(\cA_h)$  otherwise (see Remark \ref{rem:not_span_Rd}).
  \EndFor
    \State Compute reward estimate $\hat{\mu}_{b,h} = {\rm median}\big\{ b^\top \hat{\theta}_h^{(1)} , \dots, b^\top \hat{\theta}_h^{(K_\star)}\big\}$
    for each $b \in \cA_h$.
    \State Transmit bit $Z_h = 
    \mathds{1}\{ \max_{i \in \cG_A} \hat{\mu}_{i,h} \ge \max_{j \in \cG_B} \hat{\mu}_{j,h}  \}$ to learner.

    \Statex
    
    \Statex \quad  \textbf{Learner performs elimination:}
  \State \textbf{if} $Z_h = 1$ \textbf{then}
    $\cA_{h+1} \leftarrow \cG_A$
    \textbf{else} $\cA_{h+1} \leftarrow \cG_B$.
\EndFor
\State \textbf{Output:} The single arm remaining in $\cA_{\lceil \log_2 \lvert \cA \rvert \rceil+1}$ as $a_{\text{safe}}$.
\end{algorithmic}
\end{algorithm}

Algorithm~\ref{alg:safe_arm} runs at most $\lceil\log_2\lvert \cA \rvert\rceil$ rounds of
tournament-style bisection. In each round, the active set is split in two,
both halves are evaluated via a $G$-optimal design using the unquantized
sample means available within a single batch, and the half with the higher
estimated maximum is retained. After at most $\lceil\log_2\lvert \cA \rvert\rceil$ rounds,
a single arm remains.

\textbf{Regret of Step~1.}
Algorithm~\ref{alg:safe_arm} uses at most $B$ arm pulls in total per round over at most $\lceil\log_2\lvert\cA\rvert\rceil$ rounds, so the total
number of arm pulls is at most $B\lceil\log_2\lvert\cA\rvert\rceil$.
Since the instantaneous regret is at most $2$, the regret contribution
$R_{\mathrm{safe}}$ in this part satisfies
\[
  R_{\mathrm{safe}}
  \le 2B\lceil\log_2\lvert \cA \rvert\rceil
  = O(B\log\lvert \cA \rvert).
\]

\textbf{Guarantee of Step~1.}
The following lemma, proved in Appendix~\ref{subsec:safe_arm_proof},
bounds the suboptimality of the returned arm.

\begin{lemma} \label{lem:safe_arm}
    (Safe Arm Guarantees) Denote the output of the \textup{\textsc{SafeArm}} algorithm as $a_{\rm safe} \in \cA$. Then for any $B \geq 8K_\star d(\log \log d + 11)$, the event
    \begin{equation}  \label{eq:Delta_a_safe} 
    \cE_1 \coloneqq \left\{ \Delta_{a_{\rm safe}} \leq 16 \sigma \lceil\log_2 \lvert \cA \rvert \rceil \sqrt{\frac{2d}{B} \log \Big(\frac{3 \lvert \cA \rvert \cdot \lceil \log_2 \lvert \cA \rvert \rceil}{\delta}\Big)} \right\}
    \end{equation}
    holds with probability at least $1 - \delta/3$. 
\end{lemma}

\subsubsection{Step 2: Warm Start}

\begin{algorithm}
\caption{$\textsc{WarmStart}(\cA, a_{\mathrm{safe}}, B, \delta)$:
  Warm Start Epoch}
\label{alg:warm_start}
\begin{algorithmic}[1]
\State \textbf{Input:} Finite action set $\cA \subset \mathbb{R}^d$,
  safe arm $a_{\mathrm{safe}} \in \cA$, batch size $B$, confidence $\delta$.
\State Set $K_0 = \lceil 8\log(6\lvert \cA \rvert/\delta)\rceil$ and
  $S_{\max} = 4d(\log\log d + 11)$.
\State Compute design $\rho_0 \colon \cA \to [0,1]$
  satisfying~\eqref{eq:near_optimal_design}.
\For{each $a \in \supp(\rho_0)$}
  \State Set $u_0(a) = \lceil B\rho_0(a)\rceil$.
  \State $[L_a,U_a] \leftarrow
    \textsc{1BitLocalize}\Big(a,\;
      \delta_{\mathrm{loc}} = \dfrac{\delta}{6S_{\max}},\;
      \sigma' = \dfrac{\sigma}{\sqrt{u_0(a)}}
    \Big)$
    (from Appendix~\ref{subsec:1Bitlocalize}),
    \label{line:warm_start_localize}
    \Statex \hspace{\algorithmicindent} \ \ with $a_{\mathrm{safe}}$ filling the remaining $B - u_0(a)$ slots
    of each batch.
\EndFor
\For{block $j = 1,\ldots,K_0$}
  \State $\hat{Y}_{a,0}^{(j)} \leftarrow
    \textsc{1BitRefine}(a,[L_a,U_a],\lceil C_V\log_2 T\rceil)$
    (from Appendix~\ref{subsec:1bitRefine}) 
    \label{line:warm_start_refine}
    \Statex \hspace{\algorithmicindent} \ \ for each
    $a \in \supp(\rho_0)$, with $a_{\mathrm{safe}}$ filling remaining
    slots.
  \State Compute $\hat\theta_0^{(j)}$ via~\eqref{eq:theta_hat_poly}
    if $\mathrm{span}(\cA) = \mathbb{R}^d$, or
    $\hat\theta_0^{(j)} \in \mathrm{span}(\cA)$ otherwise
    (see Remark~\ref{rem:not_span_Rd}).
\EndFor
\For{each $b \in \cA$}
  \State $\hat\mu_{b,0} = \mathrm{median}\{b^\top\hat\theta_0^{(1)},
    \ldots, b^\top\hat\theta_0^{(K_0)}\}$.
\EndFor
\State \textbf{Output:}
\label{algline:warm_start_elim_rule}
  $\cA_1 = \left\{a \in \cA :
    \max_{b\in\cA}\hat\mu_{b,0} - \hat\mu_{a,0}
    \le 8\sigma\sqrt{2d/B}\right\}$.
\end{algorithmic}
\end{algorithm}

Algorithm~\ref{alg:warm_start} runs a single epoch of phased elimination
with a key twist: in every batch, $a_{\mathrm{safe}}$ fills all slots not
needed by a core-set arm. This means core-set arms are pulled only
$u_0(a) = \lceil B\rho_0(a)\rceil \le B$ times per batch, while
$a_{\mathrm{safe}}$ absorbs the remaining $B - u_0(a)$ pulls. The effective
standard deviation for estimating $\mu_a$ is therefore $\sigma/\sqrt{u_0(a)}$
rather than $\sigma/\sqrt{B}$, which is the key difference from the main
phased elimination loop. Arms whose estimated gap exceeds $8\sigma\sqrt{2d/B}$
are eliminated, producing the warm-started active set $\cA_1$.

\textbf{Regret of Step~2.}
The warm-start regret has three components: localization, refinement, and
safe arm pulls. Since $B \ge 8K_\star d(\log\log d + 11) \ge 2S_{\max} \ge |\supp(\rho_0)|$ by assumption and~\eqref{eq:near_optimal_design}, and $\lceil x\rceil \le x +1$, we have
\[
\sum_{a\in\supp(\rho_0)} u_{0}(a) \le
\sum_{a\in\supp(\rho_0)} \left(B \rho_0(a)  + 1  \right)
= B + |\supp(\rho_0)|
\le 2B.
\]

\textit{Localization.}
By Proposition~\ref{prop:localize} with $\sigma' = \sigma/\sqrt{u_0(a)}$
and failure probability $\delta_{\mathrm{loc}} = \delta/(6S_{\max})$, each invocation of $\textsc{1BitLocalize}(a, \delta_{\mathrm{loc}}, \sigma')$  uses at most
\[
    n_{\mathrm{loc},0}(a)
    = C_{\mathrm{loc}}\left(\frac{1}{2}\log\left(\frac{u_0(a)}{\sigma^2}\right)
        + \log\left(\frac{6S_{\max}}{\delta}\right)\right)
    = O\left( \log \left(\frac{dB}{\delta}\right)\right) 
\]
queries, where we use $u_0(a) \le B$, and $\sigma = \Theta(1)$, and $S_{\max} = O(d \log \log d) = \widetilde{O}(d)$.
Since each query costs $u_0(a)$ arm pulls and the regret of each pull is at most 2, summing the regret over $\supp(\rho_0)$ yields a contribution $R_{\mathrm{ws,loc}}$ to the regret satisfying
\begin{align*}
    R_{\mathrm{ws,loc}}
  \le 2\sum_{a\in\supp(\rho_0)} u_0(a) \cdot n_{\mathrm{loc},0}(a) 
  = O\left(B \log\left(\frac{dB}{\delta}\right)\right).
\end{align*}

\textit{Refinement.}
Each block uses $\lceil C_V\log_2 T\rceil$ queries per arm, each costing
$u_0(a)$ pulls and each pull incurs regret at most 2. Summing over $K_0 = O(\log(\lvert \cA \rvert/\delta))$ blocks and $\supp(\rho_0)$ yields a contribution $R_{\mathrm{ws,ref}}$ to the regret satisfying
\[
  R_{\mathrm{ws,ref}}
  \le 2K_0\sum_{a\in\supp(\rho_0)}u_0(a)\cdot\lceil C_V\log_2 T\rceil
  = O\left(B\log\left(\frac{\lvert \cA \rvert}{\delta}\right)\log T\right).
\]

\textit{Safe arm pulls.}
From the above discussion, the total number of batches consumed by localization and refinement is at most
\[
    \sum_{a \in {\supp(\rho_0)}} \left( n_{{\rm loc},0}(a) + K_0 \cdot  \lceil C_V\log_2 T\rceil \right) 
    = O\left(S_{\max} \cdot \log \Big(\frac{dB  \lvert \cA \rvert }{ \delta}\Big) \cdot \log (T)\right).
\]
In each batch, $a_{\mathrm{safe}}$ is pulled at most $B$ times.
Conditioned on $\cE_1$, Lemma~\ref{lem:safe_arm} gives
$\Delta_{a_{\mathrm{safe}}} =
\widetilde{O}\Big(\log\lvert \cA \rvert\sqrt{\frac{d}{B}\log\frac{\lvert \cA \rvert}{\delta}}\Big)$, so we get a contribution $R_{\mathrm{ws,safe}}$ to the regret satisfying
\begin{align*}
  R_{\mathrm{ws,safe}}
  \le \Delta_{a_{\mathrm{safe}}} \cdot O\left(B \cdot S_{\max} \cdot \log \Big(\frac{dB  \lvert \cA \rvert }{ \delta}\Big) \cdot \log (T)\right)
   = \widetilde{O}\left(
       d^{3/2}\sqrt{B} \cdot \log^{5/2}\left(\frac{\lvert \cA \rvert}{\delta}\right)
     \right).
\end{align*}

Combining all three components, we obtain a total contribution $R_{\mathrm{ws}}$ satisfying
\begin{equation}
  \label{eq:R_ws}
  R_{\mathrm{ws}}
  = R_{\mathrm{ws,loc}} + R_{\mathrm{ws,ref}} + R_{\mathrm{ws,safe}}
  = \widetilde{O} \left(
      B\log \left(\frac{\lvert \cA \rvert}{\delta}\right)
      + d^{3/2}\sqrt{B}\log^{5/2} \left(\frac{\lvert \cA \rvert}{\delta}\right)
    \right).
\end{equation}

\textbf{Guarantee of Step~2.}
The following lemma, proved in Appendix~\ref{subsec:warm_start_proof},
bounds the suboptimality of every arm in $\cA_1$.

\begin{lemma}[Warm Start Guarantees]
\label{lem:warm_start}
Let $\cA_1$ be the output of \textup{\textsc{WarmStart}}, and let $\cE_2$ be the event that
(i) $\Delta_a \le 16\sigma\sqrt{2d/B}$ for all $a \in \cA_1$, and
(ii) $a^* \in \cA_1$.
Then $\P(\cE_2 \mid \cE_1) \ge 1 - \delta/3$.
\end{lemma}

\subsubsection{Step 3: Phased Elimination and Total Regret}

With a warm-started active set $\cA_1$ in hand, the main phased elimination
loop of Algorithm~\ref{alg:poly_action_set} proceeds exactly as in the
small-$B$ regime, starting from $\cA_1$ rather than $\cA$.

We bound the Step~3 regret conditioned on $\cE_1 \cap \cE_2 \cap \cE_3$.
Since $a^* \in \cA_1$ by condition~(ii) of $\cE_2$,
Lemma~\ref{lem:step3_opt_arm} continues to hold from epoch~$1$ onward.

Let $R_{\mathrm{loc}}$ denote the contribution to the regret from localization, and let $R_{\mathrm{ref},h}$ denote the contribution from refinement in epoch $h$.  
For all epochs $h \ge 1$, since $\cA_h \subseteq \cA_1$, condition~(i) of $\cE_2$ guarantees that $\Delta_b \le 16\sigma\sqrt{2d/B}$ for all active arms. Therefore, the total localization regret across all $H$ epochs is bounded by
\begin{equation}
  \label{eq:new_R_loc}
  R_{\mathrm{loc}} 
  \le 16\sigma\sqrt{\frac{2d}{B}} \sum_{h=1}^H m_{\mathrm{loc}, h} 
  = \widetilde{O}\left(\sqrt{\frac{d}{B}} \cdot dB \log\left(\frac{1}{\delta}\right)\right) 
  = \widetilde{O}\left(d^{3/2}\sqrt{B} \log\left(\frac{1}{\delta}\right)\right),
\end{equation}
where we used the localization cost bound $\sum_h m_{\mathrm{loc},h} = \widetilde{O}(dB \log(1/\delta))$ from~\eqref{eq:R_loc_poly} and $\sigma = \Theta(1)$.
For the refinement regret in the first epoch ($h = 1$), we similarly use the warm-start gap bound:
\begin{equation}
  \label{eq:new_R1_ref}
  R_{\mathrm{ref}, 1}
  \le 16\sigma\sqrt{\frac{2d}{B}} \cdot m_{\mathrm{ref}, 1}
  = \widetilde{O}\!\left(d^{3/2}\sqrt{B}\,K\right),
\end{equation}
where we used~\eqref{eq:A-dependent_mref_bound}. For epochs $h \ge 2$, since $a^* \in \cA_1$,
Lemma~\ref{lem:step3_active_gap} applies and the argument is identical to
that of Theorem~\ref{thm:step3_small_b_regret} from epoch~2 onward,
giving $\sum_{h=2}^H R_{\mathrm{ref},h} = \widetilde{O}(\sqrt{KdT})$.
Adding the localization and refinement components and substituting $K = \widetilde{O}(\log(\lvert \cA \rvert/\delta))$, the total regret $R_{\mathrm{phased}}$ from the phased elimination part satisfies
\begin{equation}
  \label{eq:R_pe}
  R_{\mathrm{phased}}
  = R_{\mathrm{loc}} + R_{\mathrm{ref}, 1} + \sum_{h=2}^H R_{\mathrm{ref},h}
  = \widetilde{O}\!\left(
      d^{3/2}\sqrt{B}\log\!\Big(\frac{\lvert \cA \rvert}{\delta}\Big)
      + \sqrt{dT\log\!\Big(\frac{\lvert \cA \rvert}{\delta}\Big)}
    \right).
\end{equation}

\textbf{Total Regret.}
Combining $R_{\mathrm{safe}}$, $R_{\mathrm{ws}}$, and $R_{\mathrm{phased}}$,
conditioned on $\cE_1 \cap \cE_2 \cap \cE_3$,
\begin{equation*}
  R_T
  = \widetilde{O}\!\left(
      B\log\!\Big(\frac{\lvert \cA \rvert}{\delta}\Big)
      + d^{3/2}\sqrt{B}\log^{5/2}\!\Big(\frac{\lvert \cA \rvert}{\delta}\Big)
      + \sqrt{dT\log\!\Big(\frac{\lvert \cA \rvert}{\delta}\Big)}
    \right).
\end{equation*}
It remains to show that $\cE_1 \cap \cE_2 \cap \cE_3$ holds with probability
at least $1-\delta$.
Lemmas~\ref{lem:safe_arm} and~\ref{lem:warm_start} give $\P(\cE_1^c)\le\delta/3$ and $\P(\cE_2^c\cap\cE_1)\le\delta/3$. 
Lemma~\ref{lem:step3_good_event} gives $\P(\cE_3^c\mid\mathcal F_3)\le\delta/3$, where~$\mathcal F_3$ is the $\sigma$-algebra generated by the history before Step~3.
Since~$\cE_1\cap\cE_2\in\mathcal F_3$, it follows that $\P(\cE_3^c\cap\cE_1\cap\cE_2)\le\delta/3$. Therefore,
\begin{align*}
\P\bigl((\cE_1\cap\cE_2\cap\cE_3)^c\bigr)
  =\P(\cE_1^c)+\P(\cE_2^c\cap\cE_1)
    +\P(\cE_3^c\cap\cE_1\cap\cE_2)
  \le\frac{\delta}{3}+\frac{\delta}{3}+\frac{\delta}{3}
   =\delta.
\end{align*}
Hence, $\P(\cE_1\cap\cE_2\cap\cE_3)\ge1-\delta$ as desired.

\subsection{Proofs of Auxiliary Lemmas}
\label{app:aux_lemmas}

We now prove Lemmas~\ref{lem:safe_arm} and~\ref{lem:warm_start}, which
establish the guarantees of Algorithms~\ref{alg:safe_arm}
and~\ref{alg:warm_start} respectively.

\subsubsection{Proof of Lemma~\ref{lem:safe_arm} (Safe Arm Guarantees)}
\label{subsec:safe_arm_proof}

\textbf{Feasibility of arm pulls per batch.}
Set $m_\star = B/(2K_\star)$. Since
\[
    u_h(a) = \lceil B\rho_h(a)/(2K_\star)\rceil \le m_\star\rho_h(a) + 1,
\]
we have
\[
  K_\star\sum_{a\in\supp(\rho_h)} u_h(a)
  \le K_\star(m_\star + |\supp(\rho_h)|)
  \le \frac{B}{2} + 4K_\star d(\log\log d + 11)
  \le B,
\]
using $|\supp(\rho_h)| \le S_{\max}$ and the assumption on $B$. Any
remaining slots are filled by an arbitrary arm whose rewards are ignored by the agent.

\textbf{Concentration of arm estimates.}
Fix a round~$h$ and let~$\cF_{h-1}$ be the $\sigma$-algebra generated by the history immediately before its arm pulls, including~$\cA_h$, $\rho_h$, and~$\{u_h(a)\}$. Condition on a fixed realization of this history on which the round's
batch can be completed by time~$T$. Since each block~$j$ uses fresh arm pulls with the same fixed protocol, the $K_\star$ block estimates $b^\top\hat\theta_h^{(1)},\ldots,b^\top\hat\theta_h^{(K_\star)}$ are i.i.d.\ for any fixed~$b\in\cA_h$.

Because $\bar{Y}_{a,h}^{(j)}$ is the sample mean of $u_h(a)$ independent
rewards, $\E[\bar{Y}_{a,h}^{(j)}] = \mu_a$ and
$\mathrm{Var}(\bar{Y}_{a,h}^{(j)}) \le \sigma^2/u_h(a)$. Since
$\E[b^\top\hat\theta_h^{(j)}] = \mu_b$ (no bias), the variance
of the block estimate satisfies
\begin{align*}
  \mathrm{Var}(b^\top\hat\theta_h^{(j)})
  &\le \sum_{a\in\supp(\rho_h)}
     (b^\top G_h^{-1}a)^2\rho_h(a)^2
     \cdot\frac{\sigma^2}{u_h(a)}
   \le \frac{\sigma^2}{m_\star}
     \sum_{a\in\supp(\rho_h)}
     (b^\top G_h^{-1}a)^2\rho_h(a)
   \le \frac{2d\sigma^2}{m_\star},
\end{align*}
using $u_h(a) = \lceil m_\star\rho_h(a) \rceil \ge m_\star\rho_h(a)$ and
Lemma~\ref{lem:sum_rho_times_squared_b_G_inv_a}.

Applying Lemma~\ref{lem:mom_statement} to the $K_\star$ i.i.d.\ block
estimates with variance at most $2d\sigma^2/m_\star$, using failure
probability $\delta/(3\lvert \cA \rvert\lceil\log_2\lvert \cA \rvert\rceil)$ and substituting
$K_\star m_\star = B/2$, we have
\[
  |\hat\mu_{b,h} - \mu_b|
  \le \sqrt{\frac{64d\sigma^2}{K_\star m_\star}
      \log\!\Big(\frac{3\lvert \cA \rvert\lceil\log_2\lvert \cA \rvert\rceil}{\delta}\Big)}
  = 8\sigma\sqrt{\frac{2d}{B}
      \log\!\Big(\frac{3\lvert \cA \rvert\lceil\log_2\lvert \cA \rvert\rceil}{\delta}\Big)}
  \eqqcolon \eps_{\mathrm{safe}},
\]
with probability at least $1 - \delta/(3\lvert \cA \rvert\lceil\log_2\lvert \cA \rvert\rceil)$
under our conditioning.  A union bound over~$b\in\cA_h$, using~$|\cA_h|\le\lvert\cA\rvert$, gives
\begin{equation*}
  \P\Bigl(
    E_h
    \;\Big|\;\mathcal F_{h-1}
  \Bigr)
  \le\frac{\delta}{3\lceil\log_2\lvert\cA\rvert\rceil}
  \quad \text{where} \quad
  E_h \coloneqq \{\exists\,b\in\cA_h:
    |\hat\mu_{b,h}-\mu_b|>\eps_{\mathrm{safe}} \}.
\end{equation*}
Define
\begin{equation*}
  \cE_{\mathrm{safe}}
  \coloneqq
  \bigcap_{h=1}^{\lceil\log_2\lvert\cA\rvert\rceil}
(E_h)^c
  =
  \bigcap_{h=1}^{\lceil\log_2\lvert\cA\rvert\rceil}
  \bigcap_{b\in\cA_h}
  \left\{|\hat\mu_{b,h}-\mu_b|\le\eps_{\mathrm{safe}}\right\}.
\end{equation*}
Taking a union bound over rounds and applying the tower property give
\begin{align*}
  \P(\cE_{\text{safe}}^c) = \P\left( \bigcup_{h=1}^{\lceil\log_2\lvert\cA\rvert\rceil} E_h \right) \le \sum_{h=1}^{\lceil\log_2\lvert\cA\rvert\rceil} \P(E_h)
  =
  \sum_{h=1}^{\lceil\log_2\lvert\cA\rvert\rceil}
    \E\left[
      \P\Bigl(
       E_h
        \;\Big|\;\mathcal F_{h-1}
      \Bigr)
    \right]
    \le\frac{\delta}{3}.
\end{align*}

\textbf{Suboptimality gap of the safe arm.}
On completion of \textsc{SafeArm}, condition on $\cE_{\mathrm{safe}}$. In round $h$, let
$a_h^* = \arg\max_{a\in\cA_h}\mu_a$. If the group containing $a_h^*$ is
eliminated, there exists $b$ in the retained group with
$\hat\mu_{b,h} \ge \hat\mu_{a_h^*,h}$, so
\[
  \mu_b
  \ge \hat\mu_{b,h} - \eps_{\mathrm{safe}}
  \ge \hat\mu_{a_h^*,h} - \eps_{\mathrm{safe}}
  \ge \mu_{a_h^*} - 2\eps_{\mathrm{safe}}.
\]
Hence $\mu_{a_{h+1}^*} \ge \mu_{a_h^*} - 2\eps_{\mathrm{safe}}$,
and induction from $\mu_{a_1^*} = \mu_{a^*}$ yields 
\[
    \mu_{a_{h+1}^*} \ge \mu_{a^*} - 2h\,\eps_{\mathrm{safe}}.
\]
In other words, $\Delta_{a_{h+1}^*} \le 2h\,\eps_{\mathrm{safe}}$.
When the algorithm returns, at most
$\lceil\log_2\lvert\cA\rvert\rceil$ rounds have been completed, and the active set contains only a single arm  $a_{\mathrm{safe}}$. 
Applying the induction bound gives
$\Delta_{a_{\mathrm{safe}}} \le 2\lceil\log_2\lvert \cA \rvert\rceil\,\eps_{\mathrm{safe}}$. Substituting $\eps_{\mathrm{safe}}$ yields~\eqref{eq:Delta_a_safe}, i.e., event $\cE_1$ holds. It follows that $\P(\cE_1) \ge \P(\cE_{\mathrm{safe}}) \ge 1-\delta/3$.

\subsubsection{Proof of Lemma~\ref{lem:warm_start} (Warm Start Guarantees)}
\label{subsec:warm_start_proof}

Let~$\mathcal F_{\mathrm{ws}}$ be the $\sigma$-algebra generated
by the history immediately before Algorithm~\ref{alg:warm_start}.
In Parts~1--3 below, we condition on a fixed entry history,
including the realized arm~$a_{\mathrm{safe}}$;
probabilities in those parts refer to this conditional law. Localization and refinement observations are fresh, and filler rewards are not used in the estimates. We establish $\P(\cE_{\mathrm{ws}}^c\mid\mathcal F_{\mathrm{ws}})\le\delta/3$ uniformly over these histories, where $\cE_{\mathrm{ws}}\coloneqq\cE_{\mathrm{ws,loc}}\cap\cE_{\mathrm{ws,ref}}$ (see below for the definitions).

\textbf{Part 1: Localization event $\cE_{\mathrm{ws,loc}}$.}
For each $a \in \supp(\rho_0)$, \textsc{1BitLocalize} is called with
effective standard deviation $\sigma' = \sigma/\sqrt{u_0(a)}$ and failure
probability $\delta/(6S_{\max})$. Since $|\supp(\rho_0)| \le S_{\max}$,
Proposition~\ref{prop:localize} and a union bound give that
\begin{equation}
  \label{eq:E_ws_loc}
  \cE_{\mathrm{ws,loc}} \coloneqq
  \bigcap_{a\in\supp(\rho_0)}\!\!
  \left\{
    \mu_a \in [L_a,U_a]
    \;\text{ and }\;
    U_a-L_a \le \frac{8\sigma}{\sqrt{u_0(a)}}
  \right\}
\end{equation}
has conditional probability $\P(\cE_{\mathrm{ws,loc}}\mid\mathcal F_{\mathrm{ws}})\ge1-\delta/6$.

\textbf{Part 2: Refinement event $\cE_{\mathrm{ws,ref}}$.}
Let~$\mathcal G_0$ be the $\sigma$-algebra generated by the history after all localization calls in Algorithm~\ref{alg:warm_start}.
Condition on a fixed realization of~$\mathcal G_0$ in which localization is completed, all its intervals are valid (i.e., $\cE_{\mathrm{ws,loc}}$ holds), and enough pulls remain to complete the allocated refinement stage.  All expectations, biases, and variances below refer to the conditional law given~$\mathcal G_0$.
Since each block~$j$ uses independent arm pulls with the same fixed parameters, the $K_0$ estimates $b^\top\hat\theta_0^{(1)},\ldots,b^\top\hat\theta_0^{(K_0)}$ are i.i.d.\
for any fixed $b \in \cA$. We bound their variance and bias as follows:

\begin{itemize}[leftmargin=4ex]
  \item \textit{Variance.}
  Each call to \textsc{1BitRefine} for arm $a$ uses
  $n_{\mathrm{q}} = \lceil C_V\log_2 T\rceil$ queries. Because arm $a$ is pulled only $u_0(a)$ times per query (with the rest of the batch padded by $a_{\mathrm{safe}}$), the effective standard deviation is $\sigma' = \sigma/\sqrt{u_0(a)}$. 
  Applying the bound from Corollary~\ref{cor:refinement_fixed_n} yields
  \[
    \mathrm{Var}(\hat{Y}_{a,0}^{(j)}) 
    \le \frac{C_V (\sigma')^2 \log_2 T}{n_{\mathrm{q}}} 
    \le \frac{C_V (\sigma^2 / u_0(a)) \log_2 T}{C_V \log_2 T} 
    = \frac{\sigma^2}{u_0(a)}.
  \]
  Combining this with $u_0(a) \ge B\rho_0(a)$ and Lemma~\ref{lem:sum_rho_times_squared_b_G_inv_a}, we obtain
  \[
  \mathrm{Var}(b^\top\hat\theta_0^{(j)})
    = \sum_{a}
      (b^\top G_0^{-1}a)^2\rho_0(a)^2
      \cdot\mathrm{Var}(\hat{Y}_{a,0}^{(j)}) 
    \le \frac{\sigma^2}{B}
      \sum_{a}
      (b^\top G_0^{-1}a)^2\rho_0(a) 
    \le \frac{2d\sigma^2}{B}.
  \]

  \item \textit{Bias.}
  Corollary~\ref{cor:refinement_fixed_n} gives
  $|\mathrm{Bias}(\hat{Y}_{a,0}^{(j)})| \le 3\sigma/(T^2\sqrt{u_0(a)})
  \le 3\sigma/T^2$. By writing $\mu_b = b^\top G_0^{-1} G_0\theta$ and using the triangle inequality alongside Corollary~\ref{cor:sum_rho_times_b_G_inv_a}, we obtain
  \[
    |\E[b^\top\hat\theta_0^{(j)}] - \mu_b|
    \le \frac{3\sigma}{T^2}
       \sum_{a\in\supp(\rho_0)}\rho_0(a)|b^\top G_0^{-1}a|
    \le \frac{3\sqrt{2d}\,\sigma}{T^2}
    \le 2\sigma\sqrt{\frac{2d}{B}},
  \]
  where the last step uses $T \ge B$ and $B \ge 4d(\log\log d+11) \ge 4$,
giving $T^2 \ge B^2 \ge 4B \ge \frac{3}{2}\sqrt{B}$.
\end{itemize}

Applying Lemma~\ref{lem:mom_statement} to the $K_0$ i.i.d.\ block
estimates with variance at most $2d\sigma^2/B$ and failure probability
$\delta/(6\lvert \cA \rvert)$, then combining with the bias bound via the triangle
inequality, we obtain
\[
  |\hat\mu_{b,0} - \mu_b|
  \le \sqrt{\frac{64d\sigma^2}{BK_0}
      \log\!\Big(\frac{6\lvert \cA \rvert}{\delta}\Big)}
    + 2\sigma\sqrt{\frac{2d}{B}}
  \le 4\sigma\sqrt{\frac{2d}{B}},
\]
with probability at least $1-\delta/(6\lvert \cA \rvert)$ under our conditioning,
where the last step uses $K_0 = \lceil 8\log(6\lvert \cA \rvert/\delta)\rceil$.
Since~$\cA$ is fixed, a union bound over all~$b\in\cA$ gives
\begin{equation*}
  \P\Bigl(
    \cE_{\mathrm{ws,ref}}^c
    \;\Big|\;\mathcal G_0
  \Bigr)
  \le\frac{\delta}{6}
  \quad \text{where} \quad
  \cE_{\mathrm{ws,ref}}\coloneqq
  \bigcap_{b\in\cA}
  \left\{|\hat\mu_{b,0}-\mu_b|
       \le 4\sigma\sqrt{\frac{2d}{B}}\right\}.
\end{equation*}
Since~$\mathcal F_{\mathrm{ws}}\subseteq\mathcal G_0$ and~$\cE_{\mathrm{ws,loc}}\in\mathcal G_0$, the tower property gives
\begin{align*}
\P(\cE_{\mathrm{ws,ref}}^c\cap\cE_{\mathrm{ws,loc}}\mid\mathcal F_{\mathrm{ws}})
  =\E\left[
      \mathds{1}\{\cE_{\mathrm{ws,loc}}\}
      \cdot
      \P(\cE_{\mathrm{ws,ref}}^c \mid \mathcal G_0)
      \;\middle|\;\mathcal F_{\mathrm{ws}}
    \right]
  \le\frac{\delta}{6}.
\end{align*}

\textbf{Part 3: Optimality of retained arms.}

On completion of \textsc{WarmStart}, under
$\cE_{\mathrm{ws}} 
= \cE_{\mathrm{ws,loc}} \cap
\cE_{\mathrm{ws,ref}}$, we have
$|\hat\mu_{b,0}-\mu_b| \le
4\sigma\sqrt{2d/B}$ for all $b \in \cA$.  We then have the following two conditions.

\textit{Condition~(i): Gap bound.}
For any $a \in \cA_1$, the elimination rule (see Line~\ref{algline:warm_start_elim_rule}) gives
\[
    \hat\mu_{a^*,0} - \hat\mu_{a,0} \le 8\sigma\sqrt{2d/B}.
\]
Combining this with standard decomposition and the estimation error guarantee under $\cE_{\mathrm{ws,ref}}$ yields
\[
  \Delta_a
  = \mu_{a^*} - \mu_a
  \le
 \underbrace{(\mu_{a^*} - \hat\mu_{a^*,0})}_{\le 4\sigma\sqrt{2d/B} \text{ under } \cE_{\mathrm{ws,ref}}}
    + (\hat\mu_{a^*,0} - \hat\mu_{a,0})
    +  \underbrace{(\hat\mu_{a,0} - \mu_a)}_{\le 4\sigma\sqrt{2d/B} \text{ under } \cE_{\mathrm{ws,ref}}}
  \le 16\sigma\sqrt{\frac{2d}{B}}.
\]

\textit{Condition~(ii): Optimal arm retained.}
If $a^* \notin \cA_1$, some $b$ satisfies
$\hat\mu_{b,0} - \hat\mu_{a^*,0} > 8\sigma\sqrt{2d/B}$, which gives
$\mu_b - \mu_{a^*}
> 8\sigma\sqrt{2d/B} - 2\cdot 4\sigma\sqrt{2d/B} = 0$,
contradicting the optimality of $a^*$.  Thus, it must hold that $a^* \in \cA_1$.

\textbf{Part 4: Conclusion.}
Parts~1--2 give
\begin{equation*}
  \P(\cE_{\mathrm{ws}}^c\mid\mathcal F_{\mathrm{ws}})
  \le\P(\cE_{\mathrm{ws,loc}}^c\mid\mathcal F_{\mathrm{ws}})
     +\P(\cE_{\mathrm{ws,ref}}^c\cap\cE_{\mathrm{ws,loc}}
          \mid\mathcal F_{\mathrm{ws}})
  \le\frac{\delta}{6}+\frac{\delta}{6}
   =\frac{\delta}{3}.
\end{equation*}
Since~$\cE_{\mathrm{ws}}\subseteq\cE_2$ and~$\cE_1\in\mathcal F_{\mathrm{ws}}$, this uniform conditional bound implies $\P(\cE_2\mid\cE_1)\ge1-\delta/3$, as desired.

\subsection{Proof of Corollary~\ref{cor:gap_poly_A} (Gap to The Lower Bound)} \label{subsec:bound_gap}

In this section, we provide the analysis for
Corollary~\ref{cor:gap_poly_A}.
For $d\le|\mathcal A|=\mathrm{poly}(d)$ and $\sigma=\Theta(1)$,
we show that the expected regret of
Algorithm~\ref{alg:poly_action_set} is order-optimal up to
$\mathrm{polylog}(d,T)$ factors provided that either $B=\Omega(d^3)$ or $T=\Omega(d^2B)$.
Furthermore, in all remaining regimes, we show that its expected
regret is within a $\widetilde O(\sqrt d)$ factor of the minimax
lower bound implied by Corollary~\ref{cor:general_lower_bound}.
Throughout, $\widetilde O(\cdot)$ hides $\mathrm{polylog}(d,T)$
factors, even when $d$ and $T$ do not appear explicitly in its
argument.
Write
\[
  \mathrm{LB}(B,T,d)\coloneqq\min\{T,\,B+\sqrt{dT}\}.
\]
Corollary~\ref{cor:general_lower_bound} gives a minimax lower bound
of $\widetilde\Omega(\mathrm{LB}(B,T,d))$.
If $\mathrm{LB}(B,T,d)=T$, then the trivial upper bound $\E[R_T]\le2T$
is already optimal to within logarithmic factors.  
Hence, below we assume
$\mathrm{LB}(B,T,d)=B+\sqrt{dT}<T$, and we choose $\delta=1/T$.

We first recall the simplified expected regret bound of Corollary~\ref{cor:poly_A_upper}, under the assumptions that $|\mathcal A|=\mathrm{poly}(d)$ and $\sigma = \Theta(1)$:
\begin{equation}
  \E[R_T] = \widetilde{O}\left(B + d^{3/2}\sqrt{B} + \sqrt{dT}\right). \label{eq:E_RT_general}
\end{equation}
The middle term is $O(B)$ when $B=\Omega(d^3)$, and
$O(\sqrt{dT})$ when $T=\Omega(d^2B)$.
Therefore, either condition gives
$\E[R_T]=\widetilde O(B+\sqrt{dT})$.
We now bound the gap in the remaining regimes according to whether
the algorithm includes the \textsc{SafeArm} and \textsc{WarmStart}
stages.

\paragraph{Case~1: Without warm start.}
In this case, $B=\widetilde O(d)$ by Line~\ref{algline:tau} of Algorithm~\ref{alg:poly_action_set}.
Theorem~\ref{thm:step3_small_b_regret} gives
\[
  \E[R_T]=
  \widetilde O\left(dB+\sqrt{dT}\right),
\]
since the failure event contributes only $O(T\delta)=O(1)$ to expected regret.
If $T \le dB$, the trivial bound $\E[R_T]\le2T$ gives
\[
    \frac{\E[R_T]}{\mathrm{LB}(B,T,d)}
  =\widetilde O\left(
      \frac{T}{B+\sqrt{dT}}\right)
 =\widetilde O\left(
      \frac{T}{\sqrt{dT}}\right)
    =\widetilde O\left(
      \sqrt{\frac{T}{d}}\right)
  =\widetilde O(\sqrt B)
  =\widetilde O(\sqrt d).
\]
If $T > dB$, then $dB \le \sqrt{B} \sqrt{dT}$. 
It follows that
\begin{align*}
  \frac{\E[R_T]}{\mathrm{LB}(B,T,d)}
  =\widetilde O\left(
      \frac{dB + \sqrt{dT}}{B+\sqrt{dT}}\right)
  =\widetilde O\left(
      \frac{\sqrt{B} \sqrt{dT} }{\sqrt{dT}}\right)
  =\widetilde O(\sqrt B)
   =\widetilde O(\sqrt d).
\end{align*}

\paragraph{Case~2: With warm start.}
The pulls in \textsc{SafeArm} and the core-set pulls in
\textsc{WarmStart} together contribute at most $\widetilde O(B)$
regret. All other pulls use either $a_{\mathrm{safe}}$ during
\textsc{WarmStart} or an arm in $\cA_1$ during Step~3, since
subsequent active sets are subsets of $\cA_1$.
On $\cE_1\cap\cE_2$, these arms have gap
$\widetilde O(\sqrt{d/B})$ by
Lemmas~\ref{lem:safe_arm} and~\ref{lem:warm_start}.
The two lemmas also give
$\P((\cE_1\cap\cE_2)^c)\le2/(3T)$, so the failure event
contributes only $O(1)$ to expected regret.
Hence, including any interrupted stage, we have
\[
  \E[R_T]
  =\widetilde O\left(B+T\sqrt{\frac dB}\right).
\]
If $T\le dB$, the preceding bound gives the following (via $\sqrt{T} \le \sqrt{dB}$):
\[
    \E[R_T]=\widetilde O(B+d\sqrt T),
\]
If $T>dB$, the general upper bound in \eqref{eq:E_RT_general} gives the same conclusion,
since $d^{3/2}\sqrt B\le d\sqrt T$ and
$\sqrt{dT}\le d\sqrt T$.
Therefore,
\[
  \frac{\E[R_T]}{\mathrm{LB}(B,T,d)}
  =\widetilde O\!\left(
      \frac{B+d\sqrt T}{B+\sqrt{dT}}\right)
  =\widetilde O(\sqrt d).
\]
The ratio is therefore $\widetilde O(\sqrt d)$ in both cases. This completes the proof of Corollary~\ref{cor:gap_poly_A}.

\section{Discussion on Downlink Communication} \label{sec:downlink}

While our problem formulation does not limit the amount of downlink communication, it is worth noting what needs to be sent from the learner to the agent in each of our algorithms.
For Algorithm \ref{alg:exp_action_set}, the amount of downlink communication is relatively small; in each batch, the learner sends (i) a single arm to be played, and (ii) a single threshold query specified by its threshold and direction (i.e., $\ge$ or $\le$).\footnote{In fact, $\ge$ comparisons alone suffice. The localization procedure of~\cite[Appendix~A, Step~1]{lau2026sequential}
can be implemented in this form, and the $<$ responses
in~\eqref{eq:regional_queries} are obtained by complementing the corresponding $\ge$ bits.} 

Algorithm \ref{alg:poly_action_set} has similar communication requirements in the phased elimination part, but it incurs much higher communication costs in the pre-processing steps described in Section \ref{sec:poly_alg}.  Perhaps most significantly, for safe-arm identification (Algorithm \ref{alg:safe_arm}), the learner splits the set of non-eliminated actions in half and reveals the two halves to the agent; a direct implementation thus incurs $\Omega( |\cA| )$ overhead in the early rounds.  Moreover, the agent needs to compute $\hat{\theta}_h^{(j)}$ in \eqref{eq:theta_hat_poly}, which requires knowing the allocation $\rho_h$.  In contrast, the requirements for the warm-start procedure (Algorithm \ref{alg:warm_start}) are milder, with each batch only specifying two arm identifiers, their prescribed allocations, and a threshold query for localization or refinement.

We also note that our analysis assumes that the agent is given exact real numbers (e.g., a threshold).  One could also consider encoding these to $1/\mathrm{poly}(T)$ accuracy using $O(\log T)$ bits, but studying the effect of such approximations is beyond our scope.  Overall, we leave it for future work to study settings in which the downlink communication is also constrained.


\section{Extension to Multiple Bits per Batch (Remark~\ref{rem:q_bits})}
\label{app:q_bits}

We now extend our main results to the $q$-bit setting, where $q\ge1$ is an integer and the agent may transmit up to~$q$ bits of feedback per batch. Each batch still contains~$B$ pulls, and all feedback is returned only at its end.

\paragraph{Lower bound.}
Before batch~$i$, the feedback transcript contains at most~$q(i-1)$ bits. Using this bound in the proof of Theorem~\ref{thm:minimax_lower}, together with the first-batch penalty, gives the communication-based lower bound
\begin{equation*}
  \Omega\left(\min\left\{T,\,
    B+\frac{B}{q}\min\{d,\log|\cA|\}\right\}\right).
\end{equation*}
The  $\Omega(B)$ penalty remains even with unquantized feedback, since the actions in the first batch are chosen before any rewards are observed.

\paragraph{Upper-bound constructions.}
The high-level idea is to divide each batch of~$B$ pulls into~$q$ sub-batches of~$b=B/q$ pulls each,\footnote{For simplicity, assume $q\le B$ and $q$ divides $B$.
For general~$q$, we allocate $\bar q=\min\{q,B\}$
sub-batches within a batch, each containing
$b=\lfloor B/\bar q\rfloor$ pulls. The remaining $B-\bar q b<\bar q b$ pulls arise from rounding
and are allocated by repeating a subset of the planned pulls exactly once. Their rewards are discarded, so no additional queries are made. This at most doubles the regret. Since~$b=\Theta(B/\bar q)$,
the upper bounds hold with~$q$ replaced by~$\bar q$, up to constant factors.} allowing one query per sub-batch (and hence $q$ queries per batch).
However, these sub-batches must be handled in parallel: all their actions and queries must be fixed before the batch begins, since feedback arrives only at the end of the batch. We therefore adapt the 1-bit algorithms to use effective batch size~$b$.  A high-level outline is given as follows, and further details are given below:
\begin{itemize}[leftmargin=4ex]
  \item \emph{Refinement.} Once the localization intervals are fixed, the refinement queries in Appendix~\ref{subsec:1bitRefine} are non-adaptive. We can therefore include queries from different arms and median-of-means blocks in the same batch, without further modification to the estimator.
  
  \item \emph{Localization.} We use the localizer based on noisy binary search in Appendix~\ref{subsec:1Bitlocalize}.  Queries for different arms can share a batch, but successive queries for the same arm require feedback from earlier batches. If fewer than~$q$ arms still require localization queries, some sub-batches are unused for localization.\footnote{Using the non-adaptive localizer of~\cite[Lemma~1]{lau2026non} would allow multiple queries for the same arm to share a batch, simplifying the scheduling and the additional regret analysis. It matches the interval-width and query-complexity bounds of Proposition~\ref{prop:localize} up to absolute constants, but uses general binary queries rather than simple threshold comparisons.} We account for the resulting additional regret below.

  \item \emph{Warm start.} We use the same localization and refinement modifications as above. The remaining pulls within each \emph{sub-batch} are allocated to~$a_{\mathrm{safe}}$, as in Lines~\ref{line:warm_start_localize} and~\ref{line:warm_start_refine} of Algorithm~\ref{alg:warm_start}. The guarantees for the retained arms in Lemma~\ref{lem:warm_start} hold with~$b$ in place of~$B$, up to absolute constants.

    \item \emph{Safe arm identification.} We replace each binary comparison in Algorithm~\ref{alg:safe_arm} by a comparison among up to~$2^q$ groups of candidate arms, communicating the winning group's index.
    
\end{itemize}
Throughout, the estimation parameters and allocations use~$b$ in place of~$B$. Line~\ref{algline:tau}  of Algorithm~\ref{alg:poly_action_set} also uses~$b$ to select its small- or large-batch branch. The only exception is \textsc{SafeArm}, whose comparisons use a full batch of~$B$ pulls.  We now proceed to describe the above ideas in more detail.

\paragraph{Refinement and localization.}
In each phased elimination epoch, each sub-batch contains $b$ pulls of a single arm. As in the full-batch construction in Appendix~\ref{app: 1bit subroutines}, the query uses their average, giving effective standard deviation~$\sigma' = \sigma/\sqrt b$. We retain the original failure probabilities and use disjoint fresh observations for all queries.

For refinement, we combine queries across arms and
median-of-means blocks into batches of~$q$ queries,
with at most one batch containing fewer queries per
refinement stage. We allocate any unused sub-batches to arbitrary active arms and discard their rewards, adding at most~$2B$ regret per refinement stage.

For localization, each batch makes one query for each of~$q$ different arms whose localization is unfinished, using one sub-batch per arm. To preserve the localization protocol in Appendix~\ref{subsec:1Bitlocalize}, each arm's next query is chosen using its responses from earlier batches. If fewer than~$q$ arms still require localization queries, we use one sub-batch per such arm for its next query in the current batch. Further queries for these arms must wait for feedback at the end of the batch, so we allocate the other sub-batches to arbitrary active arms and discard their rewards.
Since every unfinished arm receives one query in each such batch, the number of batches with sub-batches unused for localization is at most the largest number of localization queries needed by any arm. Since each such batch incurs at most~$2B$ regret,
Proposition~\ref{prop:localize} gives
$O\!\left(B\left[\log(\sqrt b/\sigma)
+\log(1/\delta_{\mathrm{loc}})\right]\right)$
additional regret per localization stage.

Summing the additional regret from refinement and localization
over the~$O(\log T)$ epochs and the warm-start stage gives $\widetilde O(B \log T \log(1/\delta))$ additional regret in total.

\paragraph{Safe arm identification.}
We use the same core-set sampling and estimation procedures as in Algorithm~\ref{alg:safe_arm}, with~$B$ pulls per comparison to estimate the means of all current candidate arms.
However, successive binary comparisons require separate batches, since the surviving candidate set is determined by feedback received only at the end of the current batch.  Accordingly, in order to make full use of the $q$ bits available,
we instead partition the current candidates
into~$\min\{2^q,|\cA_h|\}$ balanced groups at each comparison. The agent now sends the index of the group containing an empirical maximizer using at most~$q$ bits (rather than one bit
selecting between two groups). After at most
$\lceil(\log_2|\cA|)/q\rceil$ comparisons, one candidate
remains, giving regret
\begin{equation*}
  O\left(B\left\lceil\frac{\log_2|\cA|}{q}\right\rceil\right)
  =O\left(B+b\log|\cA|\right).
\end{equation*}
The simultaneous confidence bounds in Appendix~\ref{subsec:safe_arm_proof} still apply, so retaining the group containing an empirical maximizer gives the same per-comparison guarantee.
There are no more comparisons than before, so the original guarantee of Lemma~\ref{lem:safe_arm} continues to hold. Since~$b\le B$, this also implies its returned-arm gap bound with~$b$ in place of~$B$ and the same failure probability, which suffices for the subsequent warm-start regret analysis.

\paragraph{Resulting regret bounds.}
The regret analyses for localization, refinement, and \textsc{WarmStart} in Appendices~\ref{app:proof_exp_upper}
and~\ref{app:proof_poly_upper} apply to the queried
sub-batches with~$B$ replaced by~$b$. Adding the~$\widetilde O(B \cdot \log T \cdot\log(1/\delta))$ overhead described above and the modified \textsc{SafeArm} regret for Algorithm~\ref{alg:poly_action_set} gives the following bounds:
\begin{itemize}
    \item Under the assumptions of Theorem~\ref{thm:exp_A_upper}, the modified Algorithm~\ref{alg:exp_action_set} satisfies, with probability at least~$1-\delta$,
    \begin{equation*}
      R_T=\widetilde O\left(
        (B+db)\log\frac1\delta
        +d\sqrt{T\log\frac1\delta}\right).
    \end{equation*}
    \item Under the assumptions of Theorem~\ref{thm:poly_A_upper}, the modified Algorithm~\ref{alg:poly_action_set} satisfies, with probability at least~$1-\delta$,
    \begin{equation*}
      R_T=\widetilde O\left(
        B\log\frac1\delta
        +b\log\frac{|\cA|}{\delta}
        +d^{3/2}\sqrt b\log^{5/2}\frac{|\cA|}{\delta}
        +\sqrt{dT\log\frac{|\cA|}{\delta}}\right).
    \end{equation*}
\end{itemize}
Substituting~$b=B/q$ then gives the bounds stated in Remark~\ref{rem:q_bits}.

\section*{NeurIPS Paper Checklist}

\begin{enumerate}

\item {\bf Claims}
    \item[] Question: Do the main claims made in the abstract and introduction accurately reflect the paper's contributions and scope?
    \item[] Answer: \answerYes{} 
    \item[] Justification: The abstract gives a brief outline of the problem setup, alongside an outline of our main results. Similarly, the introduction gives a more detailed exposition of the problem, alongside comparisons with similar works in the literature. We also forward-reference the detailed problem setup and main results in the introduction, with these sections building on the discussions in the abstract and introduction.

\item {\bf Limitations}
    \item[] Question: Does the paper discuss the limitations of the work performed by the authors?
    \item[] Answer: \answerYes{} 
    \item[] Justification: The limitations of the paper are discussed in Section \ref{subsec:discussions_limitations}. Additionally, we discuss possible future directions of work that address these limitations in the conclusion.

\item {\bf Theory assumptions and proofs}
    \item[] Question: For each theoretical result, does the paper provide the full set of assumptions and a complete (and correct) proof?
    \item[] Answer: \answerYes{} 
    \item[] Justification: A detailed problem setup (including all relevant assumptions) is given in Section \ref{sec:problem_setup}. Detailed proofs of Theorems \ref{thm:minimax_lower} and \ref{thm:non_mixing_lb} are given in Appendix \ref{app:lower_bound}, while the proofs of Theorems \ref{thm:exp_A_upper} and \ref{thm:poly_A_upper} are given in Appendices \ref{app:proof_exp_upper} and \ref{app:proof_poly_upper} respectively.

    \item {\bf Experimental result reproducibility}
    \item[] Question: Does the paper fully disclose all the information needed to reproduce the main experimental results of the paper to the extent that it affects the main claims and/or conclusions of the paper (regardless of whether the code and data are provided or not)?
    \item[] Answer: \answerNA{} 
    \item[] Justification: The paper does not include experiments.

\item {\bf Open access to data and code}
    \item[] Question: Does the paper provide open access to the data and code, with sufficient instructions to faithfully reproduce the main experimental results, as described in supplemental material?
    \item[] Answer: \answerNA{} 
    \item[] Justification: The paper does not include experiments.

\item {\bf Experimental setting/details}
    \item[] Question: Does the paper specify all the training and test details (e.g., data splits, hyperparameters, how they were chosen, type of optimizer, etc.) necessary to understand the results?
    \item[] Answer: \answerNA{} 
    \item[] Justification: The paper does not include experiments.

\item {\bf Experiment statistical significance}
    \item[] Question: Does the paper report error bars suitably and correctly defined or other appropriate information about the statistical significance of the experiments?
    \item[] Answer: \answerNA{} 
    \item[] Justification: The paper does not include experiments. 

\item {\bf Experiments compute resources}
    \item[] Question: For each experiment, does the paper provide sufficient information on the computer resources (type of compute workers, memory, time of execution) needed to reproduce the experiments?
    \item[] Answer: \answerNA{} 
    \item[] Justification: The paper does not include experiments.
    
\item {\bf Code of ethics}
    \item[] Question: Does the research conducted in the paper conform, in every respect, with the NeurIPS Code of Ethics \url{https://neurips.cc/public/EthicsGuidelines}?
    \item[] Answer: \answerYes{} 
    \item[] Justification: Due to the theoretical nature of this work, there are no ethical issues concerning human participation. Similarly, there are no data-related concerns, since no datasets are used. Regarding societal impact, this paper does not violate any of the discussed areas of concern, and our work has no direct relation to any specific individuals or any gender, race, sexuality, or religion. 
    

\item {\bf Broader impacts}
    \item[] Question: Does the paper discuss both potential positive societal impacts and negative societal impacts of the work performed?
    \item[] Answer: \answerNA{} 
    \item[] Justification: We do not anticipate any direct societal impact of the work performed, particularly given its theoretical nature.
    
\item {\bf Safeguards}
    \item[] Question: Does the paper describe safeguards that have been put in place for responsible release of data or models that have a high risk for misuse (e.g., pretrained language models, image generators, or scraped datasets)?
    \item[] Answer: \answerNA{} 
    \item[] Justification: We believe that our work poses no safeguarding risks.

\item {\bf Licenses for existing assets}
    \item[] Question: Are the creators or original owners of assets (e.g., code, data, models), used in the paper, properly credited and are the license and terms of use explicitly mentioned and properly respected?
    \item[] Answer: \answerNA{} 
    \item[] Justification: This paper uses no existing assets.

\item {\bf New assets}
    \item[] Question: Are new assets introduced in the paper well documented and is the documentation provided alongside the assets?
    \item[] Answer: \answerNA{} 
    \item[] Justification: This paper does not release any new assets.

\item {\bf Crowdsourcing and research with human subjects}
    \item[] Question: For crowdsourcing experiments and research with human subjects, does the paper include the full text of instructions given to participants and screenshots, if applicable, as well as details about compensation (if any)? 
    \item[] Answer: \answerNA{} 
    \item[] Justification: The paper does not involve human subjects.

\item {\bf Institutional review board (IRB) approvals or equivalent for research with human subjects}
    \item[] Question: Does the paper describe potential risks incurred by study participants, whether such risks were disclosed to the subjects, and whether Institutional Review Board (IRB) approvals (or an equivalent approval/review based on the requirements of your country or institution) were obtained?
    \item[] Answer: \answerNA{} 
    \item[] Justification: This paper does not involve human subjects.

\item {\bf Declaration of LLM usage}
    \item[] Question: Does the paper describe the usage of LLMs if it is an important, original, or non-standard component of the core methods in this research? Note that if the LLM is used only for writing, editing, or formatting purposes and does not impact the core methodology, scientific rigorousness, or originality of the research, declaration is not required.
    \item[] Answer: \answerNA{} 
    \item[] Justification: LLMs were used as supporting tools for manuscript preparation and technical revision, rather than as components of the research methodology. See the Declaration of LLM Usage for details.
    \item[] Guidelines:
    \begin{itemize}
        \item The answer NA means that the core method development in this research does not involve LLMs as any important, original, or non-standard components.
        \item Please refer to our LLM policy (\url{https://neurips.cc/Conferences/2026/LLM}) for what should or should not be described.
    \end{itemize}

\end{enumerate}

\end{document}